\documentclass{article} 
\usepackage{iclr2027_conference,times}
\usepackage{iclr2027_conference_added_preprint_flag}
\usepackage{longtable}

\usepackage{amsmath,amsfonts,bm}

\def\eqref#1{equation~\ref{#1}}

\def\1{\bm{1}}

\def\eps{{\epsilon}}

\DeclareMathAlphabet{\mathsfit}{\encodingdefault}{\sfdefault}{m}{sl}
\SetMathAlphabet{\mathsfit}{bold}{\encodingdefault}{\sfdefault}{bx}{n}

\DeclareMathOperator*{\argmax}{arg\,max}
\DeclareMathOperator*{\argmin}{arg\,min}

\usepackage{hyperref}
\usepackage{etoc}
\usepackage{url}
\usepackage{comment}

\usepackage{empheq}
\usepackage{array}
\usepackage{microtype}
\usepackage{empheq}
\usepackage{array}

\usepackage{kbordermatrix}
\usepackage{forest}
\usetikzlibrary{arrows.meta}

\usepackage{amsthm,amsmath,bbm,amsfonts,amssymb}
\usepackage{dsfont} 
\usepackage{mathtools}
\usepackage{commath}
\mathtoolsset{showonlyrefs=false}

\usepackage{xspace}
\usepackage[T1]{fontenc}
\usepackage[utf8]{inputenc}
\usepackage[usenames,dvipsnames]{xcolor} 

\usepackage{hyperref,url}
\usepackage{cleveref} 

\usepackage{wrapfig}
\usepackage[export]{adjustbox}
\usepackage{algorithm,algorithmic}
\usepackage{algorithmic}

\usepackage{graphicx,tabularx}
\usepackage{multirow,hhline}
\usepackage{booktabs}

\usepackage{capt-of}

\usepackage{enumitem}

\allowdisplaybreaks 

\usepackage{pbox} 

\usepackage[export]{adjustbox}

\usepackage{pifont}
\newcommand{\kj}[1]{{\color{RedOrange}[#1]}}

\usepackage{framed}
\usepackage[most]{tcolorbox}
\definecolor{kjgray}{rgb}{.7,.7,.7}

\newtheoremstyle{kjstyle}
{1ex} 
{\topsep} 
{\itshape} 
{} 
{\bfseries} 
{.} 
{.5em} 
{} 

\tcbset{kjboxstyle/.style={title={},breakable,colback=white,enhanced jigsaw,boxrule=1.3pt,sharp corners,colframe=kjgray,boxsep=0pt,coltitle={black},attach title to upper={},left=.8ex,bottom=.4em}}    

\newtheoremstyle{kjstylenoitalic}
{1ex} 
{\topsep} 
{} 
{} 
{\bfseries} 
{.} 
{.5em} 
{} 

\theoremstyle{kjstyle}\newtheorem{thm}{Theorem}
\AtBeginEnvironment{thm}{\begin{tcolorbox}[kjboxstyle]}%
  \AtEndEnvironment{thm}{\end{tcolorbox}}%
\theoremstyle{kjstyle}
\AtBeginEnvironment{lem}{\begin{tcolorbox}[kjboxstyle]}%
  \AtEndEnvironment{lem}{\end{tcolorbox}}%
\theoremstyle{kjstyle}
\AtBeginEnvironment{prop}{\begin{tcolorbox}[kjboxstyle]}%
  \AtEndEnvironment{prop}{\end{tcolorbox}}%
\theoremstyle{kjstyle}\newtheorem{cor}[thm]{Corollary}
\AtBeginEnvironment{cor}{\begin{tcolorbox}[kjboxstyle]}%
  \AtEndEnvironment{cor}{\end{tcolorbox}}%

\theoremstyle{kjstyle}
\AtBeginEnvironment{defn}{\begin{tcolorbox}[kjboxstyle]}%
  \AtEndEnvironment{defn}{\end{tcolorbox}}%

\theoremstyle{kjstyle}
\AtBeginEnvironment{ass}{\begin{tcolorbox}[kjboxstyle]}%
  \AtEndEnvironment{ass}{\end{tcolorbox}}%

\theoremstyle{kjstyle}
\AtBeginEnvironment{question}{\begin{tcolorbox}[kjboxstyle]}%
  \AtEndEnvironment{question}{\end{tcolorbox}}%

\theoremstyle{kjstylenoitalic}\newtheorem{remark}{Remark}
\AtBeginEnvironment{remark}{\begin{tcolorbox}[kjboxstyle]}%
  \AtEndEnvironment{remark}{\end{tcolorbox}}%

\usepackage{mdframed}
\usepackage{lipsum}
\definecolor{kjgray}{rgb}{.7,.7,.7}
\makeatletter



\makeatletter
\renewcommand{\paragraph}{%
  \@startsection{paragraph}{4}%
  {\z@}{0.50ex \@plus 1ex \@minus .2ex}{-1em}%
  {\normalfont\normalsize\bfseries}%
}
\makeatother

\usepackage[customcolors,shade]{hf-tikz} 

\def\ddefloop#1{\ifx\ddefloop#1\else\ddef{#1}\expandafter\ddefloop\fi}

\def\ddef#1{\expandafter\def\csname #1#1\endcsname{\ensuremath{\mathbb{#1}}}}
\ddefloop ABCDFGHIJKLMNORSTUWXYZ\ddefloop 

\def\ddef#1{\expandafter\def\csname c#1\endcsname{\ensuremath{\mathcal{#1}}}}
\ddefloop ABCDEFGHIJKLMNOPQRSTUVWXYZ\ddefloop

\def\ddef#1{\expandafter\def\csname b#1\endcsname{\ensuremath{{\mathbf{#1}}}}}
\ddefloop ABCDEFGHIJKLMNOPQRSTUVWXYZ\ddefloop  
\def\ddef#1{\expandafter\def\csname b#1\endcsname{\ensuremath{{\boldsymbol{#1}}}}}
\ddefloop abcdeghijklmnopqrtsuvwxyz\ddefloop  

\def\ddef#1{\expandafter\def\csname h#1\endcsname{\ensuremath{\hat{#1}}}}
\ddefloop ABCDEFGHIJKLMNOPQRSTUVWXYZabcdefghijklmnopqrsuvwxyz\ddefloop 
\def\ddef#1{\expandafter\def\csname hc#1\endcsname{\ensuremath{\hat{\mathcal{#1}}}}}
\ddefloop ABCDEFGHIJKLMNOPQRSTUVWXYZ\ddefloop
\def\ddef#1{\expandafter\def\csname hb#1\endcsname{\ensuremath{\hat{\mathbf{#1}}}}}
\ddefloop ABCDEFGHIJKLMNOPQRSTUVWXYZ\ddefloop %
\def\ddef#1{\expandafter\def\csname hb#1\endcsname{\ensuremath{\hat{\boldsymbol{#1}}}}}
\ddefloop abcdefghijklmnopqrstuvwxyz\ddefloop %

\def\ddef#1{\expandafter\def\csname t#1\endcsname{\ensuremath{\tilde{#1}}}}
\ddefloop ABCDEFGHIJKLMNOPQRSTUVWXYZabcdefgijklmnpqtsuvwxyz\ddefloop 
\def\ddef#1{\expandafter\def\csname tc#1\endcsname{\ensuremath{\tilde{\mathcal{#1}}}}}
\ddefloop ABCDEFGHIJKLMNOPQRSTUVWXYZ\ddefloop
\def\ddef#1{\expandafter\def\csname tb#1\endcsname{\ensuremath{\tilde{\mathbf{#1}}}}}
\ddefloop ABCDEFGHIJKLMNOPQRSTUVWXYZ\ddefloop
\def\ddef#1{\expandafter\def\csname tb#1\endcsname{\ensuremath{\tilde{\boldsymbol{#1}}}}}
\ddefloop abcdefghijklmnopqrstuvwxyz\ddefloop %

\def\ddef#1{\expandafter\def\csname bar#1\endcsname{\ensuremath{\bar{#1}}}}
\ddefloop ABCDEFGHIJKLMNOPQRSTUVWXYZabcdefghijklmnopqrtsuvwxyz\ddefloop
\def\ddef#1{\expandafter\def\csname barc#1\endcsname{\ensuremath{\bar{\mathcal{#1}}}}}
\ddefloop ABCDEFGHIJKLMNOPQRSTUVWXYZ\ddefloop
\def\ddef#1{\expandafter\def\csname barb#1\endcsname{\ensuremath{\bar{\mathbf{#1}}}}}
\ddefloop ABCDEFGHIJKLMNOPQRSTUVWXYZ\ddefloop
\def\ddef#1{\expandafter\def\csname barb#1\endcsname{\ensuremath{\bar{\boldsymbol{#1}}}}}
\ddefloop abcdefghijklmnopqrstuvwxyz\ddefloop %

\def\ddef#1{\expandafter\def\csname war#1\endcsname{\ensuremath{\overline{#1}}}}
\ddefloop ABCDEFGHIJKLMNOPQRSTUVWXYZabcdefghijklmnopqrtsuvwxyz\ddefloop
\def\ddef#1{\expandafter\def\csname warc#1\endcsname{\ensuremath{\overline{\mathcal{#1}}}}}
\ddefloop ABCDEFGHIJKLMNOPQRSTUVWXYZ\ddefloop
\def\ddef#1{\expandafter\def\csname warb#1\endcsname{\ensuremath{\overline{\mathbf{#1}}}}}
\ddefloop ABCDEFGHIJKLMNOPQRSTUVWXYZ\ddefloop
\def\ddef#1{\expandafter\def\csname warb#1\endcsname{\ensuremath{\overline{\boldsymbol{#1}}}}}
\ddefloop abcdefghijklmnopqrstuvwxyz\ddefloop %

\DeclareMathOperator{\EE}{\mathbb{E}}

\DeclareMathOperator{\PP}{\mathbb{P}}

\DeclareMathOperator{\one}{\mathds{1}}

\def\RR{{\mathbb{R}}}

\def\eps{\ensuremath{\epsilon}\xspace}

\def\T{\ensuremath{\top}}  

\def\eps{\ensuremath{\epsilon}\xspace}

\makeatletter
\newcommand{\vast}{\bBigg@{3}}
\newcommand{\Vast}{\bBigg@{4}}
\makeatother

\def\cX{\ensuremath{\mathcal{X}}\xspace}

\def\cD{\ensuremath{\mathcal{D}}\xspace}

\def\cS{{\ensuremath{\mathcal{S}}}}

\def\cH{{\ensuremath{\mathcal{H}}}}

\def\onec#1{\one\cbr{#1}}

\newcommand{\spacex}{\;\;\;\;}

\newcommand{\kapmode}[1]{{\color{blue}[kap: #1]}}

\newcommand{\inner}[2]{\left\langle #1, #2 \right\rangle}
\usepackage{parskip}

\usepackage{rotating}
\RequirePackage[OT1]{fontenc}
\usepackage{amsmath}
\usepackage{eqnarray}
\usepackage{subcaption}

\usepackage{lmodern}
\usepackage{booktabs}       
\usepackage{nicefrac}       
\usepackage{microtype}      

\usepackage{anyfontsize}
\usepackage{newtxtext}  
\usepackage{MnSymbol} 

\usepackage{enumitem}
\setlist[itemize]{topsep=.5pt,itemsep=0pt,parsep=2pt}
\setlist[enumerate]{topsep=.5pt,itemsep=0pt,parsep=2pt}

\usepackage[normalem]{ulem}
\tcbset{
	highlight math style={
		enhanced,
		colframe=red!60!black,
		colback=yellow!50,
		arc=4pt,
		boxrule=1pt,
		drop fuzzy shadow
	}
}
\newtheorem{assumption}{Assumption}[section]

\newtheorem{definition}{Definition}[section]
\newtheorem{theorem}{Theorem}[section]
\newtheorem{lemma}{Lemma}[section]

\newcommand{\REG}{\mathrm{Reg}}
\newcommand{\HREG}{\widehat{\mathrm{Reg}}}
\newcommand{\edpm}{\textsc{E2D.PMSide}\xspace}
\newcommand{\oedpm}{\textsc{OE2D.PMSide}\xspace}
\newcommand{\sqpm}{\textsc{SquareCB.PMSide}\xspace}
\newcommand{\egqpm}{\textsc{$\epsilon$-Greedy.PMSide}\xspace}

\definecolor{bg}{rgb}{0.91, 0.91, 0.91}
\newcommand{\edit}[2]{{\xspace\textcolor{blue}{\sout{#1}}}{ \textcolor{red}{#2}}}

\newcommand\chicheng[1]{{\color{red}{Chicheng:#1}}}
\newcommand\hao[1]{{\color{cyan}{Hao:#1}}}

\newcommand\kaptodo[1]{{\color{blue}{Kap-TODO:#1}}}
\newcommand\kapquestion[1]{{\color{magenta}{Kap-question:#1}}}

\newif\ifFINAL
 \FINALtrue 

\ifFINAL
  
  \def\guide#1{}
  \def\kj#1{}
  \def\kjnew#1{}
  \renewcommand{\edit}[2]{#2}

  \renewcommand\chicheng[1]{}
  \renewcommand\hao[1]{}
  \renewcommand\kaptodo[1]{}
  \renewcommand\kapquestion[1]{}
\else
  \usepackage{transparent}
  \usepackage[inline]{showlabels}
  \usepackage{rotating}
  \renewcommand{\showlabelfont}%
  {\transparent{0.8}\scriptsize\bfseries\slshape\color{Lavender}}
\fi

\title{Efficient contextual partial monitoring with general function approximation}

\author{Kapilan Balagopalan, Hao Qin \\
	Department of Computer Science \\
	University of Arizona \\
	Tucson, AZ, USA \\
	\texttt{\{kapilanbgp, hqin\}@arizona.edu} \\
	\And
	Kwang-Sung Jun \\
	Graduate School of Artificial Intelligence \\
	POSTECH \\
	Pohang, South Korea \\
	\texttt{kwangsungjun@gmail.com} \\
	\And
	Chicheng Zhang \\
	Department of Computer Science \\
	University of Arizona \\
	Tucson, AZ, USA \\
	\texttt{chichengz@arizona.edu}
}
\iclrfinalcopy

\begin{document}

\maketitle

\lhead{Preprint.}
\begin{abstract}

	We study finite stochastic partial monitoring (PM) games with side information, a paradigm of sequential learning with partial feedback. 
    On each round $t \in [T]$, both player and environment observe a context and player chooses an action and environment simultaneously chooses an outcome. The player neither observes the outcome nor the loss; instead they only observe a feedback symbol
    which is partially informative about the outcome. We utilize the  Estimation-to-Decisions (E2D) meta-algorithm of~\cite{foster21thestatistical} to design \edpm, a computationally-efficient framework for contextual partial monitoring, which uses an online regression oracle to efficiently estimate the outcome distributions based on historical (context, action, feedback symbol) tuples. Since the exploration rule given by the original E2D algorithm is intractable to compute, we propose two computationally-efficient  exploration rules that achieve optimal regret orders in $T$: in the globally observable setting, our exploration rule is based on $\epsilon$-greedy, while in the locally observable setting, our exploration rule uses inverse gap weighting~\citep{foster20beyond} which was previously only known to be effective in the bandit setting. We show that a variant of \edpm named \oedpm can be combined with offline regression oracles to give an efficient algorithm with optimal regret when the contexts are drawn independently from a fixed distribution.  We evaluate our algorithm on a synthetic label-efficient prediction game and an LLM routing task, in which a learner decides from prompt, uncertainty and internal state features whether to escalate a query from a weak model to a strong one. We show that our method has superior performance than prior approaches.
	
\end{abstract}


\section{Introduction}\label{section_main_intro}
The contextual stochastic partial monitoring problem~\citep{bartok2012partial} is an online learning problem where the learner repeatedly takes actions so as to minimize its loss based on the environment's partial feedback. 
At the beginning of the game, 
the learner and the environment are given
a pair $\mathcal{G} = (\boldsymbol{L}, \boldsymbol{\Phi})$ of $k \times d$
matrices, where $k$ is the number of actions available to the learner and $d$ is the number of outcomes
available to the environment. The \emph{loss matrix} $\boldsymbol{L} \in [0,1]^{k \times d}$ records the loss
$L_{ao}$ incurred when the learner plays action $a$ and the environment's outcome is $o$. The \emph{feedback matrix} $\boldsymbol{\Phi} \in \Sigma^{k \times d}$ records what the learner
is shown instead: playing $a$ against outcome $o$ reveals the feedback symbol $\Phi_{ao}$ from a finite alphabet
$\Sigma = \{1, \dots, m\}$. The interaction between the learner and the environment proceeds over $T$ rounds. At round $t$ the learner observes a context $x_t$ (which may be chosen adversarially), samples an action
$A_t \sim p_t(\cdot \mid x_t)$ according to a policy $p_t : \mathcal{X} \to \Delta(k)$, and the environment
draws an outcome $O_t$ from some stochastic distribution $q^{\ast}(x_t)$ independently of the learner's randomization. The learner then
observes the feedback symbol $\sigma_t = \Phi_{A_t O_t}$ and incurs the loss $L_{A_t O_t}$, which remains
hidden. The goal of the learner is to minimize its regret, which is the difference between its expected cumulative loss and that of the best policy in hindsight:
\begin{align*}
	\mathrm{Reg}_T := \sum_{t=1}^{T} \left( p_t^\T \boldsymbol{L} q^{\ast}(x_t)
	- \min_{\lambda \in \Delta(k)} \lambda^\T \boldsymbol{L} q^{\ast}(x_t) \right).
\end{align*}


Contextual stochastic partial monitoring has many applications, such as learning from one-sided feedback~\citep{helmbold00apple,sculley2007practical}, 
dynamic pricing~\citep{kleinberg2003value}, data procurement in crowdsourcing~\citep{singla2014contextual}, and label-efficient prediction~\citep{helmbold1997some}. Below we give two examples:


\paragraph{Example 1: LLM routing.} 
Modern Large language models (LLM)-based question-answering service optimizes its cost-efficiency by operating in a hybrid mode~\citep{ong2025routeLLM, chen2023frugalgpt,byeongwoo26efficient, li2026jev}
: it keeps two LLMs -- one weak but cheap and one strong but expensive; and given a prompt, it judiciously decides which LLMs to route the prompt to, obtaining its response.
At time step $t$, the service receives a new prompt represented by context $x_t$, and the weak model generates a response at negligible cost,
and the router then either serves the weak response ($a_1$) or queries the strong model and serves its response
($a_2$). The environment's (hidden) outcome is whether the weak response agrees with the strong one ($o_1$) or does not agree ($o_2$).
We take the 
strong model to be correct and to cost $1$ per query, and we incur cost of $\ell_{01} > 1$ when a
served weak model's response is wrong. Partial monitoring enters through the feedback: serving the weak response never
reveals whether it was correct, so both outcomes yield the same symbol $\odot$, and only querying the strong
model distinguishes them.
The game's loss matrix and feedback matrix are summarized below:

\begin{align*}
	\boldsymbol{L} = \kbordermatrix{ & o_1 & o_2 \\ a_1 & 0 & \ell_{01} \\ a_2 & 1 & 1 },
	\qquad
	\boldsymbol{\Phi} = \kbordermatrix{ & o_1 & o_2 \\ a_1 & \odot & \odot \\ a_2 & \bot & \wedge }.
\end{align*}
This game was also referred to as the Apple Tasting game~\cite{helmbold00apple} in earlier works.
\paragraph{Example 2: contextual label efficient prediction game~\citep{lattimore2020exploration,helmbold1997some}.}
Suppose we would like to learn a spam filter online, which predicts whether an email is spam ($o_1$) or not ($o_2$) based on its text features ($x_t$) and directs it to the inbox ($a_2$) or spam folder ($a_3$). There is also a third option ($a_1$) where the system can ask the recipient/expert for classification, which is accurate but incurs a fixed cost. Only when the learner chooses action $a_1$ does it get to see the true label of the email. The game's loss matrix and feedback matrix are summarized below:
\begin{align*}
	\boldsymbol{L} = \kbordermatrix{ & o_1 & o_2 \\ a_1 & c & c  \\ a_2 & 1 & 0 \\ a_3 & 0 & 1 },
	\qquad
	\boldsymbol{\Phi} = \kbordermatrix{ & o_1 & o_2 \\ a_1 & \bot & \wedge \\ a_2 & \odot & \odot \\ a_3 & \odot & \odot }.
\end{align*}
Depending on the feedback structure, finite stochastic partial-monitoring games are classified into four categories (classification theorem): trivial, easy, hard, and hopeless games \citep{bartok11pmminimax, lattimore2020exploration}.
In hopeless games, the feedback is too uninformative to support meaningful learning, and the regret might grow linearly with the time horizon ($\Omega(T)$). In contrast, in trivial games, there exists an action which is optimal for every outcome; consequently, they can be solved with zero regret. The main cases requiring more refined analysis are easy ($\Theta(\sqrt{T})$) and hard games ($\Theta(T^{2/3})$). Both classes satisfy global observability, which we define shortly. In addition, easy games satisfy local observability, whereas hard games do not.
\paragraph{Related work: partial monitoring with side information.}
We adopt the contextual partial monitoring setting of
\citet{bartok2012CBPside, heuillet24spm}, and use the terms \emph{contextual partial
	monitoring} and \emph{partial monitoring with side information} interchangeably for it. The classification of
easy and hard games established for finite stochastic non-contextual partial monitoring
\citep{bartok11pmminimax, lattimore2020exploration} extends to this setting, at least in the minimax regret's dependence on $T$. First,
the non-contextual lower bounds of~\cite{bartok11pmminimax} transfer immediately, since a non-contextual game can be viewed as a contextual game with a dummy context.
Second, for the upper bounds, \citet{foster2022complexity} shows the equivalence between the Decision-Estimation Coefficient~\cite{foster21thestatistical} and the exploration-by-optimization (ExO) objective of
\citet{lattimore2020exploration}, which implies that the contextual version of the Estimation-to-Decision algorithm~\cite[][Section 8]{foster21thestatistical} achieve regrets that match optimal rates given by ExO. These optimal statistical guarantees come at a computational cost, however: the algorithm attaining them selects each action
distribution by solving a minimax problem at every round, which is not known to be tractable (see Section~\ref{sec:e2d} for further discussion). Furthermore, our work is closely related to \citet{lattimore2019information, lattimore2020exploration, lattimore2022minimax, lattimore2021mirror} and \citet{foster2022complexity}. A detailed discussion is deferred to \Cref{app_related_works}.

Meanwhile, computationally efficient algorithms have been proposed for contextual stochastic partial monitoring, by imposing strong assumptions on the function class that models the true outcome distribution $q^*$. Utilizing generalized linear models for modeling (functions of) $q^*$, \citet{bartok2012CBPside} proposed \textsc{CBPSide}, a deterministic algorithm based on
confidence intervals that extends UCB-style methods from the bandit setting; it applies only to easy games,
where it achieves $O(\sqrt{T})$ regret. \citet{heuillet24spm} later introduced \textsc{CBPSide}$^*$ and its randomized
counterpart \textsc{RandCBPSide}$^*$, which cover both classes, attaining $O(\sqrt{T})$ regret on easy games and
$O(T^{2/3})$ on hard ones. \citet{kirschner2023linear} studies the linear partial monitoring and its contextual extension, which generalizes contextual finite partial monitoring; however, they require that the mappings from (context, action) to the expected reward and the signal feedback to be (context-independent) linear functions.


This motivates the question that drives our work: are there computationally efficient algorithms for contextual
partial monitoring that support general function approximation and attain the optimal rates? In this paper, we answer this question
affirmatively. Specifically, we make the following contributions:

\paragraph{Contributions.}
(1) we propose an algorithmic framework \edpm that reduces stochastic contextual partial monitoring to online regression. It is based on the E2D (Estimation to Decision) framework~\citep{foster21thestatistical}, where the algorithm aims at finding an exploration distribution that certifies a low Decision-Estimation Coefficient (See Section~\ref{sec:e2d} for its definition); 
our key observation is that this can be done by optimizing a surrogate objective that trades off between greediness and coverage, which paves the way for the design of efficient algorithms for contextual partial monitoring. 
(2) Based on the above relaxed objective, we instantiate \edpm 
in the globally observable and locally observable games to obtain \egqpm and \sqpm algorithms respectively, achieving regret of optimal order in $T$ in their respective settings. These, to our knowledge, are the first \emph{computationally efficient} contextual partial monitoring algorithms with \emph{general function approximation.}
(3) When the contexts are iid, we propose a variant of \edpm named \oedpm, which is an efficient stochastic contextual partial monitoring algorithm based on offline regressions, a computational primitive that can admit simpler implementations in practice.  
(4) We evaluate \sqpm in two contextual partial monitoring tasks: label-efficient prediction and LLM routing, and show that it outperforms several baselines.

\section{Preliminaries}

\paragraph{Basic Notation.}
We write $\boldsymbol{0}$ and $\boldsymbol{1}$ for the all-zeros and all-ones column vectors, $e_1, \dots, e_k$
for the standard basis vectors, and bold symbols for matrices. For a vector or matrix $A$,
$\lVert A \rVert_\infty = \max_{ij} |A_{ij}|$ is the entrywise maximum norm, $M_{i}$ denotes the vector generated by $i$-th row, and $M_{ij}$ denotes the element at $i$-th row and $j$-th column. For two distributions $u$ and $v$, $D^2_{\mathrm{H}}\big( u \,\big|\, v \big)$ denotes their squared Hellinger distance. Let $\mathrm{co}(\mathcal{S})$ denote the convex hull of a set $\cS$. For a subset $\mathcal{A} \subseteq [k]$ we write
$\Delta(\mathcal{A}) := \{\lambda \in \Delta(k) : \lambda_b = 0 \text{ for all } b \notin \mathcal{A}\}$
for the face of $\Delta(k)$ supported on $\mathcal{A}$. Elements of $\Delta(\mathcal{A})$ remain
$k$-dimensional vectors; they simply place no mass outside $\mathcal{A}$.

\paragraph{Signal matrix. (\cite{komiyama2015regret})} Given a feedback matrix $\boldsymbol{\Phi}$, it is convenient to describe the feedback of a single action as a linear map. The \emph{signal matrix} of
action $a$ is the $m \times d$ matrix $\boldsymbol{S}_a$ with
$\boldsymbol{S}_a[\sigma, o] = \onec{\Phi_{ao} = \sigma}$. Writing $\Delta(n)$ for the probability
simplex on $n$ elements, if the outcome is drawn from $q \in \Delta(d)$ then $\boldsymbol{S}_a q \in \Delta(m)$
is precisely the distribution of the symbol observed on playing $a$. Two candidate outcome distributions $q_1, q_2$ are
therefore distinguishable through action $a$ exactly to the extent that $\boldsymbol{S}_a q_1$ and
$\boldsymbol{S}_a q_2$ differ; and it is this quantity, rather than the gap between $q_1$ and $q_2$
themselves, that governs what the learner can hope to identify. Recall that in our stochastic contextual partial monitoring problem, each round is preceded by a context $x_t$ drawn from a set
$\mathcal{X}$, and the outcome distribution depends on it: outcomes follow $q^{\ast}(x_t)$ for an unknown
$q^{\ast} : \mathcal{X} \to \Delta(d)$.
The learner is given a model class
$\mathcal{Q} \subseteq \{q : \mathcal{X} \to \Delta(d)\}$, which we assume is well-specified.

\begin{assumption}[Realizability]\label{assumption_realizability}
	The model class $\mathcal{Q}$ contains the true outcome distribution $q^{\ast}$.
\end{assumption}

\subsection{Structure of partial monitoring games}
We now review some basic structures of partial monitoring games revealed in prior works.
\begin{definition}[Cell decomposition, \citet{bartokICML2012}] \label{def_cells}
	The \textit{cell} $\mathcal C_a $ of action $a$ is defined as the subset of outcome distributions where action $a$ is optimal. Formally, 
	$\mathcal C_a =  \{ q \in \Delta(d): (e_a-e_b)^\T\boldsymbol{L} q  \leq 0,  \forall b \in [k] \}.$ The collection $\{\mathcal{C}_a : a \in [k]\}$ is called the cell decomposition.
\end{definition}
\paragraph{Neighborhood graph~\citep{lattimore2020exploration}.}
Cells that meet in a face of dimension $d-2$ correspond to actions that can be optimal simultaneously, and this
adjacency structure organizes the analysis that follows. Call an action \emph{Pareto optimal} when its cell is
full dimensional, $\dim(\mathcal{C}_a) = d-1$, taking the dimension of a polytope to be that of its affine hull,
and write $\Pi$ for the set of such actions together with
$\Delta(\Pi) := \{\lambda \in \Delta(k) : \lambda_b = 0 \text{ for all } b \notin \Pi\}$ for the corresponding
face of $\Delta(k)$. The remaining actions are of two kinds: $a$ is \emph{degenerate} when
$\mathcal{C}_a \neq \emptyset$ but $\dim(\mathcal{C}_a) \leq d-2$, and \emph{dominated} when
$\mathcal{C}_a = \emptyset$. Since $\dim(\Delta(d)) = d-1$, neither kind is ever uniquely optimal; each is
optimal only on a set of measure zero. A game with no degenerate actions is called \emph{non-degenerate}, and
two actions are \emph{duplicates} when their loss rows coincide, $(e_a - e_b)^\T \boldsymbol{L} = \boldsymbol{0}$. Two Pareto optimal actions are \emph{neighbors} when their cells meet in a face of the largest dimension short
of coincidence, $\dim(\mathcal{C}_a \cap \mathcal{C}_b) = d-2$; the intersection reaches dimension $d-1$ only
for duplicates. Writing $\mathcal{E}$ for the resulting edge set, $(\Pi, \mathcal{E})$ is the
\emph{neighborhood graph} of $\cG$.

\subsection{Observability}
In partial-monitoring games, the learner observes feedback rather than the losses or rewards associated with its actions. Therefore, a central challenge is to infer (differences of) losses of actions from the observed feedback symbol. Whether this is possible depends on the interplay between the loss matrix and the feedback matrix.
\begin{definition}[Global observability, adapted from \citet{bartokICML2012} and \citet{lattimore2020exploration}]\label{gl_observability_graph}
	A partial monitoring game is \textit{globally observable} if, for \emph{every pair of actions}
	$(a,b) \in [k] \times [k]$, there exists a \emph{witness function}
	 $w_{(a,b)} : [k] \times \Sigma \to \mathbb{R}$ such that
	\begin{align*}
		L_{ao} - L_{bo}
		=
		\sum_{c=1}^{k} w_{(a,b)}(c,\Phi_{co}),
		\qquad \text{for all } o \in [d].
	\end{align*}
\end{definition}

\begin{definition}[Local observability, \citet{lattimore2020exploration}]
	\label{local_observability_graph}
	A globally observable, \emph{non-degenerate} partial monitoring game is \emph{locally observable} if \emph{for every edge}
	$e = (a,b) \in \mathcal{E}$ of the neighborhood graph the witness function $w_e$ of Definition~\ref{gl_observability_graph} may be
	chosen to be supported only on the endpoints, that is $w_e(c, \cdot) \equiv 0$ for all $c \notin \{a,b\}$, so that
	\begin{align*}
		L_{ao} - L_{bo} = \sum_{c \in \{a,b\}} w_e(c,\Phi_{co}),
		\qquad \text{for all } o \in [d].
	\end{align*}
\end{definition}
By definition, every locally observable game is a globally observable game. Under global observability, the loss difference
$L_{ao} - L_{bo}$ may depend on feedback from all actions $c \in [k]$, whereas under local observability, it can be expressed using only the feedback from the sequence of actions in $\cE$ connecting $a$ and $b$. Consequently, locally observable games permit more direct loss-difference estimation than globally observable games in general. Since trivial games and hopeless games are not interesting from an algorithm design perspective, for the remainder of the paper we focus on globally observable games.
\begin{assumption}[Global observability and non-degeneracy; \citet{lattimore2020exploration}] \label{assumption_ono_degen}
	We assume that the partial-monitoring game $\mathcal{G} = (\boldsymbol{L}, \boldsymbol{\Phi})$ is globally
	observable, non-degenerate, and contains no
	duplicate actions.
\end{assumption}
For simplicity, we maintain this
assumption throughout the remainder of the analysis, following \citet{lattimore2020exploration}. We discuss
implications of non-degeneracy and possible relaxations in Appendix~\ref{relax_non_degeneracy}.


\subsection{The E2D meta-algorithm for interactive decision making}
\label{sec:e2d}


~\cite{foster21thestatistical} proposes E2D (Estimation-to-Decision), a meta-algorithm for interactive decision making,
which reduces interactive decision making to online estimation. The E2D meta-algorithm yielded many practical algorithms for decision making, especially for contextual bandit problems~\citep[e.g.,][]{foster20beyond,zhu22contextual,zhu2022contextual,zhang2023practical}. In a nutshell, at each round, it calls an online model estimation oracle to obtain an estimated model $\hat{q}$; based on $\hat{q}$ and context $x$, 
it solves a minimax optimization problem to obtain an exploration distribution $p$ to take actions. Specialized to partial monitoring, the optimization problem can be written as:
\begin{align}
		\min_{p \in \Delta(k)}  \underbrace{\sup_{q \in \mathcal{Q}_{x}, \lambda \in \Delta(k)}
		\sbr{ p^\T \boldsymbol{L} q(x) - \lambda^\T \boldsymbol{L} q(x)
			- \gamma \cdot \EE_{a \sim p}\sbr{ D^2_{\mathrm{H}}\big( \boldsymbol{S}_a \hat{q}(x) \,\big|\, \boldsymbol{S}_a q(x) \big) } }}_{\mathcal{V}^{\mathrm{dec}}_{\gamma}(\hat{q}(x), \mathcal{Q}_{x}, p)},
    \label{eqn:e2d-minimax}
\end{align}
where $\mathcal{Q}_x := \left\{ q(x) : q \in \mathcal{Q} \right\} \subseteq \Delta(d)$ denotes the model class $\cQ$ restricted to $x$. The objective explicitly trades off between regret and information gain: the first part $p^\T \boldsymbol{L} q(x) - \lambda^\T \boldsymbol{L} q(x)$ encourages $p$ to have a low regret, while the second part $-\gamma \cdot \EE_{a \sim p}\sbr{ D^2_{\mathrm{H}}\big( \boldsymbol{S}_a \hat{q}(x) \,\big|\, \boldsymbol{S}_a q(x) \big) }$ encourages $p$ to explore diverse actions to correct $\hat{q}$ and gain new information about $q^*$. 
The minimax value of the objective is called the Decision-Estimation Coefficient (DEC)~\citep{foster21thestatistical}, denoted as $\mathrm{dec}_{\gamma}(\hat{q}(x), \mathcal{Q}_{x})$. To help the learner make decisions efficiently, E2D requires access to an oracle that can provide accurate estimation of the outcome distributions for all contexts online, which, when specialized to our stochastic contextual partial monitoring setting, is defined as:
\begin{definition}
\label{def:online-oracle}
	An \emph{online outcome estimation oracle} $\mathrm{Alg}_{\mathrm{On}}$ is an algorithm that estimates the unknown
	outcome distribution $q^{\ast}$ from the data in a sequential fashion. That is, at each round
	$t \in [T]$, given the feedback symbols collected so far, the oracle produces an estimate 
	\begin{align*}
		\hat{q}_t
		= \mathrm{Alg}_{\mathrm{On}}^{(t)}\!\left(\big\{(A_s, x_s, \sigma_s)\big\}_{s=1}^{t-1}\right) \in \mathrm{co}(\mathcal{Q}),
	\end{align*}
	where $\mathrm{Alg}_{\mathrm{On}}^{(t)}$ denotes the map applied at round $t$ and
	$\big\{(A_s, x_s, \sigma_s)\big\}_{s=1}^{t-1}$ is the history available at the beginning of round $t$ (note that $\sigma_s$ is an observed feedback symbol, but not the incurred loss).\\
    $\mathrm{Alg}_{\mathrm{On}}$ is said to satisfy \emph{a Hellinger regret bound} of $\overline{\mathrm{Reg}}_{\mathrm{On}}^{\mathrm{H}}\left( \mathcal{Q}, T, \delta\right)$ if its online Hellinger regret 
    \begin{align}\label{eq:hellinger_estimation_error}
    \mathrm{Reg}_{\mathrm{On}}^{\mathrm{H}}\left( \mathcal{Q}, T\right) := \sum_{t=1}^{T} \EE_{a \sim p_t(\cdot\mid x_t )} \sbr{ D^2_{\mathrm{H}}\big( \boldsymbol{S}_a\hat{q}_t(x_t) \,\big|\, \boldsymbol{S}_a q^\ast(x_t)\big) },
	\end{align}
    is at most $\overline{\mathrm{Reg}}_{\mathrm{On}}^{\mathrm{H}}\left( \mathcal{Q}, T, \delta\right)$ with probability $1-\delta$.
\end{definition}


We note that instead of requiring $\hat{q}_t(x_t)$ to be pointwise close to $q^*(x_t)$ on average, Definition~\ref{def:online-oracle} only requires their associated feedback distributions to be close on average. This is reasonable since we can only use the feedback $\sigma_s$'s to guide our online learning of $\hat{q}_t$; in the extreme case, it is possible that different $q$'s result in the same feedback distribution, making the underlying $q$ unidentifiable. Appendix \ref{app_section_oracles} gives two concrete examples of the online outcome estimation oracle. For example, For finite $\mathcal{Q}$, Vovk's aggregating algorithm~\citep{vovk1995game} attains
$\overline{\mathrm{Reg}}_{\mathrm{On}}^{\mathrm{H}}(\mathcal{Q}, T, \delta)
= O\!(\log(\lvert\mathcal{Q}\rvert/\delta))$. Equipped with an online outcome estimation oracle, E2D is known to have a regret bound of 
\[
\mathrm{Reg}_T \leq O\del{ T \max_{x \in \cX} \mathrm{dec}_{\gamma}(\hat{q}(x), \mathcal{Q}_{x}) + \gamma \overline{\mathrm{Reg}}_{\mathrm{On}}^{\mathrm{H}}\left( \mathcal{Q}, T, \delta\right) },
\]
which translates to $O(T^{2/3})$ and $O(\sqrt{T})$ regrets with optimal tuning of $\gamma$ in globally and locally observable settings, respectively~\citep{foster21thestatistical}.
However, a key drawback of the naive E2D algorithm is that directly solving the minimax optimization problem (\Cref{eqn:e2d-minimax}) is computationally intractable: although the objective is linear in $p$, it is not jointly concave in $(q, \lambda)$.
Nevertheless, it was shown in~\cite{foster21thestatistical} that as long as distribution $p$ \emph{certifies} a small value of $\mathrm{dec}_{\gamma}(\hat{q}(x), \mathcal{Q}_{x})$ (\Cref{eqn:e2d-minimax}), i.e., it achieves a small $\cV_\gamma^{\mathrm{dec}}(\hat{q}(x), \cQ_x, p)$ value, E2D will have small regret.
As we will see, our algorithm design will utilize this observation to address the computational challenge.

\section{Efficient algorithm for contextual partial monitoring}


In this section we present \edpm (Estimation-to-Decision for Partial Monitoring with Side information, Algorithm~\ref{algorithm_e2d_pmside}), which can be viewed as a specialization of the E2D meta-algorithm. At every round $t$, \edpm obtains an outcome distribution estimate $\hat{q}_t(\cdot)$ (line~\ref{lin_estimate}) from an online outcome estimation oracle with Hellinger regret guarantee. Then it uses $\hat{q}_t(\cdot)$ and observed context $x_t$ to construct an exploration distribution $p_t$ that
approximately minimizes an objective function $\mathcal{V}_{\gamma}(\hat{q}_t(x_t), p)$, a relaxation of the DEC objective $\mathcal{V}^{\mathrm{dec}}_{\gamma}(\hat{q}(x_t), \mathcal{Q}_{x_t}, p)$ (line~\ref{lin_obective}). The choice of $p_t$ balances two competing demands. The first term in the objective measures the loss of $p_t$ relative
to a comparator policy $\lambda$ under the estimated outcome distribution $\hat{q}_t(x_t)$, and is small when $p_t$ exploits what has been learned. The second term is a \emph{coverage} term (which we will define shortly) that measures the amenability of using feedback from following $p$ in evaluating $\lambda$'s loss; the smaller the coverage value the better $p$ is for exploration. The exploration parameter $\gamma$ fixes the rate of exchange between the two. 
We now define coverage formally.

\begin{definition}[Coverage]\label{def_contextual_doec_coverage}
	We define the coverage of action distribution $p \in \Delta(k)$ for action distribution $\lambda \in \Delta(k)$ for partial monitoring game $\cG = (\boldsymbol{L}, \boldsymbol{\Phi})$ as
	\begin{align*}
		\mathrm{Coverage}(p, \lambda)
		:= \sup_{q_1, q_2 \in \Delta(d)}
		\frac{ \left( \left(\lambda - p\right)^\T \boldsymbol{L} \left(q_1 - q_2\right) \right)^2 }
		{  \EE_{a \sim p} \sbr{ D^2_{\mathrm{H}}\big( \boldsymbol{S}_a q_1 \,\big|\, \boldsymbol{S}_a q_2 \big) } }. \tag{with convention $\frac{0}{0} = 0$}
	\end{align*}
\end{definition}

To make sense of this definition it may be helpful to think about the following hypothetical offline policy evaluation problem in non-contextual stochastic partial monitoring: suppose the outcome distribution $q^*$ is unknown; 
we use a behavior policy $p$ to take action and collect feedback symbols, so as to estimate the expected loss of $\lambda$ (more precisely, the expected loss difference between $\lambda$ and $p$); 
how much utility does each data point carry for such evaluation? Let $q_1$ be $q^*$ and $q_2$ be $\hat{q}$, our estimate of $q^*$ based on data. 
The numerator thus measures the squared policy evaluation error; the denominator measures the in-distribution estimation error, which we expect to decrease with more samples collected. A small $\mathrm{Coverage}(p, \lambda)$ implies that feedback signals via following $p$ is useful in estimating $p$ and $\lambda$'s loss difference. 

Our coverage in the above definition bears some similarities with the coverage defined in the contextual bandit setting by
\citet{qin2026tamingmonstercontextcomplexity}, but is different in two important ways: (1) here, the numerator and the denominator are with respect to the loss matrix (evaluation error) and the signal matrix (information gain), respectively, where in the bandit setting they are defined both using the loss function; (2) it is well-known that for some globally observable partial monitoring problems the absolute value of a loss of an action is not identifiable \cite[][Example 37.9]{lattimore20bandit}, and only loss differences between action pairs is identifiable; thus, the numerator in the coverage measures the difference of losses of two policies. Our coverage also shares some similarities with global and local alignment constants of~\cite[][Eqs. (15), (16)]{kirschner2023linear}, whose numerator $(\inner{\phi_a - \phi_b}{v})^2$ can be interpreted as some estimation error of the loss differences between two actions $a,b$, and the denominator $\| M_c v \|^2$
can be interpreted as the information gain of taking action $c$ in correcting the model estimation error $v$. Intead of measuring coverage between an action and a pair of actions, our definition  measures the coverage between two action distributions.

We now present the regret guarantee of \edpm:

\begin{algorithm}[t!]
	\caption{\edpm}
	\label{algorithm_e2d_pmside}
	\begin{algorithmic}[1]
		\STATE \textbf{parameters:}\\
		\spacex Online estimation oracle $\mathrm{Alg}_{\mathrm{On}}$ \\
		\spacex Exploration parameter $\gamma > 0$ 
		\FOR{$t = 1,2, \cdots, T$}
		\STATE Receive context $x_t$
		\STATE \label{lin_estimate}Compute estimate $\hat{q}_t = \mathrm{Alg}_{\mathrm{On}}^{{(t)}} (\{A_s, x_s,  \sigma_s \}_{s=1}^{{t-1}})$ 
		\STATE \label{lin_obective} Find $p_t$ that approximately minimizes
		\begin{align*}
			\mathcal{V}_{\gamma}(\hat{q}_t(x_t), p)
			:=
			\sup_{\lambda \in \Delta(\Pi)} \underbrace{p^\T \boldsymbol{L} \hat{q}_t(x_t)  -  \lambda^\T  \boldsymbol{L} \hat{q}_t(x_t) }_{\mathrm{exploit}}+  \frac{1}{\gamma} \cdot \underbrace{\mathrm{Coverage}(p, \lambda)}_{\mathrm{explore} } .
		\end{align*}
        using \egqpm (Alg.~\ref{alg:explore-go}) or \sqpm (Alg.~\ref{alg:explore-lo}) for globally and locally observable settings, respectively.
        
		\STATE Choose action $A_t \sim p_t$ and update the estimation oracle with $(A_t, x_t,  \sigma_t)$
		\ENDFOR
	\end{algorithmic}
	
\end{algorithm}

\begin{theorem}[Regret of \edpm]\label{theorem_doec_regret}
	Suppose $\mathrm{Alg}_{\mathrm{On}}$ has a Hellinger regret bound $\overline{\mathrm{Reg}}_{\mathrm{On}}^{\mathrm{H}}\left( \mathcal{Q}, T, \delta \right)$. Fix $\gamma > 0$. Then, for any $\delta \in (0,1)$,  with probability at least $1-\delta$,
	\begin{align*}
		\mathrm{Reg}_T
		\leq T \cdot \max_{t=1}^T \mathcal{V}_{\gamma}\big(\hat{q}_t(x_t), p_t \big)
		+ \gamma \cdot \overline{\mathrm{Reg}}_{\mathrm{On}}^{\mathrm{H}}\left( \mathcal{Q}, T, \delta \right).
	\end{align*}
\end{theorem}

Its proof is deferred to Appendix~\ref{app_section_proof_thoerem_1}. In a nutshell, our theorem follows from the meta-regret theorem of~\cite{foster21thestatistical}: as long as all action distributions $p_t$ certify that $\mathrm{dec}_{\gamma}(\hat{q}(x), \mathcal{Q}_{x})$ has an upper bound, E2D has a small regret in terms of that bound. Indeed, our $p_t$ here certifies that $\mathrm{dec}_{\gamma}(\hat{q}(x), \mathcal{Q}_{x})$ is upper bounded by $\mathcal{V}_\gamma(\hat{q}_t(x_t), p_t)$ (Please see Remark~\ref{remark_dec_bound})
; and ideally, we should choose exploration distribution $p_t$'s that make $\mathcal{V}_\gamma(\hat{q}_t(x_t), p_t)$'s as small as possible.

Unfortunately, finding $p_t$ that exactly minimizes $\mathcal{V}_{\gamma}(\hat{q}_t(x_t), p)$
may still be computationally intractable; to see this, note that the definition of coverage involves  supremum over a pair of outcome distributions $q_1, q_2$, which is not a convex optimization problem.
This motivates the question: can we find computationally efficient exploration rules $x \mapsto p(x)$ that \emph{approximately minimize} $\mathcal{V}_{\gamma}\big(\hat{q}_t(x), p(x) \big)$ for all $x$? Indeed, we show in the following two subsections that such efficient exploration rule exist and yield efficient implementations of \edpm. 
\subsection{Efficient algorithm for globally observable games}
We show that in the globally observable setting \egqpm (Alg. \ref{alg:explore-go}), yields a policy $p$
satisfying $\mathcal{V}_{\gamma}(\hat{q}(x), p(x)) \leq O(\sqrt{1/\gamma})$, which gives an $O(T^{2/3})$ regret
bound when plugged into \edpm.  \egqpm identifies the empirical best action $\hat{b}$
(Line~\ref{lin_emp_best}) and assigns it probability $1 - \epsilon$, spreading the remaining $\epsilon$
uniformly over all $k$ actions (Line~\ref{lin_e_greedy}).
\begin{algorithm}[t]
\caption{\egqpm for globally observable games}
\label{alg:explore-go}
\begin{algorithmic}[1]
\STATE\textbf{parameters:}\\
		\STATE\spacex Context $x$ with predicted outcome distribution $\hat{q}(x)$ \\
		\STATE\spacex Exploration parameter $\gamma > 16 k^2 \beta^2$\\
        \STATE \label{lin_emp_best}Let $\hat{b} \in \argmin_{a \in \Pi} e_a^\T \boldsymbol{L} \hat{q}(x)$ 
        \RETURN \label{lin_e_greedy} $p = (1-\eps) e_{\hat{b}} + \eps \tfrac{1}{k}\boldsymbol{1}$ for $\eps = k\beta / \sqrt{\gamma}$
\end{algorithmic}
\end{algorithm}

\begin{lemma}[Guarantee for Alg.~\ref{alg:explore-go} for globally observable games]\label{lemma_globally_observable_doec}
	Let $\mathcal{G} = (\boldsymbol{L}, \boldsymbol{\Phi})$ be a globally observable partial monitoring game in the
	sense of Definition~\ref{gl_observability_graph}, and set
	$\beta := \max_{(a,b) \in [k] \times [k]} \lVert w_{(a,b)} \rVert_{\infty}$. 
    We have for any $\hat{q} \in \Delta(d)$,
    \[
    \mathcal{V}_{\gamma}\big(\hat{q}, p \big)
    \leq 
    \max_{\lambda \in \Delta(\Pi)} p^T \boldsymbol{L} \hat{q} - \lambda^T \boldsymbol{L} \hat{q}
    + \frac{16\beta^2}{\gamma}  \sum_{c=1}^k \frac{1}{p_c}.
    \]
    
    Thus, for
	every $\gamma > 16 k^2 \beta^2$, every $\hat{q} \in \Delta(d)$, Alg.~\ref{alg:explore-go} outputs $p$ such that
	\begin{align*}
        \mathcal{V}_{\gamma}\big(\hat{q}, p \big)
		\leq O\!\left( k\beta \sqrt{\frac{1}{\gamma}} \right).
	\end{align*}
\end{lemma}

Proof is deferred to Appendix~\ref{app_section_efficient_global}. Furthermore, by combining Theorem~\ref{theorem_doec_regret} and Lemma~\ref{lemma_globally_observable_doec}, our algorithm achieves the following minimax rate for globally observable games.
\begin{cor}[Regret of \edpm for globally observable games]
	\label{corollary_regret_e2d_global} 
	Suppose $\mathrm{Alg}_{\mathrm{On}}$ has a Hellinger regret bound $\overline{\mathrm{Reg}}_{\mathrm{On}}^{\mathrm{H}}( \mathcal{Q}, T, \delta)$. For a globally observable partial monitoring game $(\boldsymbol{L}, \boldsymbol{\Phi})$, set
	$\beta = \max_{(a,b) \in [k] \times [k]} \lVert w_{(a,b)} \rVert_{\infty}$. Then, for any
	$\delta \in (0,1)$, and for any large enough $T$ that satisfies $\gamma > 16 k^2 \beta^2$ with 	$\gamma = ( T k \beta \,/\,
	\overline{\mathrm{Reg}}_{\mathrm{On}}^{\mathrm{H}}(\mathcal{Q}, T, \delta) )^{2/3}$, with probability of at least $1-\delta$, \edpm run with Alg~\ref{alg:explore-go} as the exploration rule
	satisfies,
	\begin{align*}
	\mathrm{Reg}_T
		\leq O\!\left( \sqrt[3]{T^2k^2\beta^2\overline{\mathrm{Reg}}_{\mathrm{On}}^{\mathrm{H}}(\mathcal{Q}, T, \delta) }
	\right).
	\end{align*}
\end{cor}
In the setting when $\cQ$ is finite, setting the estimation oracle to be Vovk's aggregating algorithm, 
the above regret bound is of order $O( \sqrt[3]{T^2k^2\beta^2\log|\mathcal{Q}|})$.
In the noncontextual setting, when $\mathcal{Q} = \Delta(d_{\mathcal{Q}})$, we can set the estimation oracle to be any efficient exp-concave optimization algorithm~\citep{hazan07logarithmic}, which has $\overline{\mathrm{Reg}}_{\mathrm{On}}^{\mathrm{H}}(\mathcal{Q}, T, \delta) = O(d_{\mathcal{Q}}\ln T)$, implying \edpm has a regret of $\tilde{O}( \sqrt[3]{T^2k^2\beta^2 d_{\mathcal{Q}}})$. This matches the lower bound of non-contextual stochastic partial monitoring~\citep{bartok11pmminimax} with a simple computationally efficient approach.

\subsection{Efficient algorithm for locally observable games}
In the locally observable setting, we aim to give an efficient exploration rule that finds $p$ such that $\mathcal{V}_{\gamma}(\hat{q}_t(x), p(x) ) \leq O (\frac 1 \gamma)$, which yields a $O(\sqrt{T})$ regret when plugged into \edpm.
Under global observability, to estimate the loss difference between two actions $(a,b)$, its associated witness function $w_{(a,b)}$ may be supported on all actions (Definition~\ref{gl_observability_graph}),
so recovering the difference between a single pair of actions can require playing all $k$ of them; this is why
the certifying policy of Algorithm~\ref{alg:explore-go} mixes with the uniform distribution over
$[k]$. In contrast, under local observability, the loss difference between two \emph{neighboring} Pareto
optimal actions can be expressed using the feedback of those two actions alone (Definition~\ref{local_observability_graph}), and the difference between an
arbitrary pair of Pareto optimal actions is then recovered by chaining such local relations along a path
joining them in the neighborhood graph, so that only the actions encountered along that path need be played. Perhaps surprisingly, we show that Inverse Gap Weighting, an exploration rule that was previously known to be effective for exploration in contextual bandits (\sqpm, Alg.~\ref{alg:explore-lo}), yields a optimal-regret algorithm here without modification.

\begin{lemma}[Guarantee for Alg.~\ref{alg:explore-lo} for locally observable games]\label{lemma_locally_observable_doec}
    Let $\mathcal{G} = (\boldsymbol{L}, \boldsymbol{\Phi})$ be a locally observable partial monitoring game in the
sense of Definition~\ref{local_observability_graph}. For
every $\gamma > 0$, and every $\hat{q} \in \Delta(d)$, Alg.~\ref{alg:explore-lo} run with $\kappa = \frac{\gamma}{8k m^2}$ and $\mu = k$, outputs $p$ such that
\begin{align*}
	\mathcal{V}_{\gamma}\big(\hat{q}, p \big)
	\leq O\!\left(  \frac{m^2 k^2}{\gamma} \right).
\end{align*}
\end{lemma}
Proof is deferred to Appendix~\ref{app_section_efficient_local}.
Furthermore, by combining Theorem~\ref{theorem_doec_regret} and Lemma~\ref{lemma_locally_observable_doec}, we can achieve the following minimax rate for locally observable games.
\begin{cor}[Regret of \edpm for locally observable games]
	\label{corollary_regret_e2d_local}
	Suppose $\mathrm{Alg}_{\mathrm{On}}$ has a Hellinger regret bound $\overline{\mathrm{Reg}}_{\mathrm{On}}^{\mathrm{H}}( \mathcal{Q}, T, \delta)$. Let Assumption~\ref{assumption_ono_degen} hold, and suppose in addition that the game is locally
	observable. Then, for any $\delta \in (0,1)$, with probability of at least $1-\delta$, \edpm run with Alg.~\ref{alg:explore-lo} as the exploration rule (with $\kappa = \frac{\gamma}{8k m^2}$ and $\mu = k$), with
	$\gamma = ( T m^2 k^2 \,/\,\overline{\mathrm{Reg}}_{\mathrm{On}}^{\mathrm{H}}(\mathcal{Q}, T, \delta) )^{1/2}$ 
	satisfies
	\begin{align*}
		\mathrm{Reg}_T
		\leq O\!\left( \sqrt{ T \cdot m^2 k^2 \cdot
			\overline{\mathrm{Reg}}_{\mathrm{On}}^{\mathrm{H}}(\mathcal{Q}, T, \delta) } \right).
	\end{align*}
\end{cor}
For finite $\mathcal{Q}$, Vovk's aggregating algorithm gives
$O(\sqrt{T m^2 k^2 \log\lvert\mathcal{Q}\rvert})$; in the non-contextual setting with
$\mathcal{Q} = \Delta(d_{\mathcal{Q}})$, an exp-concave oracle~\citep{hazan07logarithmic} gives
$\tilde{O}(\sqrt{T m^2 k^2 d_{\mathcal{Q}}})$. Both match the minimax rate in $T$ for this class.
We next explain the key idea underlying the proof of Lemma~\ref{lemma_locally_observable_doec}, which hinges on a sparse relaxation of the objective $\mathcal{V}_{\gamma}\big(\hat{q}, p \big)$.

\subsubsection{Sparse relaxation of $\mathcal{V}_{\gamma}\big(\hat{q}, p \big)$}

We aim to find tractable relaxations of $\mathcal{V}_\gamma$ which sheds lights on finding efficient exploration strategies for approximately minimizing $\mathcal{V}_\gamma$. In particular, we aim to exploit the structured feedback mechanism of locally observable games to restrict exploratory actions to those lying along paths that connect pairs of actions whose loss differences are of interest. The neighborhood graph $(\Pi, \mathcal{E})$ is connected. Assign an arbitrary action $\mathrm{root}_{\mathcal{T}}$ as the root and choose a spanning in-tree  $\mathcal{T}$ (See Figure~\ref{fig:in-tree}) such that, there is a unique directed $\mathrm{path}_{\mathcal{T}}(a)$ from any action $a \in \Pi$ to $\mathrm{root}_{\mathcal{T}} \in \Pi$. Now, we can relax the $\mathcal{V}_{\gamma}\big(\hat{q}, p \big)$ as follows,
\begin{lemma}\label{lemma_locally_observable_doec_part_1}
	Let $\mathcal{G} = (\boldsymbol{L}, \boldsymbol{\Phi})$ be a locally observable partial monitoring game in the
	sense of Definition~\ref{local_observability_graph}.
	Let $\cT$ be an in-tree of the neighborhood graph. For all $p \in \Delta(\Pi)$, 
	we have 
	\begin{align}
		\mathcal{V}_{\gamma}\big(\hat{q}(x), p \big)
		\leq &
		\max_{\lambda \in \Delta(\Pi)} p^T \boldsymbol{L} \hat{q}(x) - \lambda^T \boldsymbol{L} \hat{q}(x)
		+ \frac{32}{\gamma} \max_{e \in \mathcal{E}} \lVert w_e \rVert^2_{\infty}  \sum_{(b \in \Pi,c \in \Pi):c \in \mathrm{path}_{\mathcal{T}}(b) } \left( \frac{\lambda_b}{ p_c }+  \frac{p_b}{p_c} \right).
		\label{eqn:V-ub-lo}
	\end{align}
\end{lemma}

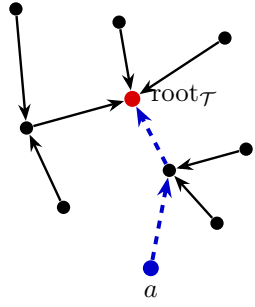
\begin{figure}[t]
	\centering
	\begin{tikzpicture}[scale=0.7,
		vertex/.style ={circle, fill=black, inner sep=0pt, minimum size=5pt},
		rootv/.style  ={circle, fill=red!85!black, inner sep=0pt, minimum size=6pt},
		srcv/.style   ={circle, fill=blue!80!black, inner sep=0pt, minimum size=6pt},
		tedge/.style  ={-{Stealth[length=2.2mm]}, line width=0.9pt},
		pedge/.style  ={-{Stealth[length=2.6mm]}, line width=1.5pt, blue!80!black, dashed},
		]
		\node[rootv] (r)  at ( 0.00,  0.00) {};
		\node[vertex](v1) at (-2.00, -0.55) {};
		\node[vertex](v2) at (-2.20,  1.70) {};
		\node[vertex](v3) at (-1.30, -2.05) {};
		\node[vertex](v4) at (-0.25,  1.45) {};
		\node[vertex](v5) at ( 1.75,  1.15) {};
		\node[vertex](b)  at ( 0.70, -1.35) {};
		\node[vertex](v6) at ( 2.15, -0.95) {};
		\node[vertex](v7) at ( 1.60, -2.35) {};
		\node[srcv]  (a)  at ( 0.35, -3.20) {};
		
		\node[anchor=west]  at ( 0.18,  0.08) {$\mathrm{root}_{\mathcal{T}}$};
		\node[anchor=north] at ( 0.35, -3.35) {$a$};
		
		\draw[tedge] (v2) -- (v1);
		\draw[tedge] (v3) -- (v1);
		\draw[tedge] (v1) -- (r);
		\draw[tedge] (v4) -- (r);
		\draw[tedge] (v5) -- (r);
		\draw[tedge] (v6) -- (b);
		\draw[tedge] (v7) -- (b);
		
		\draw[pedge] (a) -- (b);
		\draw[pedge] (b) -- (r);
	\end{tikzpicture}
	\caption{A spanning in-tree $\mathcal{T}$ of the neighborhood graph $(\Pi, \mathcal{E})$, with root
	$\mathrm{root}_{\mathcal{T}}$; every edge is directed toward the root. The dashed edges form
	$\mathrm{path}_{\mathcal{T}}(a)$, the directed path from action $a$ to the root. The construction follows
	\citet{lattimore2020exploration}.}
	\label{fig:in-tree}
\end{figure}
Proof is deferred to Appendix~\ref{app_section_efficient_local}. Above result is true for any in-tree with an arbitrary root. However, in general, we do not want the paths to be arbitrary: for a ``bad'' action $c$ and a ``good'' action $b$ (i.e., $e_c^\T L \hat{q}$ is large while $e_b^\T L \hat{q}$ is small), we would like $p_c$ to be small and $p_b$ to be large to ensure $p^T \boldsymbol{L} \hat{q}(x) - \lambda^T \boldsymbol{L} \hat{q}(x)$ to be small; however, if $c$ happens to lie in $\mathrm{path}_{\mathcal{T}}(b)$, that can cause the $\frac{\lambda_b}{p_c}$ term to be very large. How should the in-tree $\mathcal{T}$ be chosen so that the actions lying on each path are ones we are willing
to play? Our answer is to let the tree follow the estimated losses. We take $\mathcal{T}$ to be
$\boldsymbol{L}\hat{q}(x)$-decreasing (that is for a directed edge $e=(b,c)$, $e_b^\T \boldsymbol{L}\hat{q}(x) \geq e_c^\T \boldsymbol{L}\hat{q}(x)$; See Def~\ref{definition_decreasing_tree} for formal definition), so that estimated loss decreases along every path toward the root and,
in particular, the empirical best action sits at the root. 
With this choice of $\cT$, and the observation that $p$ according to \sqpm is monotonically decreasing of in the loss, the coverage term is upper bounded by $k \sum_b (\frac{\lambda_b}{p_b} + 1)$, which allows us to bound $\cV_\gamma(q,p)$ subsequently.

\begin{algorithm}[t]
\caption{\sqpm for Locally observable games}
\label{alg:explore-lo}
\begin{algorithmic}[1]
\STATE\textbf{parameters:}\\
		\STATE\spacex Context $x$ with predicted outcome distribution $\hat{q}(x)$ \\
		\STATE\spacex Exploration parameter $\gamma > 0$, $\mu = k$ and $\kappa = \frac{\gamma}{8k m^2} $ \\
        \STATE For each action $a \in \Pi$, compute $\hat{\ell}_{a} = e_a^\T \boldsymbol{L} \hat{q}(x)$ and let $b = \argmin_{a \in \Pi} \hat{\ell}_{t,a}$
		\STATE\label{lin_igw} For each $a \in \Pi \setminus \{b\}$, define $p_{a} = \frac{1}{\mu + \kappa (\hat{\ell}_{a} - \hat{\ell}_{b})}$; let $p_{b} = 1 - \sum_{a \in \Pi \setminus \{b\}} p_{a}$. 
		\RETURN $p$
\end{algorithmic}
\end{algorithm}

\paragraph{Offline oracle-efficient algorithms} 

In Appendix~\ref{app_section_offline}, we present a close variant of \edpm, namely \oedpm (Offline Estimation to Decisions for Partial Monitoring with Side Information), that efficiently reduces contextual partial monitoring to \emph{offline} estimation. Our \oedpm generalizes \textsc{FALCON}~\citep{simchilevi22bypassing} and its linear version~\citep{xu2020upper}.
Similar to \edpm, \oedpm has two versions, one that employs Alg.~\ref{alg:explore-go} for the globally observable setting, the other employs Alg.~\ref{alg:explore-lo} for the locally observable setting, which enjoys $O(T^{2/3})$ and $O(\sqrt{T})$ regret in the respective settings.
Compared with \edpm that makes $O(T)$ calls to the online outcome estimation oracle, \oedpm only makes $O(\log T)$ calls to the offline outcome estimation oracle, making it suitable for applications with low computational cost requirements. 

\section{Empirical studies}\label{main_section_experiments}

We evaluate \sqpm in two settings we discussed in Section~\ref{section_main_intro} (Examples 1 \& 2): LLM routing and contextual label-efficient prediction game. 

\paragraph{LLM routing.}
We evaluate the routing problem on a cached version of the BoolQ benchmark containing \(12{,}648\) valid responses from both models. We use \texttt{Qwen2.5-0.5B-Instruct} as the weak model and \texttt{Llama-2-13B-Chat} as the strong model. The routing outcome records whether their cached responses disagree. Each prompt is represented by a \(138\)-dimensional context comprising \(10\) weak-model uncertainty statistics, \(64\) PCA components of its hidden states, and \(64\) PCA components of a context-free prompt embedding.

We compare linear-logistic CBPSide~\citep{bartok2012CBPside, byeongwoo26efficient}, \sqpm (ours), ETC
(Explore-Then-Commit, adapted for apple tasting)~\citep{garivier16etc} with both linear logistic regression
and a nonlinear $15$-leaf histogram gradient-boosting (HGB) model, Bayesian-logistic
PG-TS~\citep{grant21apple}, and a random baseline. We consider nine values of $\ell_{01}$ between $1.8$ and
$3.3$ and evaluate each configuration over $20$ paired random permutations. \Cref{fig:llm_routing} reports the
mean realized total cost.

For each value of $\ell_{01}$ we tune every algorithm separately, selecting the configuration with the lowest
total cost. For all methods except the random baseline this amounts to choosing a multiplier
$s \in \{0.1, 0.3, 1, 3, 10\}$, which scales the confidence radius in CBPSide, the exploration parameter
$s \cdot \kappa$ in both \sqpm variants, the initial exploration budget $s \cdot T^{2/3}$ in ETC, and the prior
standard deviation in PG-TS. The random baseline instead routes each query to the strong model with a fixed
probability, independent of context, swept over $\{0, 0.2, 0.4, 0.6, 0.8, 1\}$. All methods, ours included,
are therefore tuned with hindsight knowledge of the loss. Further implementation details and results are
provided in \Cref{app_section_experiments}.

Across the evaluated loss settings, both the linear-logistic and nonlinear HGB instantiations of \sqpm achieve lower mean realized total cost than \textsc{CBPSide}, while the nonlinear HGB variant also outperforms PG-TS. These results illustrate a practical advantage of the oracle-efficient \sqpm formulation: it can incorporate flexible nonlinear prediction models without modifying the surrounding routing policy. Furthermore, generally ETC is expected to perform worse due to $O(T^{2/3})$ regret, however here due to limited test data size  \(12{,}648\) compared to larger dimension \(138\), it performs competitively. See Figure~\ref{fig:llm_routing}.

\begin{figure}[t]
	\centering
	\begin{subfigure}[b]{0.48\textwidth}
		\centering
		\includegraphics[width=\linewidth]{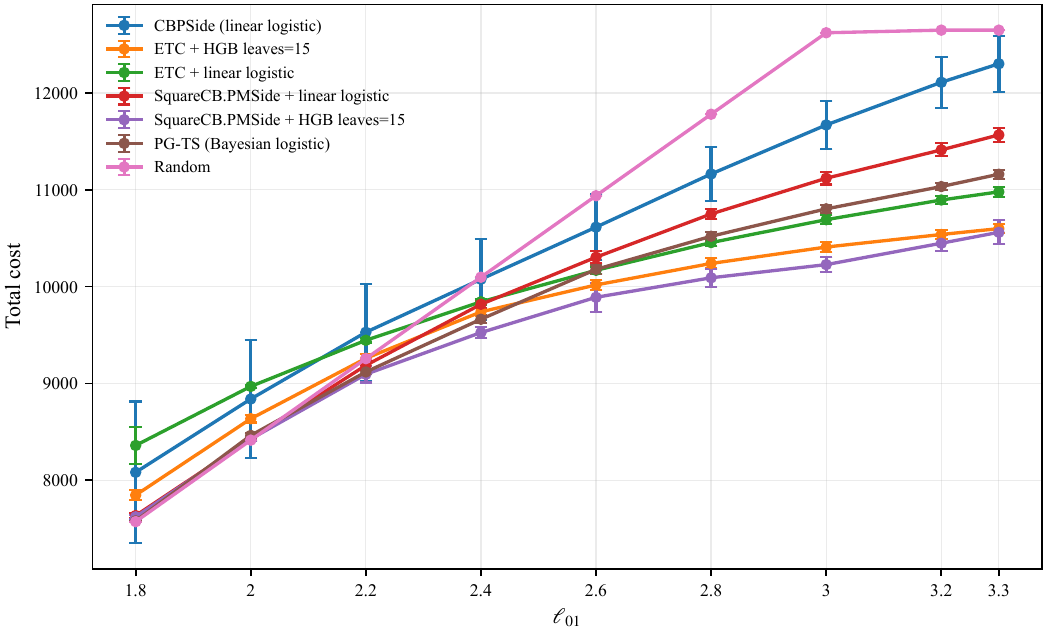}
		\caption{LLM routing: Total cost vs $\ell_{01}$}
		\label{fig:llm_routing}
	\end{subfigure}
	\hfill
	\begin{subfigure}[b]{0.45\textwidth}
		\centering
		\includegraphics[width=\linewidth]{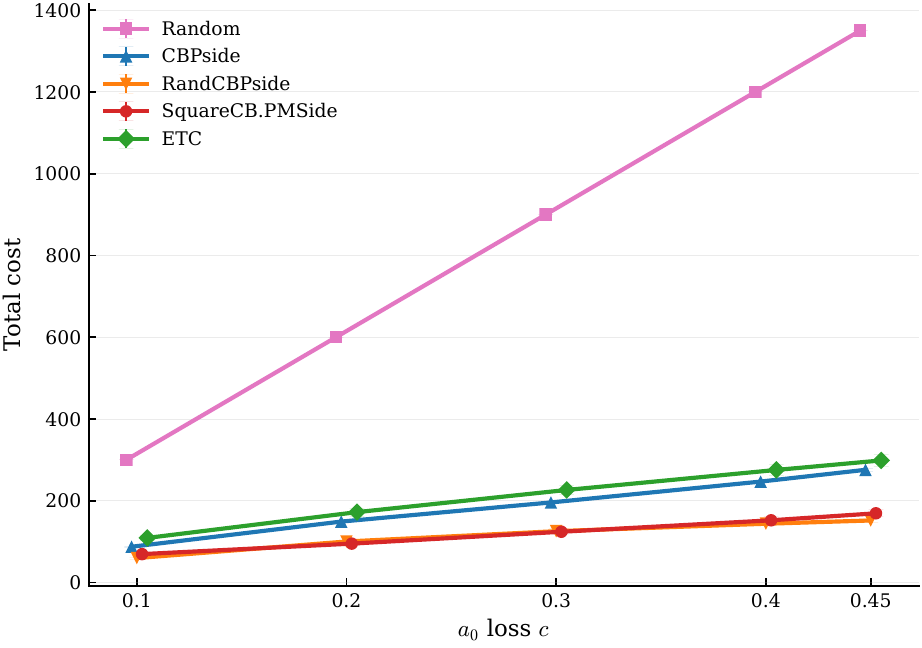}
		\caption{LE game: Total cost vs $c$}
		\label{fig:le_regret}
	\end{subfigure}
	\caption{Comparison of algorithms for and LLM routing and Label efficient prediction game}
	\label{fig:experiments}
\end{figure}

\paragraph{Contextual label-efficient prediction.}
We conduct an analogous study for the contextual label-efficient prediction game. Local observability requires the expert-query cost to satisfy $c<0.5$; accordingly, we evaluate five values of $c$ in the range $0.1 \leq c < 0.5$. Because PG-TS is designed specifically for apple-tasting games and is thus inapplicable to this problem, we replace it with RandCBPSide~\citep{heuillet24spm}. We generate the data synthetically from a linear logistic model with 10-dimensional context vectors. Further, details are provided in Appendix~\ref{app_section_experiments}. Our method consistently outperforms ETC and CBPSide across the evaluated settings and remains competitive with RandCBPSide, outperforming it for some of the expert-query costs. Moreover, RandCBPSide requires comparatively expensive confidence-width calculations, whereas our method is computationally cheaper. See Figure~\ref{fig:le_regret}.

\section{Conclusion and future work}

Our work gives the first efficient algorithms for stochastic contextual partial monitoring with optimal regret with general function approximation.
Experiments on LLM routing and a modified label-efficient game support these guarantees. Our analysis is restricted to non-degenerate games. Lemma~\ref{lem:bounds}, which underlies the $O(1/\gamma)$ bound
for locally observable games, relies on non-degeneracy. We conjecture that our \sqpm might work for degenerate games as well with minor modification. However, their analysis need a careful approach which is left to future work. Another interesting question is, can our algorithm be generalized to the setting when the loss matrix and the feedback matrix is context-dependent? This seems practically relevant for e.g. LLM routing applications, since the cost of wrongly following the weak model may depend on the importance of the prompt. Furthermore, extending our work to contextual linear partial monitoring~\citep{kirschner2023linear} and contextual partial monitoring with infinite outcomes~\citep{lattimore2022minimax} are also interesting future directions.

\subsection*{AI use statement}

In this work, we used generative AI tools for coding for experiments, writing those experiment details and polishing the writing. We did reviewed those coding implementation for correctness.
We have not used generative AI tools for writing proof or related works,
and the rest of the required disclosure tasks are not applicable to this work. We have reviewed all AI-assisted work. We take responsibility for the final content of this work,
including text, claims or artifacts produced with the aid of generative AI.

\subsection*{Ethics statement}
This work develops learning algorithms for routing queries among large language models. Such methods may reduce inference costs and computational resource consumption while improving access to capable models. However, routing decisions may propagate or amplify biases present in the underlying models and data, potentially producing unequal performance across demographic groups, languages, or application domains. Furthermore, queries may contain sensitive information, particularly when routed to externally hosted models. Our methods optimize a specified performance objective and do not independently guarantee fairness, privacy, reliability, or safety. Practical deployments should therefore incorporate appropriate data protection, subgroup evaluation, human oversight, and fallback mechanisms.

\subsection*{Reproducibility statement}
We provide complete definitions, assumptions, algorithm descriptions, and regret analyses in the main paper and appendices. The experimental section reports the datasets, preprocessing procedures, model configurations, baselines, evaluation metrics, and hyperparameters used in our study. We will release the implementation, experiment configurations, and evaluation scripts later to facilitate reproduction of our results, subject to the licenses and access conditions of the underlying datasets and language models. We also report results across multiple random seeds to account for experimental variability.

\subsubsection*{Acknowledgments}

This work was partly supported by the National Research Foundation of Korea (NRF) grant funded by the Korean government (MSIT) in 2026 (Grant No. RS-2026-25612350, Project Title: Development of an End-to-End Design and Validation Framework for Mathematically Principled Next-Generation AI Architectures and Learning Algorithms, contribution 45\%).
This work was partly supported by Institute of Information \& communications Technology Planning \& Evaluation (IITP) grant funded by the Korea government(MSIT) (RS-2019-II191906, Artificial Intelligence Graduate School Program(POSTECH), contribution 5\%; RS-2026-25616313, AI Star Fellowship(POSTECH, Interdisciplinary Program), contribution 50\%).
This work was supported in part by the National Science Foundation under grant CCF-2327013, IIS-2440266 (CAREER) and Meta Platforms, Inc.
\bibliography{library-shared}
\bibliographystyle{iclr2027_conference}

\newpage
\begin{center}
	\hrule height 4pt
	\vskip 0.25in
	
	{\fontsize{17}{20}\selectfont \bf Appendix}
	\vskip 0.25in 
	\hrule height 1pt
\end{center}
\newcommand{\apptoc}[1]{%
	\par\noindent
	\hyperref[#1]{Appendix~\ref*{#1}\quad \nameref*{#1}}%
	\ \dotfill\ \pageref{#1}%
}

\newcommand{\appsubtoc}[1]{%
	\par\noindent\hspace*{1.8em}
	\hyperref[#1]{\ref*{#1}\quad \nameref*{#1}}%
	\ \dotfill\ \pageref{#1}%
}

\appendix
\begingroup
\parskip 3pt
\apptoc{app_related_works}
\appsubtoc{doec_vs_ExO}
\appsubtoc{doec_vs_dec}
\apptoc{app_section_oracles}
\appsubtoc{app_section_online_oracles}
\appsubtoc{app_section_offline_estimation_oracles}
\apptoc{app_section_analysis_e2dpmside}
\appsubtoc{app_section_proof_thoerem_1}
\appsubtoc{app_section_efficient_global}
\appsubtoc{app_section_efficient_local}
\apptoc{app_section_offline}
\appsubtoc{app_section_oe2dpmside}
\appsubtoc{app_section_oe2d_regret_analysis}
\apptoc{app_section_experiments}
\apptoc{app_section_decreasing_tree}
\apptoc{section_lemmata_contextual}
\endgroup

\clearpage

\DeclarePairedDelimiter{\brk}{[}{]}
\DeclarePairedDelimiter{\crl}{\{}{\}}
\DeclarePairedDelimiter{\prn}{(}{)}
\DeclarePairedDelimiter{\nrm}{\|}{\|}
\DeclarePairedDelimiter{\tri}{\langle}{\rangle}
\DeclarePairedDelimiter{\dtri}{\llangle}{\rrangle}

\DeclarePairedDelimiter{\nrmlip}{\|}{\|_{\mathrm{Lip}}}

\newcommand{\Ep}{\wt{\bbE}}
\newcommand{\opt}{\operatorname{opt}}
\newcommand{\Eh}{\wh{\bbE}}
\newcommand{\bias}{\operatorname{bias}}
\newcommand{\sind}{\mathds{1}}
\newcommand{\algind}{\hspace{0.5cm}}
\newenvironment{simplealg}{\begin{mdframed}[roundcorner=1pt,backgroundcolor=black!5!white]\setlength{\parindent}{0cm}\setlength{\parskip}{0.2cm}\tt\vspace{-0.15cm}}{\end{mdframed}\vspace{-0.55cm}}
\setcitestyle{square}

\section{Additional Discussions on Important Related Works} \label{app_related_works}

\subsection{Exploration-by-optimization (ExO) versus our approach.}\label{doec_vs_ExO}

\citet{lattimore2020exploration} propose the exploration-by-optimization (ExO) algorithm for adversarial
non-contextual partial monitoring. They first form an exponentially weighted benchmark policy \edit{}{$\lambda_t$} over the Pareto
optimal actions from unbiased loss estimates,
\begin{align*}
	\lambda_{ta} = \frac{\sind_\Pi(a) \exp\left(-\eta \sum_{s=1}^{t-1} \hat y_{sa}\right)}
	{\sum_{b \in \Pi} \exp\left(-\eta \sum_{s=1}^{t-1} \hat y_{sb}\right)},
\end{align*}
where $\sind_\Pi$ is the indicator of $\Pi$ and $\hat y_t = G_t(A_t, \sigma_t) / p_{t A_t}$ is an
importance-weighted loss estimate, with $G_t$ (\Cref{app_eq_g_function}) obtained by solving \Cref{eq:opt} at
round $t$. This generalizes the EXP3 algorithm \citep{auer03nonstochastic} for adversarial bandits. 

They then improve this benchmark by solving
\begin{equation}
	\label{eq:opt}
	\begin{alignedat}{3}
		&\underset{G \in \cH,\, p \in \cP}{\text{minimise}} \qquad && \max_{o \in [d]} \Bigg[
		(p-\lambda)^\top \boldsymbol{L} e_o + \bias_{\lambda}(G ; o)
		+ \frac{1}{\eta} \sum_{a=1}^k p_a \Psi_{\lambda}\!\left(\frac{\eta G(a, \Phi_{ao})}{p_a}\right)
		\Bigg],
	\end{alignedat}
\end{equation}
where $\eta$ plays the role of our $\gamma$,
$\Psi_q(z) = \langle q, \exp(-z) + z - 1\rangle$, and
\begin{align*}
	\bias_{\lambda}(G ; o) = \left\langle \lambda,\, \boldsymbol{L} e_o - \sum_{a=1}^k G(a, \Phi_{ao})\right\rangle
	+ \max_{c \in \Pi} \left(\sum_{a=1}^k G(a, \Phi_{ao})_c - L_{co}\right).
\end{align*}

We may write this equivalently, following the notation of \citet{lattimore2021mirror}, as
\[
\underset{G \in \cH,\, p \in \cP}{\text{minimise}} \ \max_{o \in [d],\, c \in \Pi} \Bigg[
(p-\lambda)^\top \boldsymbol{L} e_o + \inner{ e_{c} - p }{ \sum_{a=1}^k G(a, \Phi_{ao}) - \boldsymbol{L} e_o }
+ \frac{1}{\eta} \sum_{a=1}^k p_a \Psi_\lambda \!\left(\frac{\eta G(a, \Phi_{ao})}{p_a}\right)
\Bigg].
\]
This step reallocates probability away from the benchmark policy and toward actions whose feedback is
informative. Writing $\opt_\lambda(\eta)$ for the optimal value of \Cref{eq:opt}, the analogue of our
$\mathcal{V}_\gamma$, \citet{lattimore2020exploration} bound it for both locally and globally observable
games without exhibiting an explicit or approximate solution to the optimization problem.

They first restrict the class of estimators $G$. Let
\[
\cH_0 = \cbr{ G \in \cH: \inner{e_b - e_c}{ \sum_{a=1}^k G(a, \Phi_{ao}) - \boldsymbol{L} e_o } = 0
	\ \ \text{for all } b, c }
\]
denote the unbiased estimators, for which $\bias_{\lambda}(G ; o) = 0$ for every $o$. The minimax objective is
then bounded above by
\[
\underset{G \in \cH_0, \, p \in \cP}{\min} \ \max_{o \in [d]} \Bigg[
(p- \lambda)^\top \boldsymbol{L} e_o
+ \frac{1}{\eta} \sum_{a=1}^k p_a \Psi_{\lambda}\!\left(\frac{\eta G(a, \Phi_{ao})}{p_a}\right)
\Bigg],
\]
and, since the bracketed expression is linear in $e_o$, maximizing over outcomes is equivalent to maximizing
over outcome distributions:
\begin{align}
	\underset{G \in \cH_0, \, p \in \cP}{\min} \ \max_{q \in \cD} \Bigg[
	(p- \lambda)^\top \boldsymbol{L} q
	+ \frac{1}{\eta} \sum_{o=1}^d q_o \sum_{a=1}^k p_a
	\Psi_{\lambda}\!\left(\frac{\eta G(a, \Phi_{ao})}{p_a}\right)
	\Bigg].
	\label{eqn:exo-obj}
\end{align}
\Cref{eqn:exo-obj} is concave in $q$ and convex in $(G, p)$, so to (nonconstructively) upper bound the value of \Cref{eqn:exo-obj}, Sion's minimax theorem permits
exchanging the two operations. Writing $F(G, p, q)$ for the bracketed expression,
\[
\underset{G \in \cH_0,\, p \in \cP}{\min} \ \max_{q \in \cD} F(G, p, q)
\;=\; \max_{q \in \cD} \ \underset{G \in \cH_0,\, p \in \cP}{\min} \ F(G, p, q).
\]
To upper bound this expression, it now suffices to choose $G$ and $p$ as functions of $q$ and $\lambda$: they build an
$\boldsymbol{L}q$-decreasing tree, take $G$ to be the loss estimator induced by that tree, and set
$p = W_{q}(\lambda)$, where $W_q$ is the water-transfer operator (see \Cref{lemma_water_transfer} for a recap). The operator
redistributes the benchmark policy $\lambda$ so that the resulting $p$ incurs no greater loss under $q$, while
allocating probability that increases toward the root of the tree. 

In contrast, \emph{given an outcome distribution $q$, }
our exploration strategy approximately solves the following optimization problem:
\[
    \min_{p \in \cP} \max_{\lambda \in \Delta(\Pi)} (p - \lambda)^\T \boldsymbol{L} \hat{q} 
    +  
    \frac{1}{\gamma} \cdot \mathrm{Coverage}(p, \lambda)
\]
Given that for any $G \in \cH_0$, Lemma~\ref{lem:coverage-g-ub} implies that
$\mathrm{Coverage}(p, \lambda) \leq 
\sum_a \sum_b \frac{\lambda_b + p_b}{p_a} \| G_b(a,:)\|_\infty^2
$, and thus it suffices to find exploration distribution $p$ that solves, 
\begin{align}
    \min_{p \in \cP} \max_{\lambda \in \Delta(\Pi)} (p - \lambda)^\T \boldsymbol{L} \hat{q} 
    +  
    \frac{1}{\gamma} \cdot \del{ \sum_a \sum_b \frac{\lambda_b + p_b}{p_a} \| G_b(a,:)\|_\infty^2 }
    \label{eqn:relaxed-doec-obj}
\end{align}

Comparing \Cref{eqn:exo-obj} with \Cref{eqn:relaxed-doec-obj}, the two objectives share a shape: each
consists of a greediness term of the form $(p - \lambda)^\T \boldsymbol{L} \hat{q}$ and a coverage term
measuring how well $p$ covers $\lambda$. They differ in three respects:
\begin{itemize}
\item  First, in the minimax players: we minimize over exploration policies against a maximizing comparator policy
$\lambda$, whereas \Cref{eqn:exo-obj} minimizes against a maximizing adversarial outcome distribution  $q$.

\item Second, in the form of the coverage term. The one in \Cref{eqn:exo-obj} is exponential,
$\Psi_{\lambda}\!\left(\eta G(a, \Phi_{ao}) / p_a\right)$, and \citet{lattimore2020exploration} relax it to
$\eta^2 \sum_b \lambda_b \lVert G(a, \cdot)_b \rVert_\infty^2 / p_a^2$, at the cost of requiring $p$ to be
bounded below entrywise by $\eta \max_{a,b} \lVert G(a, \cdot)_b \rVert_\infty$.

\item Third, in what the analysis delivers. ExO is tractable but gives no explicit exploration policy $p$. For our
objective we show that the inverse-gap-weighting rule of \sqpm certifies a small value of $\mathcal{V}_\gamma$
directly: it already allocates probability increasing toward the root of an $\boldsymbol{L}\hat{q}$-decreasing
tree, so neither the optimization over $G$ nor the water-transfer step is needed, and the in-tree only appears
in the analysis.
\end{itemize}
We also refer the reader to Table~\ref{tab:exo-comparison} for a summary. 

\begin{table}[h!]
	\centering
	\caption{Comparison with exploration-by-optimization \citep{lattimore2020exploration}. $\uparrow$ indicates a better result. NA indicates not applicable}
	\label{tab:exo-comparison}
	\begin{tabular}{lll}
		\toprule
		& \textbf{ExO} & \textbf{Ours} \\
		\midrule
		Side information      & No                                  & Yes $\uparrow$ \\
		Outcome distribution  & Adversarial	$\uparrow$	 			& Stochastic 
        \\
		Objective 			  & $\opt_q(\eta)$						& $\mathcal{V}_\gamma$\\
		Minimax players       & Policy vs.\ adversarial outcome $x$ & Policy vs.\ comparator policy $\lambda$ \\
		Optimization variables& $G$ and $p$ jointly                 & $p$ only $\uparrow$\\
		Per-round computation & Exact solution of \Cref{eq:opt}    & IGW for locally observable;\\
		&&$\eps$-greedy for globally observable $\uparrow$\\
		In-tree               & Optimized at every round          	& Analysis only $\uparrow$\\
		Water-transfer        & Required                            & Not required $\uparrow$\\
		Oracle-efficient		& No								& Efficient $\uparrow$\\
		Offline oracle compatibility &NA							& Compatible $\uparrow$\\
		Locally observable regret & ---                             & $\sqrt{k}$ smaller $\uparrow$\\
		\bottomrule
	\end{tabular}
\end{table}
\subsection{Relation to the decision-estimation coefficient}\label{doec_vs_dec}

The decision-estimation coefficient (DEC) was introduced by \citet{foster21thestatistical} as a complexity
measure for stochastic sequential decision making. \citet{foster2022complexity} extend it to the adversarial
setting, showing that the DEC applied to the convex hull of the model class characterizes the fundamental statistical complexity of that setting as well.
On the algorithmic side they propose $\mathrm{ExO}^{+}$, a high-probability variant of the
exploration-by-optimization algorithm of \citet{lattimore2020exploration, lattimore2021mirror}. Whereas the
original ExO gives in-expectation guarantees, $\mathrm{ExO}^{+}$ gives high-probability guarantees that hold
against adaptive adversaries. Like ExO, the exploration rule and the loss estimator are defined through an optimization problem rather than a closed-form
policy.

Their analysis does not bound the $\mathrm{ExO}^{+}$ or DEC objective directly for locally and globally
observable games. Instead they establish an equivalence, for general divergences $D$, between the convexified
DEC, a parameterized variant of the generalized information ratio of \citet{lattimore2021mirror}, and the
$\mathrm{ExO}^{+}$ objective. Since
\citet{lattimore2019information, lattimore2021mirror, lattimore2022minimax} bound the regret for locally and
globally observable games in terms of the generalized information ratio under the KL divergence, in agreement
with the classification theorem \citep{bartok11pmminimax, lattimore2020exploration}, those bounds transfer to
$\mathrm{ExO}^{+}$ through this equivalence. The equivalence itself holds for general $D$, but the
information-ratio bounds are established for the KL divergence; we are not aware of corresponding bounds for
other divergences, including the squared Hellinger.

To summarize, \citet{foster2022complexity}'s result implicitly implies that the E2D algorithm~\cite{foster21thestatistical} has optimal order-in-$T$ regret for stochastic contextual partial monitoring. However, 
as we discussed in Section~\ref{sec:e2d}, the E2D algorithm may be computationally inefficient. Our paper addresses the computational challenge.
For locally observable games we give
a closed-form inverse-gap-weighting rule (\sqpm), and for globally observable games a closed-form
$\varepsilon$-greedy rule (\egqpm). Both extend to be compatible with offline estimation oracles, because the exploration term of
our objective, the $\mathrm{coverage}$, does not involve the current estimate.

\section{Online and offline estimation oracles}\label{app_section_oracles}

\subsection{Online Oracles}\label{app_section_online_oracles}

In this section, we provide examples for the online estimation oracles \(\mathrm{Alg}_{\mathrm{On}}\) and prove their bounds for both finite and infinite model classes.

\begin{lemma}[Finite-class online Hellinger guarantee via Vovk's algorithm]
	\label{lem_vovk_online_estimation_finite_contextual}
	Suppose that the candidate conditional outcome-model class
	\[
	\mathcal Q
	\subseteq
	\left\{
	q:\mathcal X\to\Delta(d)
	\right\}
	\]
	is finite and realizable, so that \(q^*\in\mathcal Q\) and the oracles are allowed to output predictors in $\mathrm{co}(\mathcal{Q})$. Then there
	exists a contextual online estimation oracle
	\(\mathrm{Alg}_{\mathrm{On}}\) such that,
	\begin{align*}
		\overline{\mathrm{Reg}}^{\mathrm{H}}_{\mathrm{On}}(\mathcal{Q}, T, \delta)
		=
		2\log\left(\frac{|\mathcal Q|}{\delta}\right).
	\end{align*}
	The contextual online estimation oracle can be instantiated by applying
	Vovk's \citep{vovk1995game} aggregating algorithm to the conditional outcome models in
	\(\mathcal Q\).
\end{lemma}
\begin{proof}
	Define $\mathcal{G} \subseteq \Delta(m)$ as a set of functions $g_q : \mathcal{X} \times [k] \to \Delta(m)$ induced by $q$, formally
	
	\begin{align*}
		\mathcal{G} = \{ g_q: g_q(\sigma \mid x ,a )  = \boldsymbol{S}_a q(x)[\sigma] \;\; \forall q \in \mathcal{Q}\}
	\end{align*}

	Hence, 
	
	\begin{align*}
		\lvert \mathcal{G} \rvert \leq \lvert \mathcal{Q} \rvert
	\end{align*}

	Specifically,

	\begin{align*}
		g_{q^*}(\sigma \mid x_t , A_t ) &= \PP\left(\sigma_t = \sigma \mid x_t, A_t\right) \\
		&=  S_{A_t} q^*(x_t)[\sigma]
	\end{align*}
	
	Vovk's algorithm calculate the $g_{\hat{q}_t}$ in the following way.
	
	Let, 
	\begin{align*}
		\ell^{(t)}_{\mathrm{log}}(g_q) &:= \log \left(\frac{1}{g_q ( \sigma_t \mid x_t, A_t ) }\right) \\
		&=  \log \left(\frac{1}{ \left(S_{A_t}  q(x_t) \right)[\sigma_t] }\right).
	\end{align*}

    Then $\forall g_q \in \mathcal{G}$, the algorithm calculates a probability simplex $b_t \in \Delta(\lvert \mathcal{G} \rvert)$ as follows.
	\begin{align*}
		b_t(g_q) \propto \exp \left( - \sum_{s=1}^{t-1} \ell_{\mathrm{log}}^{(s)} (g_q) \right)
	\end{align*}
	
	Estimate $g_{\hat{q}_t}$ 
	\begin{align*}
		g_{\hat{q}_t}(\cdot) = \EE_{g \sim b_t } \sbr{ g(\cdot)}
	\end{align*}
	
	Let the corresponding $\hat{q}_t$  such that, 
	\begin{align*}
		\hat{q}_t(\cdot) = \EE_{g \sim b_t } \sbr{ q(\cdot)}.
	\end{align*}
	
	 Then Vovk's algorithm guarantees,
	 
	  \[
	 \sum_{t=1}^T \log\frac{1}{g_{\hat{q}_t}(\sigma_t \mid x_t, a_t)} 
	 - 
	 \sum_{t=1}^T \log\frac{1}{g_{q^*}(\sigma_t \mid x_t, a_t)} 
	 \leq 
	 \log|\cG|    
	 \]
	 
	Then, by applying Lemma A.14 from \citet{foster21thestatistical}, the output estimator of Vovk's aggregating algorithm $g_{\hat{q}_t}$ satisfies, with probability of at least $1-\delta$,
	
	\begin{align} \label{foster_result}
		\sum_{t=1}^{T} \EE_{t-1} \sbr{D^2_{\mathrm{H}}\left( g_{\hat{q}_t}(\cdot \mid x_t, A_t) \mid g_{q^*}(\cdot \mid x_t, A_t)\right) } \leq 2\log \frac{\lvert \mathcal{G}\rvert}{\delta} 
	\end{align}
	
	transforming this to our case, with probability at least $1-\delta$, $\hat{q}_t$ satisfies,
	\begin{align*}
		\sum_{t=1}^{T} \EE_{a \sim p_t} \sbr{D^2_{\mathrm{H}}\left( \boldsymbol{S}_a \hat{q}_t(x_t) \mid \boldsymbol{S}_a q^*(x_t)\right) } \leq 2\log \frac{\lvert \mathcal{Q}\rvert}{\delta}.
	\end{align*}

	\begin{remark} \label{iidremark}
		Define filtration $\mathcal{H}$ as,
		\begin{align*}
			\mathcal{H}_{t-1} = \sigma \left( x_1, A_1, \sigma_1, \cdots, x_{t-1}, A_{t-1},\sigma_{t-1} \right)
		\end{align*}
		Equation \eqref{foster_result} also implies that, if $x_1, \cdots, x_T \stackrel{\mathrm{i.i.d.}}{\sim} \mathcal{D}_{x}$ and $p_t, \hat{q}_t$ are $\mathcal{H}_{t-1}$ measurable, then with probability of at least $1-\delta$,
		\begin{align*}
			\sum_{t=1}^{T} \EE_{x \sim \mathcal{D}_{x}, a \sim p_t(\cdot \mid x)} \sbr{D^2_{\mathrm{H}}\left( \boldsymbol{S}_a \hat{q}_t(x) \mid \boldsymbol{S}_a q^*(x)\right) } \leq 2\log \frac{\lvert \mathcal{Q}\rvert}{\delta}.
		\end{align*}
	\end{remark}
\end{proof}

\begin{lemma}[Infinite-class online Hellinger guarantee via contextual smoothed aggregation]
	\label{lem_vovk_online_estimation_infinite_contextual}
	Suppose that the candidate conditional outcome-model class is,
	\[
	\mathcal Q
	\subseteq
	\left\{
	q:\mathcal X\to\Delta(d)
	\right\},
	\]
	is possibly infinite but has finite covering numbers under the uniform Hellinger metric and that the model is realizable \(q^*\in\mathcal{Q}\) and the oracles are allowed to output predictors in $\mathrm{co}(\mathcal{Q})$.
	
	Define $\mathcal{G} \subseteq \Delta(m)$ as a set of functions $g_q : \mathcal{X} \times [k] \to \Delta(m)$ induced by $q$, formally
	
	\begin{align*}
		\mathcal{G} = \{ g_q: g_q(\sigma \mid x ,a )  = \boldsymbol{S}_a q(x)[\sigma] \;\;  q \in \mathcal{Q}\}
	\end{align*}
	Moreover, let us define an \eps-cover $\mathcal{J}$, such that, 
	\begin{align*}
		\forall g_q \in \mathcal{G} \spacex \exists g'_q \in \mathcal{J} \text{ such that } \sup_{x \in \mathcal{X}, a \in [k] } D^2_{\mathrm{H}} \left( g_q'(\cdot \mid x,a) \mid g_q(\cdot \mid x,a)\right) \leq \eps^2
	\end{align*}
	Furthermore, let $\mathcal{N}(\mathcal{G}, \eps )$ denote the size of the smallest such cover. Define,
	\begin{align*}
		\mathrm{est}(\mathcal{G}, T) := \inf_{\eps \geq 0} \{ \log(\mathcal{N}(\mathcal{G}, \eps )) + \eps^2T \}
	\end{align*}
	Then there exists a contextual online estimation oracle
	\(\mathrm{Alg}_{\mathrm{On}}\) such that,
	\begin{align*}
		\overline{\mathrm{Reg}}^{\mathrm{H}}_{\mathrm{On}}(\mathcal{Q}, T, \delta) \leq O\left(\log(2m^2T) \cdot \mathrm{est}(\mathcal{G}, T ) + \log^2(2m^2T)\log\frac{1}{\delta} \right).
	\end{align*} 
	The contextual online estimation oracle can be instantiated by applying
	an online aggregation algorithm \cite{vovk1995game, foster21thestatistical} to a suitably discretized and smoothed
	approximation of the parameter space $\mathcal{Q}$.
\end{lemma}

\begin{proof}

	The regularity assumptions required in Assumption A.2 of \citet{foster21thestatistical} are already satisfied by setting $B =  \max \{m, e\} $
	\begin{enumerate}
		\item Number of possible outcome for a context is upper bounded by $m \leq B$
		\item $g_q(\sigma \mid x, a)$ is a probability simplex ($1 \leq B$).
	\end{enumerate}
	
	Smoothed Vovk's algorithm calculate the $g_{\hat{q}_t}$ in the following way,
	let,
	\begin{align*}
		h(\sigma \mid x, a) &= (1 -\alpha) g_q(\sigma \mid x, a) + \frac{\alpha}{m}
	\end{align*}
	then $\forall g_q \in \mathcal{J}$ calculate the a probability simplex $b_t \in \Delta(\lvert \mathcal{J} \rvert)$ as follows
	\begin{align*}
		b_t(g_q) \propto \exp \left( - \sum_{s=1}^{t-1} \ell_{\mathrm{log}}^{(s)} (h) \right)
	\end{align*}
	Estimate $\hat{g}_t$ 
	\begin{align*}
		g_{\hat{q}_t} = \EE_{g_q \sim b_t } \sbr{ g_q}
	\end{align*}
	
	Hence,
	\begin{align*}
		\hat{q}_t = \EE_{q \sim b_t } \sbr{ q}
	\end{align*}

	Then by applying Lemma A.16 from \citet{foster21thestatistical}, with probability of at least $1-\delta$,
	\begin{align}\label{foster_result2}
		\sum_{t=1}^{T} \EE_{t-1} \sbr{D^2_{\mathrm{H}}\left( g_{\hat{q}_t}(\cdot \mid x_t, A_t) \mid g_{q^*}(\cdot \mid x_t, A_t)\right) } \leq O\left(\log(2m^2T) \cdot \mathrm{est}(\mathcal{G}, T ) + \log^2(2m^2T)\log\frac{1}{\delta} \right)
	\end{align}
	
	Let $\hat{q}_t \in \mathcal{Q}$ be the corresponding solution, with probability of at least $1-\delta$, $\hat{q}_t$ satisfies,
	\begin{align*}
		\sum_{t=1}^{T} \EE_{a \sim p_t} \sbr{D^2_{\mathrm{H}}\left( \boldsymbol{S}_a \hat{q}_t(x_t) \mid \boldsymbol{S}_a q^*(x_t)\right) } \leq  O\left(\log(2m^2T) \cdot \mathrm{est}(\mathcal{G}, T ) + \log^2(2m^2T)\log\frac{1}{\delta} \right).
	\end{align*}

	\begin{remark} \label{iidremark2}
		Define filtration $\mathcal{H}$ as,
		\begin{align*}
			\mathcal{H}_{t-1} = \sigma \left( x_1, A_1, \sigma_1, \cdots, x_{t-1}, A_{t-1},\sigma_{t-1} \right)
		\end{align*}
		Equation \eqref{foster_result} also implies that, if $x_1, \cdots, x_T \stackrel{\mathrm{i.i.d.}}{\sim} \mathcal{D}_{x}$ and $p_t, \hat{q}_t$ are $\mathcal{H}_{t-1}$ measurable, then with probability of at least $1-\delta$,
		\begin{align*}
			\sum_{t=1}^{T} \EE_{x \sim \mathcal{D}_{x}, a \sim p_t(\cdot \mid x)} \sbr{D^2_{\mathrm{H}}\left( \boldsymbol{S}_a \hat{q}_t(x) \mid \boldsymbol{S}_a q^*(x)\right) } \leq  O\left(\log(2m^2T) \cdot \mathrm{est}(\mathcal{G}, T ) + \log^2(2m^2T)\log\frac{1}{\delta} \right).
		\end{align*}
	\end{remark}
\end{proof}

\subsection{Offline estimation oracle}\label{app_section_offline_estimation_oracles}

In this section, first we define \textit{offline estimation oracle} and  we provide examples for the offline estimation oracles \(\mathrm{Alg}_{\mathrm{Off}}\) and prove their estimation bounds. 

\begin{definition}[Offline estimation oracle]
	An \emph{offline estimation oracle} $\mathrm{Alg}_{\mathrm{Off}}$ is an algorithm that estimates the unknown
	parameter $q^\ast \in \mathcal{Q}$ from the observed data in a batch fashion. That is, after observing
	$t$ rounds, given the feedback symbols collected so far, the oracle produces an estimate
	\begin{align*}
		\hat{q} = \mathrm{Alg}_{\mathrm{Off}}\!\left(\big\{(A_s, x_s, \sigma_s)\big\}_{s=1}^{t}\right).
	\end{align*}
\end{definition}

\begin{assumption}[Hellinger-closeness of the offline estimation oracle] \label{assumption_hellinger_closeness_offline_contextual}
	For any $t \in \mathbb{N}$, let $x_1, \dots, x_t \overset{\text{i.i.d.}}{\sim} \mathcal{D}_x$ and let $\hat{q}$
	denote the estimate returned by $\mathrm{Alg}_{\mathrm{Off}}$ on the resulting sample. Define $\mathrm{Reg}_{\mathrm{Off}}^{\mathrm{H}}\left( \mathcal{Q}, t \right)$, such that
	\begin{align*}	
	 \EE_{x \sim \mathcal{D}_x,\, a \sim p(\cdot \mid x)}
		\sbr{ D^2_{\mathrm{H}}\big( \boldsymbol{S}_a\hat{q}(x) \,\big\|\, \boldsymbol{S}_a q^\ast(x) \big) } = \frac{\mathrm{Reg}_{\mathrm{Off}}^{\mathrm{H}}\left( \mathcal{Q}, t \right)}{t}
	\end{align*}
	where $p$ denotes the behavioral policy used to collect the sample, and the expectation is taken over a fresh
	draw $(x, a)$, conditionally on $\hat{q}$. We assume that the offline estimation oracle
	$\mathrm{Alg}_{\mathrm{Off}}$ guarantees, for any $\delta \in (0,1)$, that with probability at least $1-\delta$,
	\begin{align*}
		\mathrm{Reg}_{\mathrm{Off}}^{\mathrm{H}}\left( \mathcal{Q}, t \right)
		\leq \overline{\mathrm{Reg}}_{\mathrm{Off}}^{\mathrm{H}}\left( \mathcal{Q}, t, \delta \right),
	\end{align*}
	where $\overline{\mathrm{Reg}}_{\mathrm{Off}}^{\mathrm{H}}\left( \mathcal{Q}, t, \delta \right)$ is a known,
	deterministic upper bound.
\end{assumption}

Having established in Appendix~\ref{app_section_online_oracles} the existence of online oracles with favorable
guarantees for both finite and infinite model classes, we can always obtain offline oracles for both by
online-to-batch conversion. Beyond this generic construction, we also show that the conditional
maximum-likelihood estimator satisfies Assumption~\ref{assumption_hellinger_closeness_offline_contextual} with a
favorable bound $\overline{\mathrm{Reg}}_{\mathrm{Off}}^{\mathrm{H}}(\mathcal{Q}, t, \delta)$.

\newtheorem*{app_lem_mle_offline_estimation_finite_contextual}{Lemma \ref{lem_mle_offline_estimation_finite_contextual}}

\begin{lemma}[Existence of a contextual offline estimation oracle for a finite class]\label{app_lemma_mle_offline_estimation_finite_contextual}
	Suppose that the candidate conditional outcome-model class
	\[
	\mathcal Q
	\subseteq
	\left\{
	q:\mathcal X\to\Delta(d)
	\right\}
	\]
	is finite and realizable, so that \(q^*\in\mathcal Q\). Then there
	exists a contextual offline estimation oracle
	\(\mathrm{Alg}_{\mathrm{Off}}\) such that
	Assumption~\ref{assumption_hellinger_closeness_offline_contextual}
	is satisfied with
	\begin{align*}
		\overline{\mathrm{Reg}}_{\mathrm{Off}}^{\mathrm{H}}\left( \mathcal{Q}, t, \delta \right) 
		=
		2\log\left(\frac{|\mathcal Q|}{\delta}\right).
	\end{align*}
	The contextual offline estimation oracle can be instantiated using the
	conditional maximum-likelihood estimator
	\begin{align*}
		\hat q_t
		\in
		\argmax_{q\in\mathcal Q}
		\sum_{s=1}^{t}
		\log
		\left(
		\left(S_{A_s}q(x_s)\right)[\sigma_s]
		\right).
	\end{align*}
	\begin{proof}
		Let 
		\begin{align*}
			z_s &= (x_s, A_s) \tag{Where $x_s \sim \mathcal{X}$(i.i.d) and $A_s \sim p_s \left( z_{1:s-1}, y_{1:s-1}, x_s\right)$}\\
			y_s &= \sigma_s
		\end{align*}
		Hence,
		\begin{align*}
			z_s \sim  \mathcal{D}_s \left( z_{1:s-1}, y_{1:s-1}\right)  \tag{for some distribution $\mathcal{D}_s$}
		\end{align*}
		
		Define
		\begin{align*}
			f^*(z_s, y_s) &= \PP(y_s \mid z_s ) \\
			&= \PP (\sigma_s \mid A_s, x_s) \\
			&= \PP (\sigma_s \mid S_{A_s} q^*(x_s)) \\
			&=  S_{A_s} q^*(x_s)[\sigma]
		\end{align*}
		
		Furthermore, define
		\begin{align*}
			\mathcal{F} = \{ f: f(z, y) = \boldsymbol{S}_a q(x)[\sigma] \;\; \forall q \in \mathcal{Q}\}
		\end{align*}
		Hence,
		\begin{align*}
			\lvert \mathcal{F} \rvert \leq \lvert \mathcal{Q} \rvert
		\end{align*}
		Following the work of \citet{agarwal2020flambe} (section E)
		
		\begin{align*}
			\hat{f} &:= \argmax_{f \in \mathcal{F}} \sum_{s=1}^{t} \log f(z_s, y_s) \\
			&= \argmax_{q}   \sum_{s=1}^{t} \PP (\sigma_s \mid S_{A_s} q(x_s))
		\end{align*}
		Therefore, let
		\begin{align*}
			\hat{q}_t &= \argmax_q \sum_{s=1}^{t} \log \left( \PP \left( \sigma_s  \mid x_s, A_s, q \right)\right) \\
			&=  \argmax_q \sum_{s=1}^{t} \log \left( S_{A_s} \cdot q(x_s) [\sigma_s ]  \right) \\
		\end{align*}
		then,
		\begin{align*}
			\hat{f}(z, y) =  \boldsymbol{S}_a \hat{q}_t(x)
		\end{align*}
		
		By Theorem 21 of \citet{agarwal2020flambe}, this maximum likelihood solution, satisfies the following,
		\begin{align*}
			\sum_{s=1}^{t} \EE_{z \sim D_s} \big \lVert \hat{f} (z, \cdot) - f^*(z, \cdot )\big\rVert_{\mathrm{TV} }^2  \leq 2 \log \left( \frac{\lvert \mathcal{F} \rvert}{\delta}\right)
		\end{align*}
		
		Although Theorem 21 is stated in terms of squared total variation distance, Remark 23 observes that its proof directly establishes the corresponding squared-Hellinger bound.

		\begin{align*}
			\sum_{s=1}^{t} \EE_{z \sim D_s} \sbr{D^2_{\mathrm{H}} \left(  \hat{f} (z, \cdot) \big\mid  f^*(z, \cdot ) \right) } \leq 2 &\log \left( \frac{\lvert \mathcal{F} \rvert}{\delta}\right) \\
		\end{align*}
		
		Therefore, in our case
		\begin{align*}
			\sum_{s=1}^{t} \EE_{x \sim \mathcal{X}, a \sim p_s } \sbr{D^2_{\mathrm{H}} \left( \boldsymbol{S}_a \cdot \hat{q}_t(x) \big\mid \boldsymbol{S}_a \cdot q^*(x) \right) } \leq 2 &\log \left( \frac{\lvert \mathcal{Q} \rvert}{\delta}\right) \\
		\end{align*}
	\end{proof}
\end{lemma}

\section{Analysis of \edpm}\label{app_section_analysis_e2dpmside}

\subsection{Proof of Theorem~\ref{theorem_doec_regret}}\label{app_section_proof_thoerem_1}

\newtheorem*{app_theorem_doec_regret}{Theorem \ref{theorem_doec_regret}}
\begin{app_theorem_doec_regret}[Regret of \edpm]\label{app_theorem_doec_regret}
	Suppose $\mathrm{Alg}_{\mathrm{On}}$ has a Hellinger regret bound $\overline{\mathrm{Reg}}_{\mathrm{On}}^{\mathrm{H}}\left( \mathcal{Q}, T, \delta \right)$. Fix $\gamma > 0$. Then, for any $\delta \in (0,1)$,  with probability at least $1-\delta$,
\begin{align*}
	\mathrm{Reg}_T
	\leq T \cdot \max_{t=1}^T \mathcal{V}_{\gamma}\big(\hat{q}_t(x_t), p_t \big)
	+ \gamma \cdot \overline{\mathrm{Reg}}_{\mathrm{On}}^{\mathrm{H}}\left( \mathcal{Q}, T, \delta \right).
\end{align*}
\end{app_theorem_doec_regret}
	\begin{proof}
	Let 
	\begin{align*}
		\lambda^*_{x_t} =  \argmin_{p' \in \Delta(\Pi)} (p')^\T   \boldsymbol{L} q^*(x_t) \tag{The best policy for context $x_t$}
	\end{align*}
	
	\begin{align*}
		\mathrm{Reg}_T &= \sum_{t=1}^{T} p_t^\T \boldsymbol{L} q^*(x_t)   - (\lambda^*_{x_t} )^\T   \boldsymbol{L} q^*(x_t)    \\
		&= \sum_{t=1}^{T} p_t^\T \boldsymbol{L} \hat{q}_t(x_t)  -  (\lambda^*_{x_t}  )^\T \boldsymbol{L} \hat{q}_t(x_t) + (\lambda^*_{x_t}  - p_t )^\T \boldsymbol{L} ( \hat{q}_t(x_t) - q^*(x_t)   )  
	\end{align*}	
	
	\begin{align*}
		  \left((\lambda^*_{x_t}  - p_t )^\T \boldsymbol{L} ( \hat{q}_t(x_t) - q^*(x_t)  ) \right)^2 
		&=   \frac{ \left( \left(\lambda^*_{x_t}  - p_t\right)^\T \boldsymbol{L} (\hat{q}_t(x_t) - q^*(x_t)  )\right)^2 }{  \EE_{a \sim p_t} \sbr{D^2_{\mathrm{H}}\big( \boldsymbol{S}_a \hat{q}_t(x_t) \,\big|\, \boldsymbol{S}_a q^*(x_t)\big) } } \cdot \left(  \EE_{a \sim p_t} \sbr{D^2_{\mathrm{H}}\big( \boldsymbol{S}_a \hat{q}_t(x_t) \,\big|\, \boldsymbol{S}_a q^*(x_t)\big) }  \right) \\
		&\leq  \spacex \max_{q_1,q_2 \in \Delta(d)} \frac{ \left( \left(\lambda^*_{x_t}  - p_t\right)^\T \boldsymbol{L} (q_1 - q_2)\right)^2 }{ \EE_{a \sim p_t } \sbr{ D^2_{\mathrm{H}} \big( \boldsymbol{S}_a q_1 \,\big|\, \boldsymbol{S}_a q_2 \big)}  }\cdot \left(  \EE_{a \sim p_t} \sbr{D^2_{\mathrm{H}}\big( \boldsymbol{S}_a \hat{q}_t(x_t) \,\big|\, \boldsymbol{S}_a q^*(x_t)\big) }  \right).
	\end{align*}

	Therefore,
	
	\begin{align*}
		\left((\lambda^*_{x_t}  - p_t )^\T \boldsymbol{L} ( \hat{q}_t(x_t) - q^*(x_t)   ) \right) ^2  &\leq \mathrm{Coverage}(p_t, \lambda^*_{x_t})  \cdot\left(  \EE_{a \sim p_t} \sbr{D^2_{\mathrm{H}}\big( \boldsymbol{S}_a \hat{q}_t(x_t) \,\big|\, \boldsymbol{S}_a q^*(x_t)\big) }  \right)\\
		\left((\lambda^*_{x_t}  - p_t )^\T \boldsymbol{L} ( \hat{q}_t(x_t) - q^*(x_t)   ) \right)   &\leq \sqrt{\mathrm{Coverage}(p_t, \lambda^*_{x_t})  \cdot\left(  \EE_{a \sim p_t} \sbr{D^2_{\mathrm{H}}\big( \boldsymbol{S}_a \hat{q}_t(x_t) \,\big|\, \boldsymbol{S}_a q^*(x_t)\big) }  \right)}\\
		&\leq \frac{1}{\gamma}	\mathrm{Coverage}(p_t, \lambda^*_{x_t})  +  \gamma \left(  \EE_{a \sim p_t} \sbr{D^2_{\mathrm{H}}\big( \boldsymbol{S}_a \hat{q}_t(x_t) \,\big|\, \boldsymbol{S}_a q^*(x_t)\big) } \right).  \tag{AM-GM}\\
	\end{align*}

	\begin{align*}
		\mathrm{Reg}_T &\leq  \sum_{t=1}^{T} \sup_{\lambda \in \Delta(\Pi)} \sbr{p_t^\T \boldsymbol{L} \hat{q}_t(x_t)  -  \lambda^\T  \boldsymbol{L} \hat{q}_t(x_t) +  \frac{1}{\gamma}\mathrm{Coverage}(p_t, \lambda ) }  +   \sum_{t=1}^{T} \gamma \left(  \EE_{a \sim p_t} \sbr{D^2_{\mathrm{H}}\big( \boldsymbol{S}_a \hat{q}_t(x_t) \,\big|\, \boldsymbol{S}_a q^*(x_t)\big) }  \right)  \\
		&\leq  \sum_{t=1}^{T} \sup_{\lambda \in \Delta(\Pi)} \sbr{p_t^\T \boldsymbol{L} \hat{q}_t(x_t)  -  \lambda^\T  \boldsymbol{L} \hat{q}_t(x_t) +  \frac{1}{\gamma} \cdot \mathrm{Coverage}(p_t, \lambda ) } +  \gamma \left( \mathrm{Reg}_{\mathrm{On}}^{\mathrm{H}} \left( \mathcal{Q}, T\right) \right).
	\end{align*}

	Therefore the regret can be bounded as,
	
	\begin{align*}
		\mathrm{Reg}_T &\leq \sum_{t=1}^{T}  \mathcal{V}_{\gamma}\big(\hat{q}_t(x_t), p_t \big) +  \gamma \left( 	\mathrm{Reg}_{\mathrm{On}}^{\mathrm{H}} \left( \mathcal{Q}, T\right)  \right) \\
		&\leq  T \cdot \max_{t=1}^T \mathcal{V}_{\gamma}\big(\hat{q}_t(x_t), p_t \big)  + \gamma \left( 	\mathrm{Reg}_{\mathrm{On}}^{\mathrm{H}} \left( \mathcal{Q}, T\right)  \right).
	\end{align*}
	
	Hence, with probability of at least $1-\delta$, 
	\begin{align*}
		\mathrm{Reg}_T &\leq  T \cdot \max_{t=1}^T \mathcal{V}_{\gamma}\big(\hat{q}_t(x_t), p_t \big) +	\gamma \cdot \overline{\mathrm{Reg}}_{\mathrm{On}}^{\mathrm{H}}\left( \mathcal{Q}, T, \delta \right).
	\end{align*}
	
\end{proof}

Moreover,
\begin{remark}[DEC bound]\label{remark_dec_bound}
	For all $x \in \mathcal{X}, \hat{q} \in \mathcal{Q}, p \in \Delta(d)$,
	\begin{align*}
		\mathcal{V}^{\mathrm{dec}}_{\gamma}(\hat{q}(x), \mathcal{Q}_{x}, p) \leq \mathcal{V}_{\gamma}(\hat{q}(x), p).
	\end{align*}
\end{remark}

\begin{proof}
	\begin{align*}
		\mathcal{V}^{\mathrm{dec}}_{\gamma}(\hat{q}(x), \mathcal{Q}_{x}, p) &=  \sup_{q \in \mathcal{Q}_{x}, \lambda \in \Delta(k)}
		\sbr{ p^\T \boldsymbol{L} q(x) - \lambda^\T \boldsymbol{L} q(x)
			- \gamma \cdot \EE_{a \sim p}\sbr{ D^2_{\mathrm{H}}\big( \boldsymbol{S}_a \hat{q}(x) \,\big|\, \boldsymbol{S}_a q(x) \big) } }\\
		&=  \sup_{q \in \mathcal{Q}_{x}, \lambda \in \Delta(\Pi)}
		\sbr{ p^\T \boldsymbol{L} q(x) - \lambda^\T \boldsymbol{L} q(x)
			- \gamma \cdot \EE_{a \sim p}\sbr{ D^2_{\mathrm{H}}\big( \boldsymbol{S}_a \hat{q}(x) \,\big|\, \boldsymbol{S}_a q(x) \big) } }\\
		&= \sup_{q \in \mathcal{Q}_{x}, \lambda \in \Delta(\Pi)}
		\sbr{ p^\T \boldsymbol{L} \hat{q}(x) - \lambda^\T \boldsymbol{L} \hat{q}(x) + (\lambda  - p )^\T \boldsymbol{L} ( \hat{q}(x) - q(x)   ) 
			- \gamma \cdot \EE_{a \sim p}\sbr{ D^2_{\mathrm{H}}\big( \boldsymbol{S}_a \hat{q}(x) \,\big|\, \boldsymbol{S}_a q(x) \big) } }\\
		&\leq \sup_{q \in \mathcal{Q}_{x}, \lambda \in \Delta(\Pi)}
		\sbr{ p^\T \boldsymbol{L} \hat{q}(x) - \lambda^\T \boldsymbol{L} \hat{q}(x) +  \frac{1}{\gamma}	\mathrm{Coverage}(p, \lambda) }\\
		&= \sup_{ \lambda \in \Delta(\Pi)}
		\sbr{ p^\T \boldsymbol{L} \hat{q}(x) - \lambda^\T \boldsymbol{L} \hat{q}(x) +  \frac{1}{\gamma}	\mathrm{Coverage}(p, \lambda) }\\
		&= \mathcal{V}_{\gamma}(\hat{q}(x), p)
	\end{align*}
\end{proof}

\subsection{Proofs for~\Cref{alg:explore-go} (\egqpm)}\label{app_section_efficient_global}

Recall that in partial monitoring the loss difference between two actions can be expressed in terms of the
feedback from other actions (\Cref{gl_observability_graph,local_observability_graph}), though the two
observability conditions supply that expression differently. The following lemma abstracts both constructions
into a single reference witness $G$ and bounds the coverage in terms of it.

\begin{lemma}\label{lem:coverage-g-ub}
	Suppose there exists a \textit{reference witness function} $G:\mathcal{A}\times \Sigma \to \mathbb{R}^{\mathcal{A}}$, in the sense that for all $ b, c \in \mathcal{A} \subseteq [k]$,
	\begin{align}
		\sum_{a \in \mathcal{A} } G_b(a, \Phi_{ao}) - \sum_{a \in \mathcal{A} }  G_c(a, \Phi_{ao})  = L_{bo} - L_{co} \spacex \forall o \in [d].
		\label{eqn:unbiased}
	\end{align}
	
	where, $G_b$ denotes $b$-th index of $G$. Then $\forall p, \lambda \in \Delta(\mathcal{A})$,
	\[
	\mathrm{Coverage}(p, \lambda ) 
	\leq 
	4 \sum_{a \in \mathcal{A} }  \frac{( \sum_{b \in \mathcal{A} } (p_b - \lambda_b) \| G_b(a, :) \|_\infty )^2}{p_a}
	\leq 
	8 \sum_{a \in \mathcal{A} }  \sum_{b \in \mathcal{A} }  \frac{p_b + \lambda_b}{p_a} \| G_b(a, :) \|_\infty^2
	\]
\end{lemma}
\begin{proof}
		First, since $D^2_{\mathrm{H}}\big( u \,\big|\, v \big) \geq \frac{1}{4} \| u - v \|_1^2$ (Lemma~\ref{dist_ineq}), we have, 
		\[
		\mathrm{Coverage}(p, \lambda)
		= \sup_{q_1, q_2 \in \Delta(d)}
		\frac{ \left( \left(\lambda - p\right)^\T \boldsymbol{L} \left(q_1 - q_2\right) \right)^2 }
		{  \EE_{a \sim p} \sbr{ D^2_{\mathrm{H}}\big( \boldsymbol{S}_a q_1 \,\big|\, \boldsymbol{S}_a q_2 \big) } }
		\leq 
		4 \sup_{q_1, q_2 \in \Delta(d)}
		\frac{ \left( \left(\lambda - p\right)^\T \boldsymbol{L} \left(q_1 - q_2\right) \right)^2 }
		{  \sum_{a  \in \mathcal{A}}  \| \boldsymbol{S}_a (q_1 - q_2) \|_1^2 }
		\]

	Taking a weighted summation of \Cref{eqn:unbiased} over $b, c$ with weight $p_b \lambda_c$ gives,
	
	\[
	(p-\lambda)^T ( \sum_a G(a, \Phi_{ao}) - L_{:,o} ) = 0, \forall o
	\]
	And thus, 
	\begin{align*}
		(p-\lambda)^T L (q_1 - q_2) 
		= &
		\sum_o (p-\lambda)^T \sum_{a \in \mathcal{A}} G(a, \Phi_{ao}) (q_1 - q_2)_o \\
		= & 
		\sum_o (p-\lambda)^T \sum_{a  \in \mathcal{A} }\sum_\sigma G(a, \sigma) (\boldsymbol{S}_a)_{\sigma,o} (q_1 - q_2)_o  \\
		= & (p-\lambda)^T \sum_\sigma 
		\sum_{a  \in \mathcal{A}} G(a,\sigma) (\boldsymbol{S}_a (q_1 - q_2) )_\sigma \\
		\leq & 
			\sum_{b  \in \mathcal{A} } \sum_\sigma 
			\sum_{a \in \mathcal{A}} (p_b - \lambda_b) G_b(a,\sigma)  (\boldsymbol{S}_a (q_1 - q_2) )_\sigma
	\end{align*}
	And therefore, exchanging the summation by moving 
		$\sum_\sigma$ to the innermost, 
		and by triangle inequality,
		$| \sum_{\sigma} G_b(a, \sigma) \boldsymbol{S}_a(q_1 - q_2)_\sigma | \leq \| G_b(a, :) \|_\infty \| \boldsymbol{S}_a(q_1 - q_2) \|_1
		$; this implies that
	\[
	\abs{ (p-\lambda)^T L (q_1 - q_2) }
	\leq 
	\sum_{a \in \mathcal{A}} \sum_{b \in \mathcal{A}} |p_b - \lambda_b| \| G_{b}(a,:) \|_\infty \| \boldsymbol{S}_a(q_1 - q_2) \|_1.
	\]
	The first inequality of the lemma now follows from the definition of coverage and Cauchy-Schwarz inequality.
	
	The second inequality follows from that 
	\begin{align*}
		(\sum_{b \in \mathcal{A}} (p_b - \lambda_b) \| G_b(a, :) \|_\infty )^2
		\leq &
		2 (\sum_{b \in \mathcal{A} } p_b \| G_b(a, :) \|_\infty )^2 + 2 (\sum_{b  \in \mathcal{A}} \lambda_b \| G_b(a, :) \|_\infty )^2
		\\
		\leq &
		2 \sum_{b \in \mathcal{A}} p_b \| G_b(a, :) \|_\infty^2 + 2 \sum_{b \in \mathcal{A}} \lambda_b \| G_b(a, :) \|_\infty^2
	\end{align*}
\end{proof}

In \Cref{lem:coverage-g-ub} the indices $b$ and $c$ range over a set $\mathcal{A}$, which we take to be $[k]$
for globally observable games and $\Pi$ for locally observable games. In the globally observable case the
behavioral policy places mass on every action, so the sums must run over all of $[k]$; in the locally
observable case both policies lie in $\Delta(\Pi)$ and the witness vanishes outside $\Pi$, so the restriction
is without loss.

\newtheorem*{app_lemma_globally_observable_doec}{Lemma \ref{lemma_globally_observable_doec}}

\begin{app_lemma_globally_observable_doec}[Guarantee for Alg.~\ref{alg:explore-go} for globally observable games]\label{app_lemma_globally_observable_doec}
	Let $\mathcal{G} = (\boldsymbol{L}, \boldsymbol{\Phi})$ be a globally observable partial monitoring game in the
	sense of Definition~\ref{gl_observability_graph}, and set
	$\beta := \max_{(a,b) \in [k] \times [k]} \lVert w_{(a,b)} \rVert_{\infty}$. 
	We have for any $\hat{q} \in \Delta(d)$,
	\[
	\mathcal{V}_{\gamma}\big(\hat{q}, p \big)
	\leq 
	\max_{\lambda \in \Delta(\Pi)} p^T \boldsymbol{L} \hat{q} - \lambda^T \boldsymbol{L} \hat{q}
	+ \frac{16\beta^2}{\gamma}  \sum_{c=1}^k \frac{1}{p_c}.
	\]
	
	Thus, for
	every $\gamma > 16 k^2 \beta^2$, and every $\hat{q} \in \Delta(d)$, Alg.~\ref{alg:explore-go} outputs $p$ such that
	\begin{align*}
		\mathcal{V}_{\gamma}\big(\hat{q}, p \big)
		\leq O\!\left( k\beta \sqrt{\frac{1}{\gamma}} \right).
	\end{align*}
\end{app_lemma_globally_observable_doec}
\begin{proof}
	For globally observable games, we can construct \textit{reference witness function} by simply setting, for all $b \in [k]$
	
	\begin{align*}
		G_b(\cdot, \cdot) = w_{b, a_0} (\cdot, \cdot) \tag{for some reference action $a_0$}
	\end{align*}

	Hence, by Lemma~\ref{lem:coverage-g-ub},
	\begin{align*}
		\mathrm{Coverage}(p, \lambda) &\leq 16 \max_{(a,b) \in [k] \times [k]} \lVert w_{(a,b)} \rVert_{\infty}^2 \sum_{c=1}^k \frac{1}{p_c} \\
		&= 16 \beta^2 \sum_{c=1}^k \frac{1}{p_c}.
	\end{align*}

This finishes the proof of the first part of the lemma.
	
	Let $\hat{a}$ be the empirical best arm based on our current estimate $\hat{q}$. That is,
	\begin{align*}
		\hat{a} := \argmin_{a \in \Pi} e_{a}^\T \boldsymbol{L} \hat{q}(x).
	\end{align*}
	
	Set,
	
	\begin{align*}
		p = (1-\epsilon) e_{\hat{a}} + \epsilon \frac{\boldsymbol{1}}{k},
	\end{align*}
	
	so that $p_c \geq \frac{\epsilon}{k}$ for every $c \in [k]$, hence
	\begin{align*}
		\sum_{c=1}^k \frac{1}{p_c} \leq \frac{k^2}{\epsilon}.
	\end{align*}
	
	\begin{align*}
		\mathcal{V}_{\gamma}\big(\hat{q}(x), p \big) &= \sup_{\lambda \in \Delta(\Pi)}
		\sbr{ p^\T \boldsymbol{L} \hat{q}(x) - \lambda^\T \boldsymbol{L} \hat{q}(x)
			+ \frac{1}{\gamma} \cdot \mathrm{Coverage}(p, \lambda) } \\
		&\leq p^\T \boldsymbol{L} \hat{q}(x) - e_{\hat{a}}^\T \boldsymbol{L} \hat{q}(x)
		+ \frac{1}{\gamma}  16 \beta^2  \sum_{c=1}^k \frac{1}{p_c}  \\
		&\leq \left((1-\epsilon) e_{\hat{a}} + \epsilon \frac{\boldsymbol{1}}{k}\right)^\T \boldsymbol{L} \hat{q}(x) - e_{\hat{a}}^\T \boldsymbol{L} \hat{q}(x)
		+ \frac{16k^2\beta^2}{\gamma \epsilon}  \\
		&\leq \epsilon \, \tfrac{1}{k}\boldsymbol{1}^\T \boldsymbol{L} \hat{q}(x)
		+ \frac{16k^2\beta^2}{\gamma \epsilon}   \\
		&\leq \alpha + \frac{16k^2\beta^2}{\gamma  \epsilon}  \tag{$\boldsymbol{L} \in [0,1]^{k \times d}$} \\
		&= O\!\left( k\beta \sqrt{\frac{1}{\gamma}} \right). \tag{by choosing $\epsilon =  k \beta \sqrt{\frac{1}{\gamma}} < 1$}	
	\end{align*}
This finishes the proof of the second part of the lemma.
\end{proof}

\subsection{Proofs for~\Cref{alg:explore-lo} (\sqpm)}\label{app_section_efficient_local}
We can choose a spanning in-tree over the neighborhood graph $(\Pi, \mathcal{E})$. That is, assign an arbitrary action $\mathrm{root}_{\mathcal{T}}$ as the root and choose an in-tree $\mathcal{T}$ such that, there is a unique directed $\mathrm{path}_{\mathcal{T}}(a)$ from any action $a \in \Pi$ to $\mathrm{root}_{\mathcal{T}} \in \Pi$ (See Figure~\ref{fig:in-tree}). Now, define \textit{witness reference function} $G_{\mathcal{T}} : [k] \times \Sigma \to \mathbb{R}^{ \Pi }$ $\forall a \in \Pi$ such that,
\begin{align}\label{app_eq_g_function}
	G^{\mathcal{T}}_a(c, \sigma) &:= \sum_{e \in \mathrm{path}_{\mathcal{T}}(a)} w_{e}(c, \sigma),
	\qquad a \in \Pi,
\end{align}
where $w_e$ denotes the witness function associated with the edge $e$ by local observability
(Definition~\ref{local_observability_graph}). Here it is noteworthy that, let $e = (e_1,e_2)$ then $w_e(c, \sigma) = 0$ $\forall c \notin \{e_1,e_2\}$. Therefore,
\begin{align}\label{g_sparsity}
	G^{\mathcal{T}}_a(c, \sigma) = 0 \text{ if } c \notin \mathrm{path}_{\mathcal{T}}(a).
\end{align}
Therefore, the local observability condition induces a sparse $G^{\mathcal{T}}_a(\cdot, \sigma)$ function, which will become handy for us to allocate probability only for the actions populating it.
Furthermore, we can also show that, $\forall a \in \Pi$
\begin{align}\label{equation_g_accumulation}
	\sum_{c \in \Pi} G^{\mathcal{T}}_a(c, \Phi_{co}) = \sum_{c \in \Pi}  \sum_{e \in \mathrm{path}_{\mathcal{T}}(a)} w_{e}(c, \Phi_{co}) =   \sum_{e \in \mathrm{path}_{\mathcal{T}}(a)} \sum_{c \in \Pi} w_{e}(c, \Phi_{co}) = L_{ao} - L_{\mathrm{root}_{\mathcal{T}}o}
\end{align}
Hence, \textit{witness reference function} $G_{\mathcal{T}}$ satisfies the conditions of Lemma~\ref{lem:coverage-g-ub}. Hence,

Now we relax $\mathcal{V}_{\gamma}\big(\hat{q}, p \big)$ with sparse $G^{\mathcal{T}}_a(\cdot, \sigma)$.

\newtheorem*{app_lemma_locally_observable_doec_part_1}{Lemma \ref{lemma_locally_observable_doec_part_1}}

\begin{app_lemma_locally_observable_doec_part_1}\label{app_lemma_locally_observable_doec_part_1}
	Let $\mathcal{G} = (\boldsymbol{L}, \boldsymbol{\Phi})$ be a locally observable partial monitoring game in the
sense of Definition~\ref{local_observability_graph}.
Let $\cT$ be an in-tree of the neighborhood graph and $G^\cT$ be its associated \textit{witness reference function}. For all $p \in \Delta(\Pi)$, 
we have 
\begin{align}
	\mathcal{V}_{\gamma}\big(\hat{q}(x), p \big)
	\leq & 
	\max_{\lambda \in \Delta(\Pi)} p^T \boldsymbol{L} \hat{q}(x) - \lambda^T \boldsymbol{L} \hat{q}(x)
	+ \frac{8}{\gamma} \sum_{b \in \Pi }\sum_{c \in \Pi} \left( \frac{\lambda_b}{ p_c }+  \frac{p_b}{p_c} \right)\lVert G^{\mathcal{T}} (c, \cdot)_b \rVert^2_{\infty} \nonumber \\
	\leq &
	\max_{\lambda \in \Delta(\Pi)} p^T \boldsymbol{L} \hat{q}(x) - \lambda^T \boldsymbol{L} \hat{q}(x)
	+ \frac{32}{\gamma} \max_{e \in \mathcal{E}} \lVert w_e \rVert^2_{\infty} \sum_{(b \in \Pi,c \in \Pi):c \in \mathrm{path}_{\mathcal{T}}(b) } \left( \frac{\lambda_b}{ p_c }+  \frac{p_b}{p_c} \right).\nonumber
\end{align}
\end{app_lemma_locally_observable_doec_part_1}
\begin{proof}
	
	By Lemma~\ref{lem:coverage-g-ub},
	
	\begin{align*}
		\mathrm{Coverage}(p, \lambda) &\leq  8\sum_{b \in \Pi }\sum_{c \in \Pi} \left( \frac{\lambda_b}{ p_c }+  \frac{p_b}{p_c} \right)\lVert G^{\mathcal{T}}_b(c, \cdot ) \rVert^2_{\infty}.
	\end{align*}
	
	This finishes the proof of the first inequality of the lemma.
	
	Furthermore, we have
	\begin{align*}
		G^{\mathcal{T}}_b(c, \cdot ) = \sum_{e \in \mathrm{path}_{\mathcal{T}}(b)} w_{e}(c, \cdot),
	\end{align*}
	Let $e = (e_1, e_2)$, then by Definition~\ref{local_observability_graph}, $w_e(c, \cdot ) = 0 \;\; \forall c\neq e_1, e_2$.
	Hence,
	\begin{align*}
		\mathrm{Coverage}(p, \lambda ) &\leq   8\sum_{b \in \Pi }\sum_{c \in \Pi} \left( \frac{\lambda_b}{ p_c }+  \frac{p_b}{p_c} \right)\lVert G^{\mathcal{T}}_b(c, \cdot ) \rVert^2_{\infty} \\
		&= 8\sum_{b \in \Pi }\sum_{c \in \Pi} \left( \frac{\lambda_b}{ p_c }+  \frac{p_b}{p_c} \right)\lVert  \sum_{e \in \mathrm{path}_{\mathcal{T}}(b)} w_{e}(c, \cdot)\onec{c \in \{e_1, e_2\}} \rVert^2_{\infty} \\
		&= 8\max_{e \in \mathcal{E}} \lVert w_e \rVert^2_{\infty}  \sum_{b \in \Pi }\sum_{c \in \Pi} \left( \frac{\lambda_b}{ p_c }+  \frac{p_b}{p_c} \right)\lVert  \sum_{e \in \mathrm{path}_{\mathcal{T}}(b)}  \onec{c \in \{e_1, e_2\}} \rVert^2_{\infty}\\
		&=  32 \max_{e \in \mathcal{E}} \lVert w_e \rVert^2_{\infty}  \sum_{(b,c): c  \in \mathrm{path}_{\mathcal{T}}(b)}  \left( \frac{\lambda_b}{ p_c }+  \frac{p_b}{p_c} \right) \tag{one vertex will be a part of maximum 2 edges in an in-tree path}\\
		&\leq 32 \max_{e \in \mathcal{E}} \lVert w_e \rVert^2_{\infty} \sum_{(b \in \Pi,c \in \Pi): c  \in \mathrm{path}_{\mathcal{T}}(b)} \left( \frac{\lambda_b}{ p_c }+  \frac{p_b}{p_c} \right) \\
	\end{align*}
	This finishes the proof of the second inequality of the lemma.
\end{proof}

Above result is true for any in-tree with an arbitrary root. However, in general, we do not want the paths to be arbitrary: for a ``bad'' action $c$ and a ``good'' action $b$ (i.e., $e_c^\T L \hat{q}$ is large while $e_b^\T L \hat{q}$ is small), we would like $p_c$ to be small and $p_b$ to be large to ensure $p^T \boldsymbol{L} \hat{q}(x) - \lambda^T \boldsymbol{L} \hat{q}(x)$ to be small; however, if $c$ happens to lie in $\mathrm{path}_{\mathcal{T}}(b)$, that can cause the $\frac{\lambda_b}{p_c}$ term to be very large. Hence,

\begin{enumerate}
	\item We need to choose an in-tree $\mathcal{T}$, such that the path to any action $a$ have actions with smaller estimated losses, such that assigning higher probability for them is both informative and rewarding.
	\item We need to find a mechanism to allocate probability along this in-tree $\mathcal{T}$ respecting 1), such that both $p^T \boldsymbol{L} \hat{q}(x) - \lambda^T \boldsymbol{L} \hat{q}(x)$ and $ \frac{32}{\gamma} \max_{e \in \mathcal{E}} \lVert w_e \rVert^2_{\infty} \sum_{(b \in \Pi,c \in \Pi):c \in \mathrm{path}_{\mathcal{T}}(b) } \left( \frac{\lambda_b}{ p_c }+  \frac{p_b}{p_c} \right)$ are controlled.
\end{enumerate}

The best solution for 1) is an $\boldsymbol{L}\hat{q}$-decreasing in-tree. We define $\mathcal{T}$-decreasing in-tree as follows, 

\begin{definition}[$\mathcal{T}$-decreasing in-tree \citep{lattimore2020exploration}]\label{definition_decreasing_tree}
	Let $\mathcal{T}$ be an in-tree over $[\Pi]$. A vector $y \in \mathbb{R}^\Pi$ is called $\mathcal{T}$-decreasing if $y_a \geq y_b$ for all directed edge $(a,b) \in \mathcal{T}$. That is, $a \mapsto y_a$ is decreasing towards the root of $\mathcal{T}$.
\end{definition}

Basically, an $\boldsymbol{L}\hat{q}$-decreasing in-tree $\mathcal{T}$ will have the empirical best action ($\argmin_{a \in \Pi} e_a^\T \boldsymbol{L}\hat{q}(x)$) as the root and for every directed edge $e = (a,b)$, $e_a^\T \boldsymbol{L}\hat{q}(x) \geq e_b^\T \boldsymbol{L}\hat{q}(x)$. Such a $\cT$ is guaranteed to exist by~\cite{lattimore2020exploration} (Lemma 19).

For 2), we observe that, inverse-gap-weighting (IGW), by construction, assigns probability in an increasing fashion toward the root. Now we provide the proof for Lemma~\ref{lemma_locally_observable_doec}.

\newtheorem*{app_lemma_locally_observable_doec}{Lemma \ref{lemma_locally_observable_doec}}

\begin{app_lemma_locally_observable_doec}[Guarantee for Alg.~\ref{alg:explore-lo} for locally observable games]\label{app_lemma_locally_observable_doec}
    Let $\mathcal{G} = (\boldsymbol{L}, \boldsymbol{\Phi})$ be a locally observable partial monitoring game in the
sense of Definition~\ref{local_observability_graph}. For
every $\gamma > 0$, and every $\hat{q} \in \Delta(d)$, Alg.~\ref{alg:explore-go} run with $\kappa = \frac{\gamma}{8k m^2}$ and $\mu = k$, outputs $p$ such that
\begin{align*}
	\mathcal{V}_{\gamma}\big(\hat{q}, p \big)
	\leq O\!\left(  \frac{m^2 k^2}{\gamma} \right).
\end{align*}
\end{app_lemma_locally_observable_doec}

	\begin{proof}
		Recall that we define in Alg.~\ref{alg:explore-lo} that, $\hat{b} = \argmin_{a \in \Pi} e_a^\T\boldsymbol{L}\hat{q}(x)$ and $\forall a \in \Pi \setminus \{\hat{b}\}$, $p(a) = \frac{1}{\mu + \kappa \left( e_a  - e_{\hat{b}} \right)^\T L \hat{q}(x) } $ and $p (\hat{b}) = 1 -\sum_{a \in \Pi\setminus\{\hat{b}\}} p(a)$.

	We choose $\cT$ to be $\boldsymbol{L} \hat{q}(x)$-decreasing in-tree in the neighborhood graph;  
     we have that for any $(b,c)$ such that $c \in \mathrm{path}_{\mathcal{T}}(b)$, $p_c \geq p_b$. This is true because for any $c \in \mathrm{path}(b)$ of $\boldsymbol{L} \hat{q}(x)$-decreasing in-tree, estimated loss of $c$ is lower than estimated loss of $b$. Hence, by construction of IGW based policy $p$,  $p_c \geq p_b$.

By Lemma~\ref{lemma_locally_observable_doec_part_1} we have,
	\begin{align*}
	\mathrm{Coverage}(p, \lambda ) &\leq 32 \max_{e \in \mathcal{E}} \lVert w_e \rVert^2_{\infty} \sum_{(b,c): c  \in \mathrm{path}_{\mathcal{T}}(b)} \left( \frac{\lambda_b}{ p_c }+  \frac{p_b}{p_c} \right).
\end{align*}
Since for every $c$ in $\mathrm{path}_\cT(b)$, $p_c \geq p_b$

\begin{align*}
	\mathrm{Coverage}(p, \lambda) &\leq 32 \max_{e \in \mathcal{E}} \lVert w_e \rVert^2_{\infty} \sum_{b \in \Pi} \left( \frac{\lambda_b}{ p_b }+  \frac{p_b}{p_b} \right)\\
	&\leq 32 \max_{e \in \mathcal{E}} \lVert w_e \rVert^2_{\infty} k^2 + 32 \max_{e \in \mathcal{E}} \lVert w_e \rVert^2_{\infty} k \sum_{b \in \Pi} \frac{\lambda_b}{ p_b } \\
\end{align*}

	Moreover,
	\begin{align*}
		&\sup_{\lambda \in \Delta(\Pi)} \Bigg\{ p^\top \boldsymbol{L}\hat{q}(x) - \lambda^\top \boldsymbol{L}\hat{q}(x) + \frac{1}{\gamma}\cdot 32 \max_{e \in \mathcal{E}} \lVert w_e \rVert^2_{\infty} k \sum_{b \in \Pi} \frac{\lambda_b}{ p_b } \Bigg\} \\
		&\spacex \leq \sup_{\lambda \in \Delta(\Pi)}  \sbr{ \sum_{b \in \Pi} \left(p_b - \lambda_b\right)  \boldsymbol{L}_{b,:} \hat{q}(x)  +  \frac{1}{\gamma}\cdot 32 \max_{e \in \mathcal{E}} \lVert w_e \rVert^2_{\infty} k \sum_{b \in \Pi} \lambda_b \left(\mu + \kappa \left( e_b  - e_{\hat{b}} \right)^\T L \hat{q}(x) \right) } \\
	\end{align*}
	
	In the above inequality, we use the observation that $p_{\hat{b}}$ is lower bounded by $\frac{1}{k}$. This is true, since $\mu = k$.

	\begin{align*}
		&\sup_{\lambda \in \Delta(\Pi)} \Bigg\{ p^\top \boldsymbol{L}\hat{q}(x) - \lambda^\top \boldsymbol{L}\hat{q}(x) + \frac{1}{\gamma}\cdot 32 \max_{e \in \mathcal{E}} \lVert w_e \rVert^2_{\infty} k \sum_{b \in \Pi} \frac{\lambda_b}{ p_b } \Bigg\} \\
		&\spacex \leq  \sup_{\lambda \in \Delta(\Pi)}  \sbr{ \sum_{b \in \Pi} \left(p_b - \lambda_b\right)  \boldsymbol{L}_{b,:} \hat{q}(x)  +  \frac{1}{\gamma}\cdot 32 \max_{e \in \mathcal{E}} \lVert w_e \rVert^2_{\infty} k^2 +  \sum_{b \in \Pi} \lambda_b   \left( e_b  - e_{\hat{b}} \right)^\T L \hat{q}(x)} \tag{$\mu = k$, $\kappa = \frac{\gamma}{ 32 \max_{e \in \mathcal{E}} \lVert w_e \rVert^2_{\infty} k}$}\\
	\end{align*}
	
	Furthermore,
	
	\begin{align*}
		&\sup_{\lambda \in \Delta(\Pi)}  \sbr{ \sum_{b \in \Pi} \left(p_b - \lambda_b\right)  \boldsymbol{L}_{b,:} \hat{q}(x) +  \sum_{b \in \Pi} \lambda_b   \left( e_b  - e_{\hat{b}} \right)^\T L \hat{q}(x)}\\
		 &\spacex =  \sum_{b \in \Pi} p_b \cdot  e_b^\T \boldsymbol{L} \hat{q}(x) - e_{\hat{b}}^\T \boldsymbol{L} \hat{q}(x)  \\
		 &\spacex =  \sum_{b \in \Pi\setminus \{\hat{b}\}} p_b \cdot  e_b^\T \boldsymbol{L} \hat{q}(x) - e_{\hat{b}}^\T \boldsymbol{L} \hat{q}(x)  \\
		 &\spacex=  \sum_{b \in \Pi\setminus \{\hat{b}\}}  \frac{1}{\mu + \kappa \left( e_a  - e_{\hat{b}} \right)^\T L \hat{q}(x) } \cdot   e_b^\T \boldsymbol{L} \hat{q}(x) - e_{\hat{b}}^\T \boldsymbol{L} \hat{q}(x)  \\
		 &\spacex \leq \frac{k}{\kappa} = \frac{1}{\gamma}\cdot 32 \max_{e \in \mathcal{E}} \lVert w_e \rVert^2_{\infty} k^2
	\end{align*}

	Hence, $\forall \hat{q} \in \Delta(d)$
	\begin{align*}
		\mathcal{V}_{\gamma}\big(\hat{q}, p \big)  &\leq   \frac{1}{\gamma}\cdot 96 \max_{e \in \mathcal{E}} \lVert w_e \rVert^2_{\infty} k^2.
	\end{align*}

	Furthermore, revisiting Definition~\ref{local_observability_graph},
	
	\begin{align*}
		\boldsymbol{L}_{ao} - \boldsymbol{L}_{bo} &= \sum_{c \in \{a,b\}} w_e(c,\Phi_{co}),
		\qquad \text{for all } o \in [d].\\
		&= w_e(a,\Phi_{co}) + w_e(b,\Phi_{co}).
	\end{align*}
	
	By Lemma~\ref{lem:bounds}, one can always find, a witness function such that, $\Vert w_e \rVert_\infty \leq \frac{m}{2}$, Hence, $\forall \hat{q} \in \Delta(d)$
	
	\begin{align*}
		\mathcal{V}_{\gamma}\big(\hat{q}, p \big)  &= O\left( \frac{m^2 k^2 }{\gamma}\right).
	\end{align*}
	This completes the proof.
\end{proof}

\subsection{Non-degeneracy assumption}\label{relax_non_degeneracy}

Definition~\ref{local_observability_graph} is applicable only when the game is non-degenerate. When there is degenerate action in $[k]$, the definition for local observability becomes,

\begin{definition}[Local observability, general form; \citet{bartokICML2012}]
	\label{local_observability_graph_degenerate}
	A globally observable partial monitoring game is \emph{locally observable} if for every edge
	$e = (a,b) \in \mathcal{E}$ of the neighborhood graph the witness function $w_e$ of
	Definition~\ref{gl_observability_graph} may be chosen supported on
	\begin{align*}
		N^{+}_{a,b} := \{ c \in [k] : \mathcal{C}_{a} \cap \mathcal{C}_{b} \subseteq \mathcal{C}_{c} \},
	\end{align*}
	that is, $w_e(c, \cdot) \equiv 0$ for all $c \notin N^{+}_{a,b}$, so that
	\begin{align*}
		L_{ao} - L_{bo} = \sum_{c \in N^{+}_{a,b}} w_e(c, \Phi_{co}),
		\qquad \text{for all } o \in [d].
	\end{align*}
\end{definition}

The set $N^{+}_{a,b}$ contains the two endpoints, but may also contain degenerate actions whose cells lie
inside the shared face $\mathcal{C}_a \cap \mathcal{C}_b$. For $d = 3$, for instance, a degenerate action with
$\dim(\mathcal{C}_c) = 1$ is optimal only along a line segment, whereas a Pareto optimal action is optimal on a
two-dimensional region, yet such an action can still be a necessary source of information.

Two parts of our development break under this more general condition. First, the exploration rule must assign
probability to the degenerate actions as well, which requires a modification of \sqpm. Second,
Lemma~\ref{lem:bounds} no longer applies and must be replaced by Proposition 21 of \citet{lattimore2020exploration} (see Lemma~\ref{lem:bounds-global} for restatement), which introduces a exponential factor in $d$. 
The modification to the
algorithm is straightforward, and we conjecture that \sqpm continues to guarantee $O(\sqrt{T})$ regret, but the
analysis requires care. We leave this extension to future work.

\section{Contextual partial monitoring: reduction to offline regression}\label{app_section_offline}

For definitions and examples of offline outcome estimation oracles, Please see Section~\ref{app_section_offline_estimation_oracles}.

\subsection{\oedpm}\label{app_section_oe2dpmside}

The bounds on $\mathcal{V}_{\gamma}(\hat{q}(x), p(x))$ established for globally observable games
(Lemma~\ref{lemma_globally_observable_doec}) and for locally observable games
(Lemma~\ref{lemma_locally_observable_doec}), under the exploration strategies \egqpm and \sqpm respectively, remain
valid in the offline setting, since both hold for every $\hat{q} \in \Delta(d)$ rather than for a particular
estimate.

What does not carry over is the regret decomposition of Theorem~\ref{theorem_doec_regret}. Inspecting the guarantee of
Assumption~\ref{assumption_hellinger_closeness_offline_contextual}, an offline oracle's Hellinger bound makes two demands
of the sampling policy: it must be held fixed over the rounds to which the bound is applied, and it must be the
policy that collected the data from which $\hat{q}$ was estimated.

The first is obtained by freezing the policy within each epoch (Line~\ref{lin_estimate_freeze}). The second is
where coverage does the work. Because $\mathrm{Coverage}(p, \lambda)$ is a supremum over pairs of models in the
class and never references $\hat{q}$, a policy with small coverage distinguishes \emph{any} two members of
$\mathcal{Q}$, in particular $\hat{q}$ and $q^{\ast}$, whichever estimate the oracle happens to
return. Lemma~\ref{ope_lemma} makes this precise.

The DEC of \citet{foster21thestatistical} lacks this property: its exploration term
$-\gamma \cdot \EE_{a \sim p}\sbr{ D^2_{\mathrm{H}}\big( \boldsymbol{S}_a \hat{q}(x) \,\big|\, \boldsymbol{S}_a q \big) }$ is anchored at
$\hat{q}$. In an epoch-based algorithm (Lines~\ref{lin_epoch_start}--\ref{lin_epoch_end}), the estimation error
for a given $\hat{q}$ therefore cannot be controlled using a policy from a different epoch, which played no part
in producing it.

Now we present our algorithm \oedpm.

\begin{algorithm}[h!]
	\caption{\oedpm}
	\label{algorithm_oe2d_pmside}
	\begin{algorithmic}[1]
		\STATE \textbf{parameters:}\\
		\spacex Offline estimation oracle $\mathrm{Alg}_{\mathrm{Off}}$ \\
		\spacex Exploration parameters $\{\gamma_{\ell}\}_{\ell=1}^{L}$ \\
		\spacex Error rates $\{\delta_{\ell}\}_{\ell=1}^{L}$\\
		\spacex Epoch schedule $0 = \tau_0 < \tau_1 < \cdots < \tau_{L}$ \\
		\FOR{$t = 1$ \textbf{to} $\tau_1$}
		\STATE  Observe the context $x_t$
		\STATE  Find $p_t$ that approximately minimizes, 
		\begin{align*}
			\sup_{\lambda \in \Delta(\Pi)}  \mathrm{Coverage}(p, \lambda).
		\end{align*}
		\STATE  Choose action $A_t \sim p_t$ and observe feedback $\sigma_t$
		\ENDFOR
		\FOR {$\ell = 2,3, \cdots, L$}  \label{lin_epoch_start} 
		\STATE \label{lin_estimate_freeze} Compute estimate $\hat{q}_{\ell}  = \mathrm{Alg}_{\mathrm{Off}} \left(\big\{A_s, x_s,  \sigma_s\big\}_{s=\tau_{\ell-2} + 1}^{\tau_{\ell-1}}\right)$ 
		\FOR{$t = \tau_{\ell-1} + 1$ \textbf{to} $\tau_{\ell }$ }
		\STATE Receive context $x_t$\\
		\STATE Find $p_t$ that approximately minimizes
		\begin{align*}
			\mathcal{V}_{\gamma_\ell}(\hat{q}_{\ell}(x_t), p)
			:=
			\sup_{\lambda \in \Delta(\Pi)} \underbrace{p^\T \boldsymbol{L} \hat{q}_{\ell}(x_t)  -  \lambda^\T  \boldsymbol{L} \hat{q}_{\ell}(x_t) }_{\mathrm{exploit}}+  \frac{1}{\gamma_{\ell}} \cdot \underbrace{\mathrm{Coverage}(p, \lambda)}_{\mathrm{explore} } .
		\end{align*}
		\STATE Choose action $A_t \sim p_t$ and observe feedback $\sigma_t$.
		\ENDFOR 
		\ENDFOR \label{lin_epoch_end}
	\end{algorithmic}
	
\end{algorithm}

 Furthermore, we can employ Algorithm~\ref{alg:explore-go} for globally observable games, such that for
 every $\gamma_{\ell} > 16 k^2 \beta^2$, and every $\hat{q} \in \Delta(d)$, Alg.~\ref{alg:explore-go} outputs $p$ such that
 \begin{align*}
 	\mathcal{V}_{\gamma_{\ell}}\big(\hat{q}, p \big)
 	\leq O\!\left( k\beta \sqrt{\frac{1}{\gamma_{\ell}}} \right). \tag{Lemma~\ref{lemma_globally_observable_doec}}
 \end{align*}
 
 and for locally observable games, we can employ Algorithm~\ref{alg:explore-lo}, such that For
 every $\gamma_{\ell} > 0$, and every $\hat{q} \in \Delta(d)$, Alg.~\ref{alg:explore-go} run with $\kappa = \frac{\gamma}{8k^2 m^2}$ and $\mu = k$, outputs $p$ such that
 \begin{align*}
 	\mathcal{V}_{\gamma_{\ell}}\big(\hat{q}, p \big)
 	\leq O\!\left(  \frac{m^2 k^3}{\gamma_{\ell}} \right). \tag{Lemma~\ref{lemma_locally_observable_doec}}
 \end{align*}
 
So in both cases, we have algorithms to find policies that establishes an upper bound for our objective function. Let us denote it by $V_{\ell}$. So,
\begin{align*}
	V_{\ell} &= O\!\left( k\beta \sqrt{\frac{1}{\gamma_{\ell}}} \right) \tag{for globally observable games}\\
	V_{\ell} &=  O\!\left(  \frac{m^2 k^3}{\gamma_{\ell}} \right) \tag{for locally observable games}
\end{align*}

Furthermore, within a single epoch $\ell$ for all round $t \in [\tau_{\ell-1}+1, \tau_{\ell}]$, the estimator does not change($\hat{q}_{\ell}(\cdot)$), only the context $x_t$ changes. So, within an epoch $\ell$, 
\begin{align*}
	x_{t_1} = x_{t_2} &\implies \hat{q}_{\ell}(x_{t_1}) =  \hat{q}_{\ell}(x_{t_2}) \implies p_{t_1} = p_{t_2} \spacex \forall t_1,t_2 \in [\tau_{\ell-1}+1, \tau_{\ell}]
\end{align*}

So, we can always think that within an epoch, every $p_t$ is sampled from a fixed distribution $p_{\ell}(\cdot \mid x_t)$, that is, we can denote,
\begin{align*}
	p_t \leftarrow p_{\ell}(\cdot\mid x_t) \spacex \forall t \in [\tau_{\ell-1}+1, \tau_{\ell}].
\end{align*}

\subsection{Regret analysis} \label{app_section_oe2d_regret_analysis}

Next, for readability and reduce the notations clutter we define the following,

\begin{definition}\label{definition_symbols}
	Define:
	\begin{itemize}
		\item $\mathcal{L}^\star(\lambda \mid x) = \lambda^T \boldsymbol{L} q^*(x)$: The loss incurred by policy $\lambda$ under true outcome distribution $q^*$ for a context $x$.
		\item $\widehat{\mathcal{L}}_{\ell}(\lambda \mid x) = \lambda^T \boldsymbol{L} \hat{q}_{\ell}(x)$: The loss incurred by policy $\lambda$ under estimated distribution $\hat{q}_{\ell}$ for a context $x$.
		\item $\lambda^* = \argmin_{\lambda \in \Delta(\Pi)} \mathcal{L}^\star(\lambda \mid x) $: The true best policy for context $x$.
		\item $\hat{\lambda}_{\ell} = \argmin_{\lambda \in \Delta(\Pi)}\widehat{\mathcal{L}}_{\ell}(\lambda \mid x)$: The estimated best policy for context $x$.
		\item $\mathrm{Reg}(\lambda \mid x) = \mathcal{L}^\star(\lambda \mid x) - \mathcal{L}^\star(\lambda^\star \mid x)$: The regret incurred by policy $\lambda$ under true outcome distribution $q^*$ for a context $x$.
		\item $\widehat{\mathrm{Reg}}_{\ell}(\lambda \mid x) = \widehat{\mathcal{L}}_{\ell}(\lambda \mid x) - \widehat{\mathcal{L}}_{\ell}(\hat{\lambda} \mid x)  $: The regret incurred by policy $\lambda$ under estimated distribution $\hat{q}_{\ell}$ for a context $x$.
		\item $ C(p_{\ell}, \lambda ) = \mathrm{Coverage}(p_{\ell}, \lambda,)$
		\item $E_{\ell} =  \EE_{a \sim p_{\ell-1}} \sbr{D^2_{\mathrm{H}}\left( \boldsymbol{S}_a  \hat{q}_{\ell}(x) \mid \boldsymbol{S}_a  q^*(x)\right) }$
	\end{itemize}
	In all the above definitions, we might drop $(\cdot \mid x)$, when it is evident from the context. For example, we might denote $\mathcal{L}^\star(\lambda \mid x) $ $\rightarrow$ $\mathcal{L}^\star(\lambda)$. 
\end{definition}

Next we present two important properties of our policies and estimators.

\textbf{Certification: }

This property is a direct consequence of our algorithms finding the policies which bounds our objective function by $V_{\ell}$.

$\forall x \in \mathcal{X}$ and $\forall \ell \geq 1$,

\begin{align*}
	\mathcal{V}_{\gamma_{\ell}}(\hat{q}_{\ell}(x), p_{\ell}) = \sup_{\lambda \in \Delta(\Pi)} p_{\ell}^\T \boldsymbol{L} \hat{q}_{\ell}(x)  -  \lambda^\T  \boldsymbol{L} \hat{q}_{\ell}(x) +  \frac{1}{\gamma_{\ell}} \cdot \mathrm{Coverage}(p_{\ell}, \lambda) \leq V_{\ell}
\end{align*}

Hence, 

$\forall x \in \mathcal{X}$, $\forall \ell \geq 1$ and, $\forall \lambda \in \Pi$

\begin{align*}
	 &p_{\ell}^\T \boldsymbol{L} \hat{q}_{\ell}(x)  -  \lambda^\T  \boldsymbol{L} \hat{q}_{\ell}(x) +  \frac{1}{\gamma_{\ell}} \cdot \mathrm{Coverage}(p_{\ell}, \lambda) \leq V_{\ell} \\
	 &\implies \widehat{\mathcal{L}}_{\ell}(p_{\ell} \mid x)  - \widehat{\mathcal{L}}_{\ell}(\lambda \mid x)  +  \frac{1}{\gamma_{\ell}} \cdot C(p_{\ell}, \lambda) \leq V_{\ell}
\end{align*}

Since, this is true for all $x \in \mathcal{X}$, we take the liberty to drop it,

\begin{align}\label{eq_certification}
	\widehat{\mathcal{L}}_{\ell}(p_{\ell})  - \widehat{\mathcal{L}}_{\ell}(\lambda)  +  \frac{1}{\gamma_{\ell}} \cdot C(p_{\ell}, \lambda) \leq V_{\ell}. \tag{certification}
\end{align}

\textbf{Off-policy evaluation (OPE):}

\begin{lemma}[Off-policy evaluation lemma]\label{ope_lemma} 
	With all the definitions in Definition~\ref{definition_symbols},  $\forall \lambda,  \forall \ell \geq 2  $, and $\forall x$,
	\begin{align}\label{eq_offpolicy_evaluation}
		\bigg\lvert \widehat{\mathcal{L}}_{\ell}(\lambda - p_{\ell-1}\mid x) - \mathcal{L}^\star(\lambda  - p_{\ell -1} \mid x) \bigg\rvert \leq \frac{1}{6 \gamma_{\ell}} \cdot C(p_{\ell-1}, \lambda)  + \frac{3\gamma_{\ell}}{2}E_{\ell}. \tag{Off-policy evaluation}  
	\end{align}
\end{lemma}
\begin{proof}
	\begin{align*}
		&\left(	\widehat{\mathcal{L}}_{\ell}(\lambda - p_{\ell -1}\mid x) - \mathcal{L}^\star(\lambda  - p_{\ell-1} \mid x)\right)^2 \\
		&\spacex =\left( (\lambda - p_{\ell -1} )^T \boldsymbol{L} \hat{q}_{\ell}(x) - (\lambda - p_{\ell -1} )^T \boldsymbol{L} q^*(x)\right)^2\\
		&\spacex= \frac{\left((\lambda - p_{\ell -1} )^T \boldsymbol{L} ( \hat{q}_{\ell}(x) - q^*(x) )\right)^2}{ \EE_{a \sim p_{\ell -1}} \sbr{D^2_{\mathrm{H}}\left( \boldsymbol{S}_a  \hat{q}_{\ell}(x) \mid \boldsymbol{S}_a  q^*(x)\right) }  } \cdot \EE_{a \sim p_{\ell-1}} \sbr{D^2_{\mathrm{H}}\left( \boldsymbol{S}_a  \hat{q}_{\ell}(x) \mid \boldsymbol{S}_a  q^*(x)\right) }\\
		&\spacex\leq  \mathrm{Coverage}(p_{\ell-1}, \lambda) \cdot  \EE_{a \sim p_{\ell-1}} \sbr{D^2_{\mathrm{H}}\left( \boldsymbol{S}_a  \hat{q}_{\ell}(x) \mid \boldsymbol{S}_a  q^*(x)\right) } \\
		&\bigg\lvert \widehat{\mathcal{L}}_{\ell}(\lambda - p_{\ell-1}\mid x) - \mathcal{L}^\star(\lambda  - p_{\ell -1} \mid x) \bigg\rvert \\
		&\spacex\leq \sqrt{\mathrm{Coverage}(p_{\ell -1}, \lambda) \cdot \EE_{a \sim p_{\ell-1}} \sbr{D^2_{\mathrm{H}}\left( \boldsymbol{S}_a  \hat{q}_{\ell}(x) \mid \boldsymbol{S}_a  q^*(x)\right) } }\\
		&\spacex\leq \frac{1}{6 \gamma_{\ell}} \cdot \mathrm{Coverage}(p_{\ell-1}, \lambda)  + \frac{3\gamma_{\ell}}{2}\EE_{a \sim p_{\ell-1}} \sbr{D^2_{\mathrm{H}}\left( \boldsymbol{S}_a  \hat{q}_{\ell}(x) \mid \boldsymbol{S}_a  q^*(x)\right) } \tag{AM-GM} \\
		&\spacex= \frac{1}{6 \gamma_{\ell}} \cdot C(p_{\ell-1}, \lambda)  + \frac{3\gamma_{\ell}}{2}E_{\ell}.  
	\end{align*}
	
\end{proof}
The second term is the generalization error of the offline oracle: $\hat{q}_{\ell}$ is estimated from data
collected in the previous epoch under $p_{\ell-1}$, so
Assumption~\ref{assumption_hellinger_closeness_offline_contextual} (Hellinger boundedness) 
applies to it directly. This coupling between the
estimate and the policy that produced it is what makes the offline analysis more delicate than the online one,
where the estimator is updated as the data arrive.

The first term involves no estimate at all: $\mathrm{Coverage}(p_{\ell-1}, \lambda)$ depends
only on the model class, on $\lambda$ and on $p_{\ell-1}$. A policy that makes it small therefore controls the
evaluation error for every $\hat{q} \in \mathrm{co}(\mathcal{Q}_{x})$ at once, in particular for whichever
estimate the oracle returns in epoch $\ell$.

Now we present an abstract regret theorem about epoch-based offline oracle-efficient interactive learning algorithm from \citet{qin2026comb}.

\begin{theorem}[Generalization of regret concentration theorem from \citet{qin2026comb}] \label{app_theorem_regret_concentration}
	Let $\Lambda$ be a set of policies, $\mathcal{L}(\cdot \mid x): \Lambda \to [0, 1]$ a value function and $\mathcal{L}^\star(\cdot \mid x)$ and $\widehat{\mathcal{L}}_{\ell}(\cdot \mid x)$, $\ell \in [L]$ are value functions measured under the true model and estimated models respectively.
	$\lambda^\star \in \Lambda$ and $\hat{\lambda}_{\ell} \in \Lambda$ are respectively the optimal policies under respective true and estimated models. Let $C_{\ell} : \Lambda \times \Lambda \to \RR_+$ be coverage maps, $p_{\ell} \in \Lambda$ are decisions, $0 < \gamma_1 \leq \gamma_2 \leq \cdots \leq \gamma_L$ learning rates, and $V_{\ell}, E_{\ell} \geq 0$ numbers such that,
	\begin{enumerate}
		\item[\textup{(C1)}] \textup{(certificate)} for every $\ell \geq 1$ and every $\lambda \in \Lambda$,
		\[
		\widehat{\mathcal{L}}_{\ell}(p_{\ell} \mid x) - \widehat{\mathcal{L}}_{\ell}(\lambda\mid x)  + \frac{1}{\gamma_{\ell}}\, C_{\ell}(p_{\ell}, \lambda) \leq V_{\ell}.
		\]
		\item[\textup{(C2)}] \textup{(off-policy evaluation)} for every $\ell \geq 2$ and every $\lambda \in \Lambda$,
		\[
		\abs{ \widehat{\mathcal{L}}_{\ell}(\lambda - p_{\ell-1}\mid x) - \mathcal{L}^\star(\lambda- p_{\ell-1}\mid x) }
		\leq
		\frac{1}{6 \gamma_{\ell}}\, C_{\ell-1}(p_{\ell-1}, \lambda) + \frac{3 \gamma_{\ell}}{2}\, E_{\ell} .
		\]
	\end{enumerate}
	Define $G_1 \coloneqq 1$ and, for $\ell \geq 2$,
	\[
	\gamma_{\ell} G_{\ell} \coloneqq \gamma_{\ell-1} G_{\ell-1} + \gamma_{\ell-1} V_{\ell-1} + 9 \gamma_{\ell}^2 E_{\ell} ,
	\qquad\text{i.e.}\qquad
	G_{\ell} = \frac{1}{\gamma_{\ell}} \sbr{ \gamma_1 + \sum_{i=2}^{\ell} ( \gamma_{i-1} V_{i-1} + 9 \gamma_i^2 E_i ) }  .
	\]
	Then, for every $\ell \in [L]$ and every $\lambda \in \Lambda$,
	\begin{align*}
		\REG(\lambda\mid x) \leq 2 \HREG_{\ell}(\lambda\mid x) + G_{\ell}
		\qquad\text{and}\qquad
		\HREG_{\ell}(\lambda\mid x) \leq 2 \REG(\lambda\mid x) + G_{\ell} .
	\end{align*}
	Consequently $\REG(p_{\ell}\mid x) \leq 2 V_{\ell} + G_{\ell}$ for every $\ell$, and
	\begin{align*}
		\sum_{\ell=1}^L (\tau_{\ell} - \tau_{\ell-1})\, \REG(p_{\ell}\mid x)
		\leq
		\tau_1 + \sum_{\ell =2}^L \frac{\tau_{\ell} - \tau_{\ell-1}}{\gamma_{\ell}} \sbr{ 2 \gamma_{\ell} V_{\ell} + \gamma_1 + \sum_{i=2}^{\ell} ( \gamma_{i-1} V_{i-1} + 9 \gamma_i^2 E_i) }  .
	\end{align*}
\end{theorem}

Proof is deferred to Appendix~\ref{section_lemmata_contextual}.
We can apply this theorem directly and obtain a per-context regret bound. Now we state the main theorem that bounds the regret with high probability for any epoch schedule.

\begin{theorem}[Regret of \oedpm]\label{main_theorem_oe2d_pmside}
	Suppose $\mathrm{Alg}_{\mathrm{Off}}$ has a Hellinger regret bound $\overline{\mathrm{Reg}}_{\mathrm{Off}}^{\mathrm{H}}\left( \mathcal{Q}, T, \delta \right)$. For any non-degenerate globally observable partial monitoring game without duplicate actions, assume that $x_t \stackrel{i.i.d}{\sim} \mathcal{D}_x$, then for any $\delta \in (0,1)$, \oedpm run with
	 \begin{enumerate}
		\item exploration rule Algorithm~\ref{alg:explore-go} for globally observable game (with $\gamma_{\ell} > 16k^2\beta^2 \; \forall \ell \in [L] $).
		\item exploration rule Algorithm~\ref{alg:explore-lo} for locally observable game (with $\kappa = \frac{\gamma_{\ell}}{8k^2m^2} \; \forall \ell \in [L] $ and $\mu = k$)
		\item exploration parameters such that, $0 < \gamma_1 \leq \gamma_2 \leq \cdots \leq \gamma_L$
	\end{enumerate} satisfies,
	
	\begin{align*}
	&\widetilde{\mathrm{Reg}}_T \leq  L \cdot \tau_1 + 2 L \max_{\ell \in [L]}  \frac{\tau_{\ell} - \tau_{\ell-1}}{\gamma_{\ell}} \left( 2 \gamma_{\ell} V_{\ell} + \gamma_1 + \sum_{i=2}^{\ell} \left( \gamma_{i-1} V_{i-1} + 9 \gamma_i^2 \overline{\mathrm{Reg}}_{\mathrm{Off}}^{\mathrm{H}} (\mathcal{Q}, \tau_{i} - \tau_{i-1} , \delta_{i}) \right) \right).
\end{align*}
where,
\begin{itemize}
	\item $\widetilde{\mathrm{Reg}}_T =  \EE_{x_1, \ldots, x_T \stackrel{i.i.d}{\sim} \mathcal{D}_x}\sbr{\mathrm{Reg}_T} $ 
	\item $V_{\ell} = O\!\left( k\beta \sqrt{\frac{1}{\gamma_{\ell}}} \right) $ for globally observable games
	\item $V_{\ell} =  O\!\left(  \frac{m^2 k^3}{\gamma_{\ell}} \right)$ for locally observable games
\end{itemize}
	\begin{proof}	
		From Theorem~\ref{app_theorem_regret_concentration}, the regret per-context is bounded by,
		\begin{align*}
			\sum_{\ell=1}^L (\tau_{\ell} - \tau_{\ell-1})\, \REG(p_{\ell}\mid x)
			&\leq \tau_1 + \sum_{\ell=2}^L (\tau_{\ell} - \tau_{\ell-1}) ( 2 V_{\ell} + G_{\ell} )
			\tag{$\REG(p_1\mid x) \leq 1$}
			\\
			&= \tau_1 + \sum_{\ell=2}^L \frac{\tau_{\ell} - \tau_{\ell-1}}{\gamma_{\ell}} \sbr{ 2 \gamma_{\ell} V_{\ell} + \gamma_1 + \sum_{i=2}^{\ell} ( \gamma_{i-1} V_{i-1} + 9 \gamma_i^2 E_i )}.
		\end{align*}
		Let us denote,
		\begin{align*}
			\widetilde{\mathrm{Reg}}_T =  \EE_{x \stackrel{i.i.d}{\sim} \mathcal{D}_x}\sbr{\mathrm{Reg}_T}
		\end{align*}
		Here, note that, even though the randomness w.r.t to contexts are removed, $\widetilde{\mathrm{Reg}}_T$, is still a random variable in terms of policy.
		Hence,
		\begin{align*}
			\widetilde{\mathrm{Reg}}_T  &= \EE_{x \stackrel{i.i.d}{\sim} \mathcal{D}_x}\sbr{\mathrm{Reg}_T} =  \EE \sbr{\sum_{\ell=1}^L \sum_{t=\tau_{\ell-1}}^{\tau_{\ell} }\REG(p_\ell\mid x_t)}\\
			&\leq \EE\sbr{\sum_{\ell=1}^L  \tau_1 + \sum_{\ell=2}^L \frac{\tau_{\ell} - \tau_{\ell-1}}{\gamma_{\ell}} \left( 2 \gamma_{\ell} V_{\ell} + \gamma_1 + \sum_{i=2}^{\ell} ( \gamma_{i-1} V_{i-1} + 9 \gamma_i^2 E_i )\right)  }\\
			&\leq L \cdot \tau_1 + 2 L \max_{\ell \in [L]}   \frac{\tau_{\ell} - \tau_{\ell-1}}{\gamma_{\ell}} \left( 2 \gamma_{\ell} V_{\ell} + \gamma_1 + \sum_{i=2}^{\ell} \left( \gamma_{i-1} V_{i-1} + 9 \gamma_i^2 \EE_{x \stackrel{i.i.d}{\sim} \mathcal{D}_x}\sbr{E_i} \right)\right) .
		\end{align*}
		
		If the contexts are i.i.d such that, $x_t \stackrel{i.i.d}{\sim} \mathcal{D}_x$, then,
		
		\begin{align*}
			&\widetilde{\mathrm{Reg}}_T  \leq L \cdot \tau_1\\
			&\spacex + 2 L \max_{\ell \in [L]}   \frac{\tau_{\ell} - \tau_{\ell-1}}{\gamma_{\ell}} \left( 2 \gamma_{\ell} V_{\ell} + \gamma_1 + \sum_{i=2}^{\ell} \left( \gamma_{i-1} V_{i-1} + 9 \gamma_i^2 \EE_{x \stackrel{i.i.d}{\sim} \mathcal{D}_x}\EE_{a \sim p_{\ell-1}} \sbr{D_{H}^2 \left( \boldsymbol{S}_a \hat{q}_{\ell}(x) \mid \boldsymbol{S}_a q^*(x)\right)} \right) \right)\\
			&\spacex \leq L \cdot \tau_1 + 2 L \max_{\ell \in [L]}  \frac{\tau_{\ell} - \tau_{\ell-1}}{\gamma_{\ell}} \left( 2 \gamma_{\ell} V_{\ell} + \gamma_1 + \sum_{i=2}^{\ell} \left( \gamma_{i-1} D_{i-1} + 9 \gamma_i^2 \frac{\mathrm{Reg}_{\mathrm{Off}}^{\mathrm{H}} (\mathcal{Q}, \tau_{i} - \tau_{i-1}) }{\tau_{i} - \tau_{i-1}}\right) \right).
		\end{align*}
		
		We know that, with probability of at least 1- $\delta'$,
		\begin{align*}
			\mathrm{Reg}_{\mathrm{Off}}^{\mathrm{H}} (\mathcal{Q}, \tau_{i} - \tau_{i-1})  \leq \overline{\mathrm{Reg}}_{\mathrm{Off}}^{\mathrm{H}} (\mathcal{Q}, \tau_{i} - \tau_{i-1} , \delta)
		\end{align*} 
		Set $\delta_{\ell}  = \frac{\delta}{L} $, with a probability at least $1-\delta$,
		
		\begin{align*}
			&\widetilde{\mathrm{Reg}}_T \leq  L \cdot \tau_1 + 2 L \max_{\ell \in [L]}  \frac{\tau_{\ell} - \tau_{\ell-1}}{\gamma_{\ell}} \left( 2 \gamma_{\ell} V_{\ell} + \gamma_1 + \sum_{i=2}^{\ell} \left( \gamma_{i-1} V_{i-1} + 9 \gamma_i^2 \frac{\overline{\mathrm{Reg}}_{\mathrm{Off}}^{\mathrm{H}} (\mathcal{Q}, \tau_{i} - \tau_{i-1} , \delta_{i})}{ \tau_{i} - \tau_{i-1}} \right) \right).
		\end{align*}
		
	\end{proof}
	
\end{theorem}
	The final regret depends on epoch schedule. We choose a doubling epoch schedule.
	
\begin{cor}
		Suppose $\mathrm{Alg}_{\mathrm{Off}}$ has a Hellinger regret bound $\overline{\mathrm{Reg}}_{\mathrm{Off}}^{\mathrm{H}}\left( \mathcal{Q}, T, \delta \right)$. For any non-degenerate globally observable partial monitoring game without duplicate actions, assume that $x_1, \ldots, x_T \stackrel{i.i.d}{\sim} \mathcal{D}_x$, then for any $\delta \in (0,1)$, the algorithm \oedpm run with,
		\begin{enumerate}
			\item exploration rule Algorithm~\ref{alg:explore-go} for globally observable game (with $\gamma_{\ell} > 16k^2\beta^2 \; \forall \ell \in [L] $).
			\item exploration rule Algorithm~\ref{alg:explore-lo} for locally observable game (with $\kappa = \frac{\gamma_{\ell}}{8k^2m^2} \; \forall \ell \in [L] $ and $\mu = k$)
			\item doubling schedule, such that $\tau_{\ell} = 2^{\ell}$
			\item exploration parameters such that, $0 < \gamma_1 \leq \gamma_2 \leq \cdots \leq \gamma_L$
			\item with assumption $\overline{\mathrm{Reg}}_{\mathrm{Off}}^{\mathrm{H}} (\mathcal{Q}, t , \delta)$ does not depend on $t$. Generally this is the case.
			\item with $\delta_{\ell} = \frac{\delta}{L}$
		\end{enumerate}    satisfies the following guarantees,
		
	\begin{enumerate}
		\item For globally observable games, with $\gamma_{\ell} = \left(\frac{ C \cdot  \tau_{\ell}\cdot k \beta}{\overline{\mathrm{Reg}}_{\mathrm{Off}}^{\mathrm{H}} (\mathcal{Q}, \tau_{\ell}/2, \delta/L)} \right)^{\frac{2}{3}}$ for some constant $C$,
		\begin{align*}
			&\widetilde{\mathrm{Reg}}_T \leq  O\left( T^{\frac{2}{3}}\log^2(T) (k \beta )^{\frac{2}{3}} (\overline{\mathrm{Reg}}_{\mathrm{Off}}^{\mathrm{H}} (\mathcal{Q}, T , \delta/ \log (T)))^{\frac{1}{3}} \right) 
		\end{align*}
		\item For locally observable games, with $\gamma_{\ell}= \sqrt{\frac{ D \cdot  \tau_{\ell}\cdot k^3 m^2 }{\overline{\mathrm{Reg}}_{\mathrm{Off}}^{\mathrm{H}} (\mathcal{Q}, \tau_{{\ell}}/2, \delta_{{\ell}})} }$  for some constant $D$,
			\begin{align*}
			\widetilde{\mathrm{Reg}}_T = O\left( \log^2(T) \sqrt{T\cdot k^3 m^2 \overline{\mathrm{Reg}}_{\mathrm{Off}}^{\mathrm{H}} (\mathcal{Q}, T/2, \delta/\log(T) ) }  \right) \\
		\end{align*}
	\end{enumerate}
	where, $\widetilde{\mathrm{Reg}}_T =  \EE_{x \stackrel{i.i.d}{\sim} \mathcal{D}_x}\sbr{\mathrm{Reg}_T} $.

\end{cor}

\begin{proof}
	First of all, for doubling epoch schedule,
	\begin{align*}
		\tau_{L} &= O(T)\\
		\tau_{\ell} - \tau_{\ell-1} &= \frac{\tau_{\ell}}{2} \\
		L &= O(\log T)
	\end{align*}
	
	\begin{align*}
		&\widetilde{\mathrm{Reg}}_T \leq  L \cdot \tau_1 + 2 L \max_{\ell \in [L]}  \frac{\tau_{\ell} - \tau_{\ell-1}}{\gamma_{\ell}} \left( 2 \gamma_{\ell} V_{\ell} + \gamma_1 + \sum_{i=2}^{\ell} \left( \gamma_{i-1} V_{i-1} + 9 \gamma_i^2 \frac{\overline{\mathrm{Reg}}_{\mathrm{Off}}^{\mathrm{H}} (\mathcal{Q}, \tau_{i} - \tau_{i-1} , \delta_{i})}{ \tau_{i} - \tau_{i-1}} \right) \right).\\
		&= L \cdot \tau_1 + 2 L \max_{\ell \in [L]}  \frac{\tau_{\ell}}{2 \gamma_{\ell}} \left( 2 \gamma_{\ell} V_{\ell} + \gamma_1 + \sum_{i=2}^{\ell} \left( \gamma_{i-1} V_{i-1} + 9 \gamma_i^2 \frac{\overline{\mathrm{Reg}}_{\mathrm{Off}}^{\mathrm{H}} (\mathcal{Q}, \tau_{i}/2 , \delta_{i})}{ \tau_{i}/2} \right) \right).\\
		&\leq  L \cdot \tau_1 + 2 L \max_{\ell \in [L]}  \frac{\tau_{\ell}}{2 \gamma_{\ell}} \left(\gamma_1 +   \sum_{i=2}^{\ell} \gamma_{i} \left( 2 V_{i} + 18 \gamma_i\frac{\overline{\mathrm{Reg}}_{\mathrm{Off}}^{\mathrm{H}} (\mathcal{Q}, \tau_{i}/2 , \delta_{i})}{ \tau_{i}} \right) \right).\\
		&\leq L \cdot \tau_1 + 2 L^2 \max_{\ell \in [L]}  \frac{\tau_{\ell}}{2 \gamma_{\ell}} \left(\gamma_1 +   \max_{i \in [\ell]}   \gamma_{i} \left( 2 V_{i} + 18 \gamma_i\frac{\overline{\mathrm{Reg}}_{\mathrm{Off}}^{\mathrm{H}} (\mathcal{Q}, \tau_{i}/2 , \delta_{i})}{ \tau_{i}}  \right) \right).\\
		&= L \cdot \tau_1 + 2 L^2 \max_{\ell \in [L]}  \frac{\tau_{\ell}}{2 \gamma_{\ell}} \left(\gamma_1 +   \max_{i \in [\ell]}   \frac{\gamma_{i} }{\tau_i }\left( 2 \tau_i V_{i} + 18 \gamma_i \overline{\mathrm{Reg}}_{\mathrm{Off}}^{\mathrm{H}} (\mathcal{Q}, \tau_{i}/2 , \delta_{i}) \right) \right).\\
	\end{align*}
	
	For globally observable games,
	
	\begin{align*}
		V_i = C \cdot k \beta \sqrt{\frac{1}{\gamma_{i}}} \tag{for some constant $C$} 
	\end{align*}
	
	\begin{align*}
			\widetilde{\mathrm{Reg}}_T &\leq  &= L \cdot \tau_1 + 2 L^2 \max_{\ell \in [L]}  \frac{\tau_{\ell}}{2 \gamma_{\ell}} \left(\gamma_1 +   \max_{i \in [\ell]}   \frac{\gamma_{i} }{\tau_i }\left( 2 \tau_i C \cdot k \beta \sqrt{\frac{1}{\gamma_{i}}}  + 18 \gamma_i \overline{\mathrm{Reg}}_{\mathrm{Off}}^{\mathrm{H}} (\mathcal{Q}, \tau_{i}/2 , \delta_{i}) \right) \right).
	\end{align*}
	
	set $\gamma_i = \left(\frac{ C \tau_{i}\cdot k \beta}{\overline{\mathrm{Reg}}_{\mathrm{Off}}^{\mathrm{H}} (\mathcal{Q}, \tau_{i}/2, \delta_{i})} \right)^{\frac{2}{3}}$
	
	\begin{align*}
		&\widetilde{\mathrm{Reg}}_T \\
		\spacex &\leq   L \cdot \tau_1 + 2 L^2 \max_{\ell \in [L]}  \frac{\tau_{\ell}}{2 \gamma_{\ell}} \left(\gamma_1 +   \max_{i \in [\ell]} \frac{\gamma_{i}}{\tau_i} \left( 2  \cdot ( C \cdot \tau_i k \beta )^{\frac{2}{3}} (\overline{\mathrm{Reg}}_{\mathrm{Off}}^{\mathrm{H}} (\mathcal{Q}, \tau_{i}/2, \delta_{i}))^{ \frac{1}{3}}  + 18 (C \cdot \tau_i k \beta)^{\frac{2}{3}} (\overline{\mathrm{Reg}}_{\mathrm{Off}}^{\mathrm{H}} (\mathcal{Q}, \tau_{i}/2, \delta_{i}))^{ \frac{1}{3}} \right) \right)\\
		\spacex &=   L \cdot \tau_1 + 2 L^2 \max_{\ell \in [L]}  \frac{\tau_{\ell}}{2 \gamma_{\ell}} \left(\gamma_1 +   20 \max_{i \in [\ell]}\frac{\gamma_{i}}{\tau_i}  ( C \cdot \tau_i k \beta )^{\frac{2}{3}} (\overline{\mathrm{Reg}}_{\mathrm{Off}}^{\mathrm{H}} (\mathcal{Q}, \tau_{i}/2, \delta_{i}))^{ \frac{1}{3}}  \right).\\
		\spacex &=   L \cdot \tau_1 + 2 L^2 \max_{\ell \in [L]}  \frac{\tau_{\ell}}{2 \gamma_{\ell}} \left(\gamma_1 +   20 \max_{i \in [\ell]}\frac{1}{\tau_i}  ( C \cdot \tau_i k \beta )^{\frac{4}{3}} (\overline{\mathrm{Reg}}_{\mathrm{Off}}^{\mathrm{H}} (\mathcal{Q}, \tau_{i}/2, \delta_{i}))^{ - \frac{1}{3}}  \right).\\
		\spacex &\leq L \cdot \tau_1 + 2 L^2 \max_{\ell \in [L]}  \frac{\tau_{\ell}}{2 \gamma_{\ell}} \left(\gamma_1 +   20 \tau_{\ell}^{\frac{1}{3}} ( C \cdot k \beta )^{\frac{4}{3}} ( \max_{i \in [\ell]} \overline{\mathrm{Reg}}_{\mathrm{Off}}^{\mathrm{H}} (\mathcal{Q}, \tau_{i}/2, \delta_{i}))^{ - \frac{1}{3}}  \right).\\
		\spacex &= L \cdot \tau_1 + L^2 \max_{\ell \in [L]}  \frac{ (\overline{\mathrm{Reg}}_{\mathrm{Off}}^{\mathrm{H}} (\mathcal{Q}, \tau_{\ell}/2 , \delta_{\ell}))^{\frac{2}{3}} \tau_{\ell}^{\frac{1}{3}} }{ ( C \cdot k \beta )^{\frac{2}{3}}} \left(\gamma_1 +   20 \tau_{\ell}^{\frac{1}{3}} ( C \cdot k \beta )^{\frac{4}{3}} ( \max_{i \in [\ell]} \overline{\mathrm{Reg}}_{\mathrm{Off}}^{\mathrm{H}} (\mathcal{Q}, \tau_{i}/2, \delta_{i}))^{ - \frac{1}{3}}  \right).\\
		&= O\left( L^2  ( C \cdot k \beta )^{\frac{2}{3}} \max_{\ell \in [L]}  \tau_{\ell}^{\frac{2}{3}}  \frac{\max_{\ell \in [L]} (\overline{\mathrm{Reg}}_{\mathrm{Off}}^{\mathrm{H}} (\mathcal{Q}, \tau_{\ell}/2 , \delta_{\ell}))^{\frac{2}{3}} }{\min_{\ell \in [L]} (\overline{\mathrm{Reg}}_{\mathrm{Off}}^{\mathrm{H}} (\mathcal{Q}, \tau_{\ell}/2 , \delta_{\ell}))^{\frac{1}{3}} }\right) \\
		&= O\left( T^{\frac{2}{3}}\log^2(T) (k \beta )^{\frac{2}{3}} (\overline{\mathrm{Reg}}_{\mathrm{Off}}^{\mathrm{H}} (\mathcal{Q}, T , \delta/ \log (T)))^{\frac{1}{3}} \right) \tag{Here we assume $\overline{\mathrm{Reg}}_{\mathrm{Off}}^{\mathrm{H}} (\mathcal{Q}, t , \delta)$ do not depend on $t$. Generally this is the case.}
	\end{align*}
	
	And for locally observable games,
	
	\begin{align*}
		V_i = D \cdot \frac{m^2 k^3}{\gamma_{i}} \tag{for some constant $D$} 
	\end{align*}
	
	\begin{align*}
		\widetilde{\mathrm{Reg}}_T &\leq  &= L \cdot \tau_1 + 2 L^2 \max_{\ell \in [L]}  \frac{\tau_{\ell}}{2 \gamma_{\ell}} \left(\gamma_1 +   \max_{i \in [\ell]}   \frac{\gamma_{i} }{\tau_i }\left( 2 \tau_i D \cdot m^2 k^3 \frac{1}{\gamma_{i}}  + 18 \gamma_i \overline{\mathrm{Reg}}_{\mathrm{Off}}^{\mathrm{H}} (\mathcal{Q}, \tau_{i}/2 , \delta_{i}) \right) \right).
	\end{align*}
	
	set $\gamma_i = \sqrt{\frac{ D \cdot  \tau_{i}\cdot k^3 m^2 }{\overline{\mathrm{Reg}}_{\mathrm{Off}}^{\mathrm{H}} (\mathcal{Q}, \tau_{i}/2, \delta_{i})} }$

	\begin{align*}
	\widetilde{\mathrm{Reg}}_T &\leq   L \cdot \tau_1 + 2 L^2 \max_{\ell \in [L]}  \frac{\tau_{\ell}}{2 \gamma_{\ell}} \left(\gamma_1 +   20 \max_{i \in [\ell]}   \frac{\gamma_{i} }{\tau_i }\left(   \sqrt{ D \cdot  \tau_{i}\cdot k^3 m^2 \overline{\mathrm{Reg}}_{\mathrm{Off}}^{\mathrm{H}} (\mathcal{Q}, \tau_{i}/2, \delta_{i}) } \right) \right)\\
	&=  L \cdot \tau_1 + 2 L^2 \max_{\ell \in [L]}  \frac{\tau_{\ell}}{2 \gamma_{\ell}} \left(\gamma_1 +   20 D \cdot k^3 m^2\right)\\
	&=  L \cdot \tau_1 + 2 L^2 \max_{\ell \in [L]} \frac{ \sqrt{\tau_{\ell}  \overline{\mathrm{Reg}}_{\mathrm{Off}}^{\mathrm{H}} (\mathcal{Q}, \tau_{\ell}/2, \delta_{\ell})}    }{\sqrt{ D \cdot k^3 m^2}}\left(\gamma_1 +   20 D \cdot k^3 m^2\right)\\
	&= O\left( L^2 \sqrt{T\cdot k^3 m^2 \overline{\mathrm{Reg}}_{\mathrm{Off}}^{\mathrm{H}} (\mathcal{Q}, T/2, \delta/L ) }  \right) \\
	&= O\left( \log^2(T) \sqrt{T\cdot k^3 m^2 \overline{\mathrm{Reg}}_{\mathrm{Off}}^{\mathrm{H}} (\mathcal{Q}, T/2, \delta/\log(T) ) }  \right) \\
	\end{align*}
	
\end{proof}

\section{Empirical studies}\label{app_section_experiments}

\paragraph{Additional LLM-routing details.}
The cached BoolQ dataset contains \(12{,}697\) prompts, of which \(12{,}648\) have parseable responses from both \texttt{Qwen2.5-0.5B-Instruct}, used as the weak model, and \texttt{Llama-2-13B-Chat}, used as the strong reference model. The binary routing outcome indicates whether their cached responses disagree; BoolQ gold labels are not used as routing outcomes. Each prompt is represented by a \(138\)-dimensional context consisting of \(10\) weak-model uncertainty and probability statistics, \(64\) whitened principal components of the weak model's hidden states, and \(64\) whitened principal components of a context-free prompt embedding. All \(12{,}648\) eligible examples are presented online, and outcomes are revealed only after querying the strong model. Each policy makes a routing decision at every round, whereas its estimator is updated only once in $32$ rounds, using all feedback revealed before that round. \textsc{CBPSide}~\citep{bartok2012CBPside,byeongwoo26efficient} uses a linear logistic disagreement model and the confidence radius
\[
\beta_t(x_t)
=
\min\!\left\{
0.5s\sqrt{x_t^\top V_{t-1}^{-1}x_t},
\,0.5
\right\},
\]
where \(s\) is the tuned multiplier. The two \sqpm variants use \(\mu=2\), \(\kappa=s\sqrt{T}\), and inverse-propensity weights capped at \(10\). Their only intended difference is the disagreement-probability estimator. The linear variant fits an \(L_2\)-regularized logistic model with \(C=1\), the \texttt{lbfgs} solver, and a weighted standardization step. The nonlinear variant uses scikit-learn's histogram gradient-boosting (HGB) classifier with binary log loss, \(50\) boosting iterations, learning rate \(0.05\), at most \(15\) leaves per tree, a minimum of \(20\) samples per leaf, \(L_2\) regularization \(1\), and no early stopping. At each eligible update, both estimators are refitted using their complete accumulated inverse-propensity-weighted feedback. \textsc{PG-TS}~\citep{grant21apple} uses a Bayesian logistic disagreement model with
\[
\theta\sim\mathcal{N}\!\left(0,(s\sigma_0)^2I\right),
\qquad \sigma_0=1,
\]
and performs \(15\) Pólya--Gamma Gibbs transitions whenever its posterior is updated. 

Both ETC baselines~\citep{garivier16etc} use an initial exploration budget of $s \cdot T^{2/3}$ rounds to fit
a linear-logistic model and a histogram gradient boosting (HGB) model respectively. The fitted models are then
frozen, and the remaining queries are routed according to them. The random baseline instead routes each query to the strong model with a fixed
probability, independent of context, swept over $\{0, 0.2, 0.4, 0.6, 0.8, 1\}$.

Table~\ref{tab:multiplier} shows how $s$ is used to tune different hyper-parameters for different algorithms. For each policy and each \(\ell_{01}\), we select \(s\in\{0.1,0.3,1,3,10\}\) by minimizing mean total cost over \(20\) paired random permutations of the online arrival order. Random is computed analytically by matching the routing rate of the selected nonlinear \sqpm policy. Figure~\ref{fig:llm_routing} reports mean realized total cost with pointwise \(95\%\) Student-\(t\) confidence intervals. The learning curves in Figure~\ref{fig:routing-cumulative-regret-all} report cumulative excess cost relative to a fixed five-fold cross-fitted HGB reference policy. These extend the results discussed in \Cref{main_section_experiments}. Each reference probability is produced by an HGB model that excludes the evaluated example and is used only for evaluation, never as online feedback. Because multiplier selection and evaluation use the same permutations, all selected-policy results are interpreted as exploratory tuned comparisons.

\paragraph{Observations} From Figure~\ref{fig:llm_routing}, it is evident that, both the linear-logistic and nonlinear HGB instantiations of \sqpm achieve lower mean realized total cost than \textsc{CBPSide}, while the nonlinear HGB variant also outperforms PG-TS. These results illustrate a practical advantage of the oracle-efficient \sqpm formulation: it can incorporate flexible nonlinear prediction models without modifying the surrounding routing policy. Furthermore, generally ETC is expected to perform worse due to $O(T^{2/3})$ regret, however here due to limited test data size  \(12{,}648\) compared to larger dimension \(138\), it performs competitively. It is evident in the learning curves as well (Figure~\ref{fig:routing-cumulative-regret-all}). Moreover, \Cref{fig:routing-accuracy-tradeoff} shows the trade-off between routing accuracy and the rate at which the
strong model is invoked; algorithms achieving higher accuracy at a lower invocation rate are preferable.
\sqpm with the HGB model attains the best trade-off among the algorithms compared.

\begin{table}[h]
	\centering
	\caption{Meaning of the multiplier $s$ for each algorithm.}
	\label{tab:multiplier}
	\begin{tabular}{@{}lll@{}}
		\toprule
		Algorithm & Scaled quantity & Effect of larger $s$ \\
		\midrule
		\textsc{CBPside}         & confidence width $s  w_i$              & more exploration \\
		\textsc{RandCBPside}     & sampled confidence width $s  w_i$      & more exploration \\
		\textsc{SquareCB.PMSide} & IGW rate $\kappa = s\sqrt{T}$       & less exploration \\
		\textsc{ETC}             & exploration length $s\,T^{2/3}$     & more exploration \\
		\textsc{PG-TS}          & prior variance $s\sigma_0$       				& more exploration \\
		\bottomrule
	\end{tabular}
\end{table}

\newcommand{\routingfigdir}{figures/llm-routing}

\begin{figure}[t]
	\centering
	\includegraphics[width=\linewidth]
	{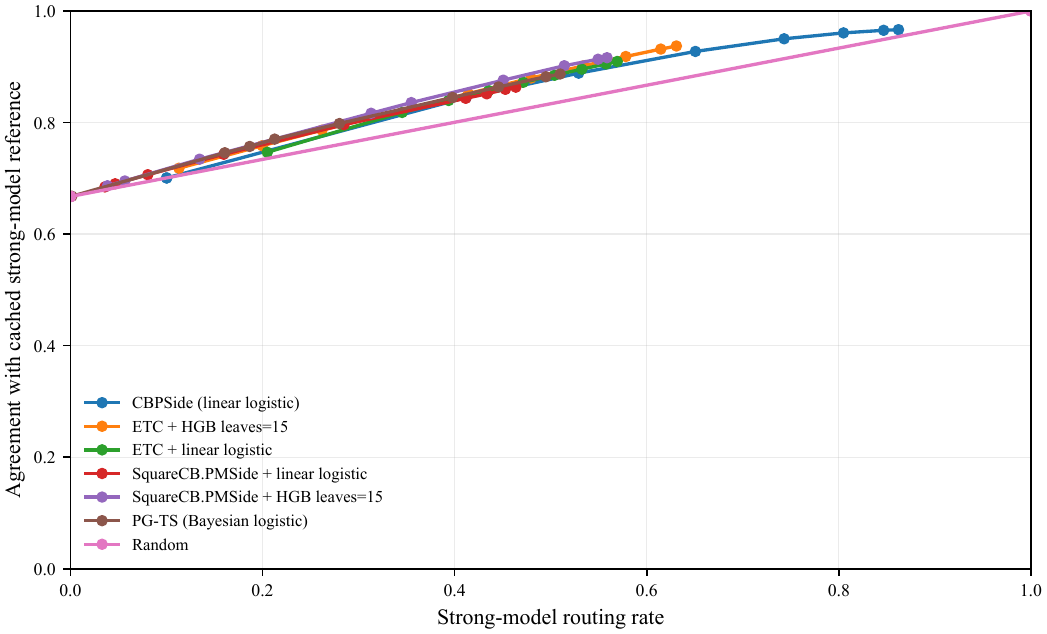}
	\caption{
		Routing–accuracy trade-off on BoolQ.
		The horizontal axis reports the fraction of prompts routed to the
		strong model, while the vertical axis reports agreement with the
		cached strong-model response. Thus, points closer to the upper-left
		corner are preferable, as they achieve higher agreement while making
		fewer strong-model queries.
		Each marker is the mean over 20 shuffled online orders, with confidence
		intervals omitted for visual clarity.
		Along each curve, the connected markers follow increasing values
		$\ell_{01}\in\{1.8,2.0,2.2,2.4,2.6,2.8,3.0,3.2,3.3\}$.
		Increasing $\ell_{01}$ assigns a larger penalty to an incorrect
		weak-model response and therefore generally moves the policy toward
		more strong-model routing and higher agreement.
	}
	\label{fig:routing-accuracy-tradeoff}
\end{figure}

\begin{figure*}[p]
	\centering
	\begin{adjustbox}{
			max totalsize={\textwidth}{0.82\textheight},
			center
		}
		\begin{tabular}{@{}cc@{}}
			\subcaptionbox{\(\ell_{01}=1.8\)
				\label{fig:routing-cum-l01-18}}[0.46\textwidth]{
				\includegraphics[width=\linewidth]
				{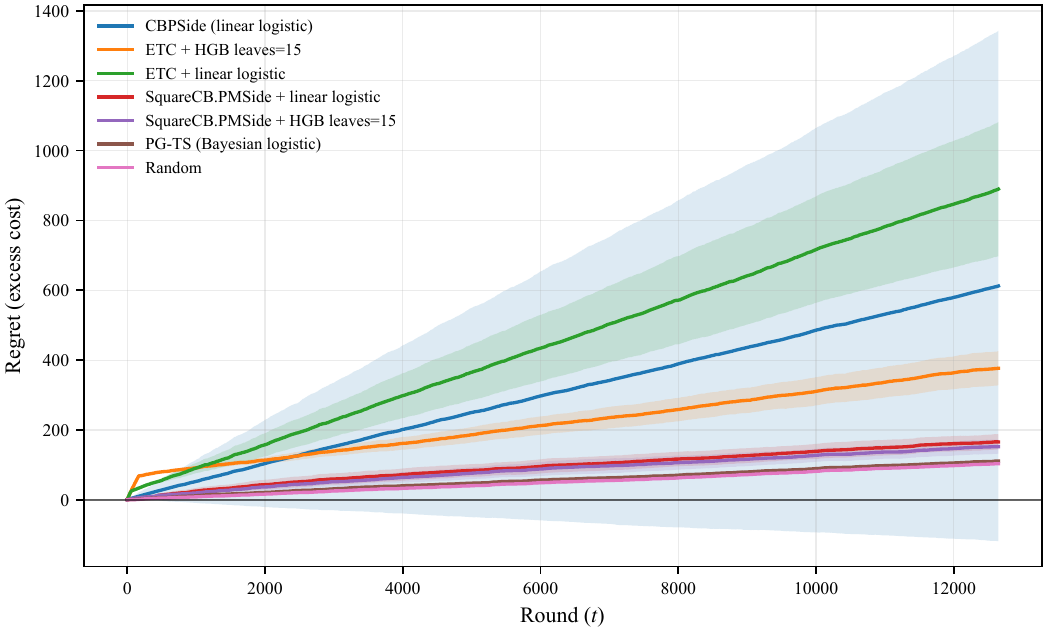}
			}
			&
			\subcaptionbox{\(\ell_{01}=2.0\)
				\label{fig:routing-cum-l01-20}}[0.46\textwidth]{
				\includegraphics[width=\linewidth]
				{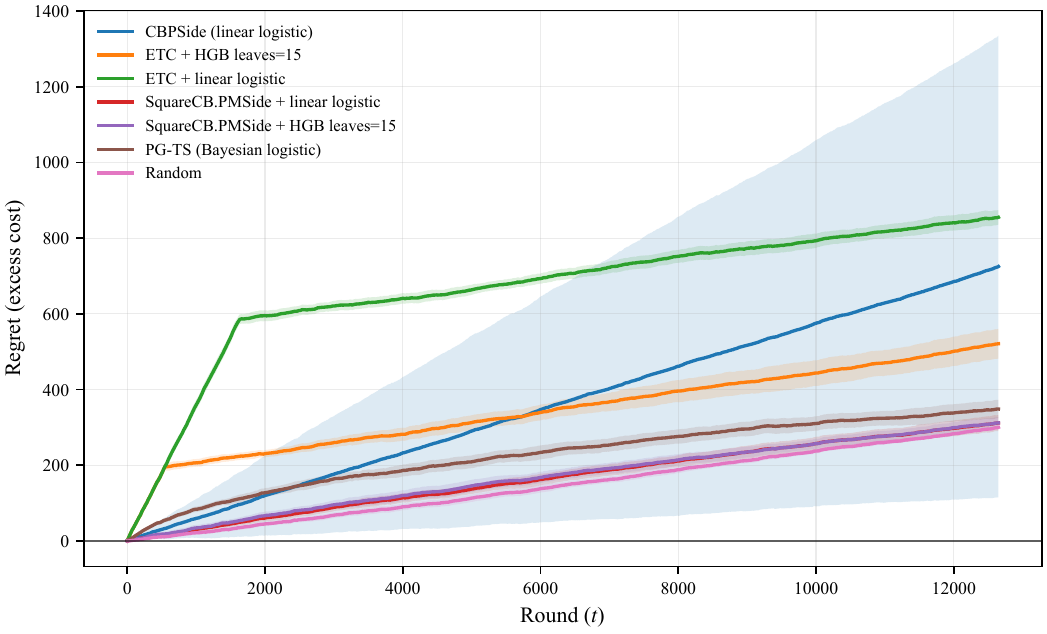}
			}
			\\[2pt]
			
			\subcaptionbox{\(\ell_{01}=2.2\)
				\label{fig:routing-cum-l01-22}}[0.46\textwidth]{
				\includegraphics[width=\linewidth]
				{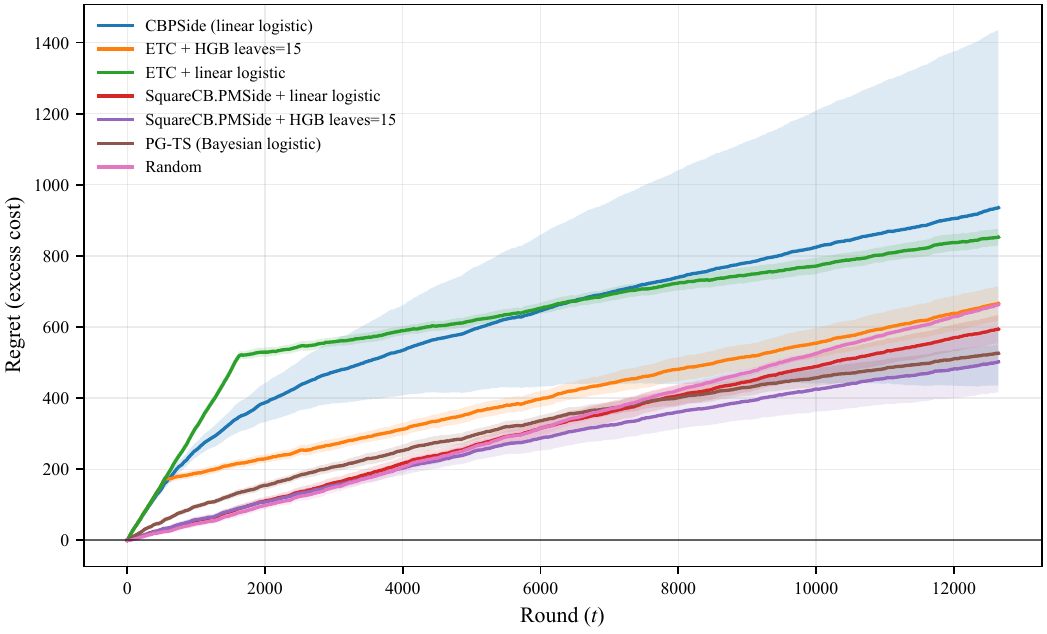}
			}
			&
			\subcaptionbox{\(\ell_{01}=2.4\)
				\label{fig:routing-cum-l01-24}}[0.46\textwidth]{
				\includegraphics[width=\linewidth]
				{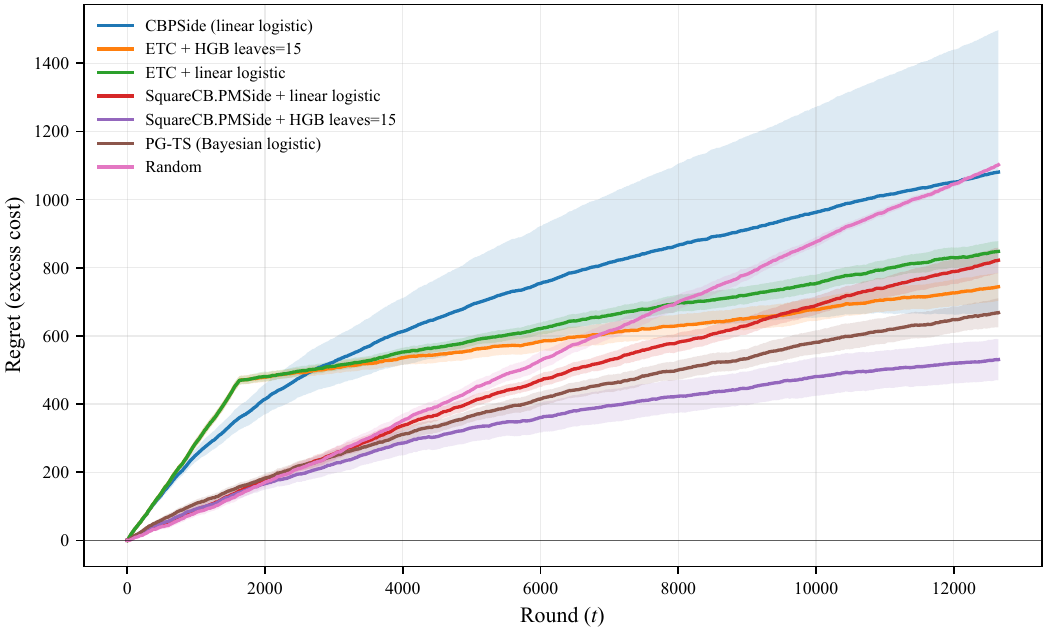}
			}
			\\[2pt]
			
			\subcaptionbox{\(\ell_{01}=2.6\)
				\label{fig:routing-cum-l01-26}}[0.46\textwidth]{
				\includegraphics[width=\linewidth]
				{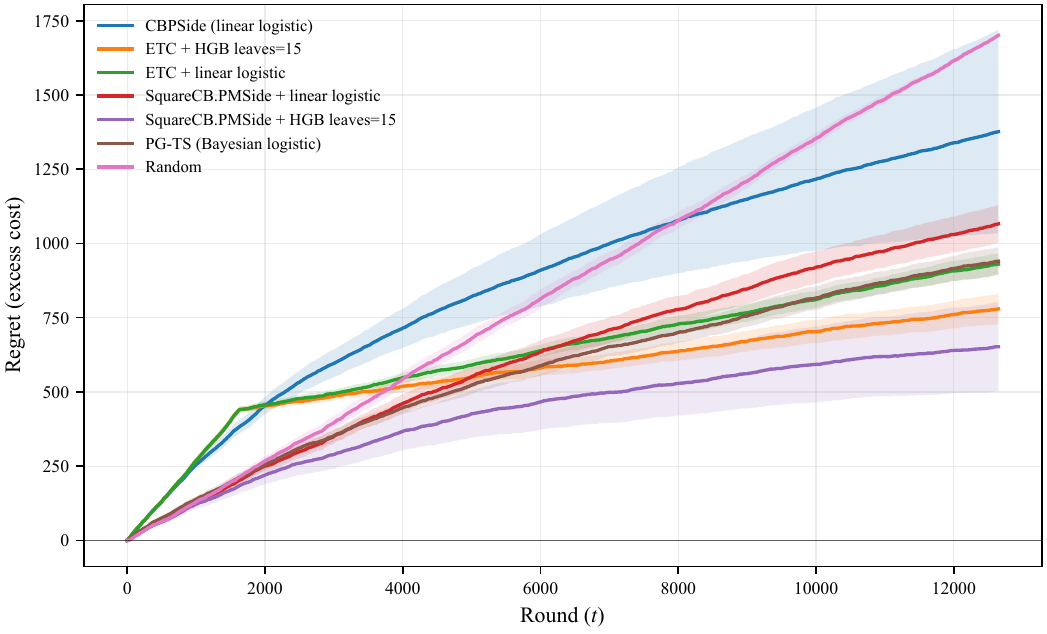}
			}
			&
			\subcaptionbox{\(\ell_{01}=2.8\)
				\label{fig:routing-cum-l01-28}}[0.46\textwidth]{
				\includegraphics[width=\linewidth]
				{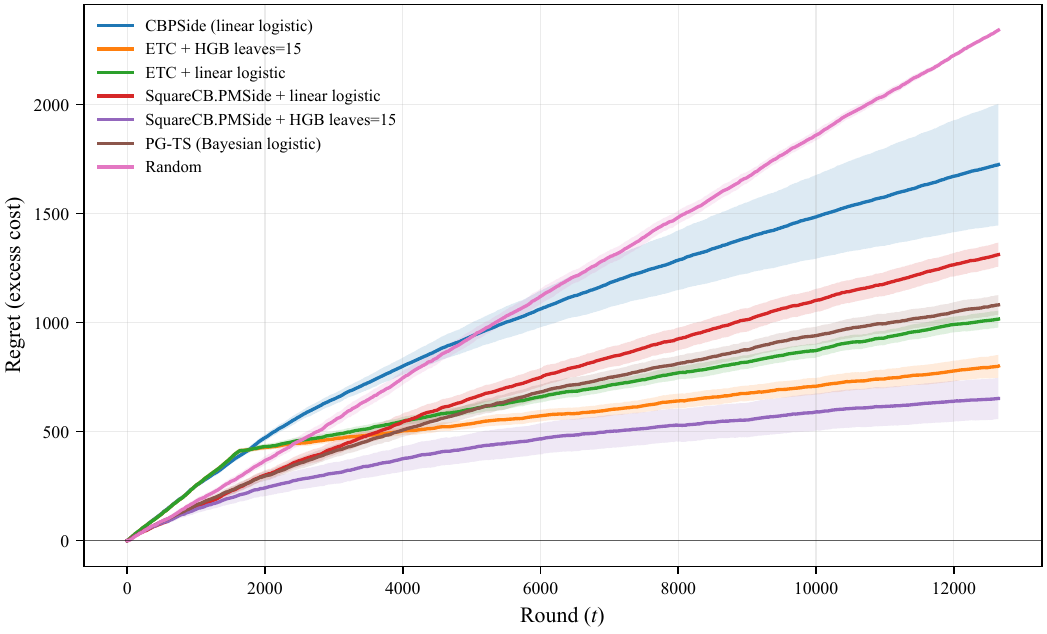}
			}
			\\[2pt]
			
			\subcaptionbox{\(\ell_{01}=3.0\)
				\label{fig:routing-cum-l01-30}}[0.46\textwidth]{
				\includegraphics[width=\linewidth]
				{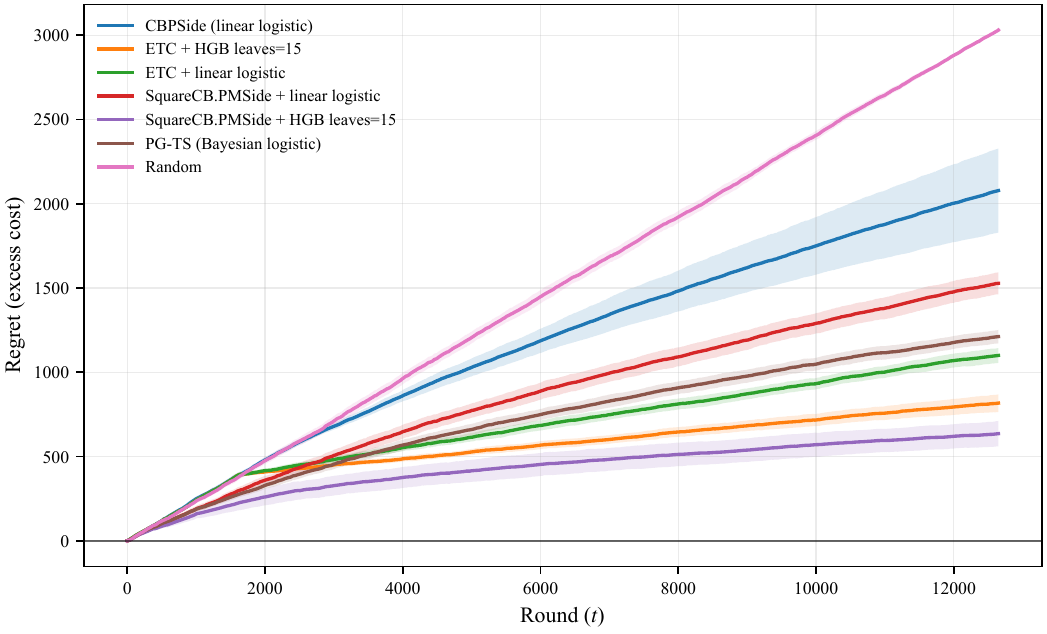}
			}
			&
			\subcaptionbox{\(\ell_{01}=3.2\)
				\label{fig:routing-cum-l01-32}}[0.46\textwidth]{
				\includegraphics[width=\linewidth]
				{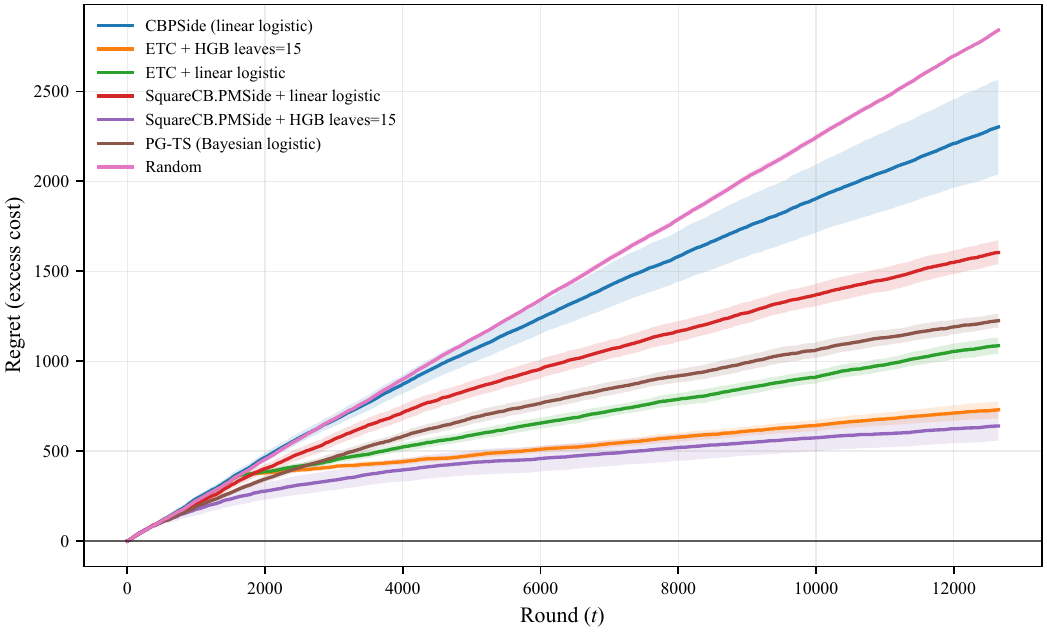}
			}
			\\[2pt]
			
			\multicolumn{2}{c}{
				\subcaptionbox{\(\ell_{01}=3.3\)
					\label{fig:routing-cum-l01-33}}[0.46\textwidth]{
					\includegraphics[width=\linewidth]
					{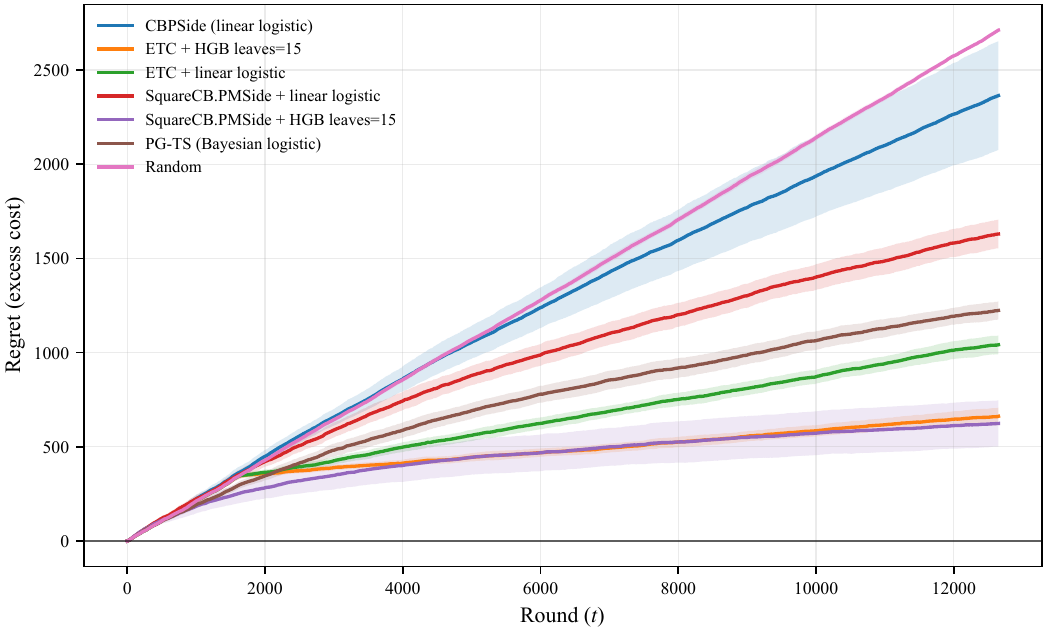}
				}
			}
		\end{tabular}
	\end{adjustbox}
	
	\caption{
		Cumulative empirical regret, measured as excess realized cost relative
		to the fixed five-fold cross-fitted HGB reference policy, for each
		value of \(\ell_{01}\). Curves show means over \(20\) paired random
		arrival orders, and shaded regions are pointwise \(95\%\)
		Student-\(t\) confidence intervals. Lower values are better; negative
		regret indicates that the online policy incurs lower realized cost
		than the fixed reference policy.
	}
	\label{fig:routing-cumulative-regret-all}
\end{figure*}


\paragraph{Additional contextual label-efficient prediction experimental details.}
Contexts are drawn i.i.d.\ as \(x_t\sim\mathrm{Unif}\big([0,1]^{10}\big)\), and a
single weight vector \(w\) is sampled once from \(\mathrm{Unif}(0,0.1)^{10}\) and
normalized so that \(w_i\ge 0\) and \(\sum_i w_i=1\); consequently \(w^\top x\)
is a convex combination of the coordinates of \(x\) and lies in \([0,1]\) by
construction. The binary outcome is the deterministic threshold
\(y_t=\onec{w^\top x_t\ge 0.5}\). We evaluate \(c\in\{0.1,0.2,0.3,0.4,0.45\}\); we stop short of $c = 0.5$ because at that point the cell of the expert action collapses to a single point:
the action becomes degenerate, the two prediction actions become neighbors, and the game falls outside the
non-degenerate class that our analysis covers. Each experiment uses horizon \(T=3000\) and is repeated over \(20\) independent
seeds, and for each policy and each \(c\) we select \(s\in\{0.1,0.3,1,3,10\}\) by
minimizing mean total cost over those \(20\) paired seeds. Each method has exactly one exploration parameter. To place them on a common
footing we scale each by a shared multiplier $s$, sweeping
$s\in\{0.1,0.3,1,3,10\}$ as summarised in Table~\ref{tab:multiplier}. Furthermore, The random baseline instead routes each query to the strong model with a fixed
probability, independent of context, swept over $\{0, 0.2, 0.4, 0.6, 0.8, 1\}$.

\paragraph{Observations } \Cref{fig:le_regret} shows total cost as a function of $c$, and \Cref{fig:curves-all} the learning curves for
all compared algorithms across the full sweep. These extend the results discussed in \Cref{main_section_experiments}.

\begin{figure}[t]
	\centering
	\begin{subfigure}[b]{0.32\linewidth}
		\includegraphics[width=\linewidth]{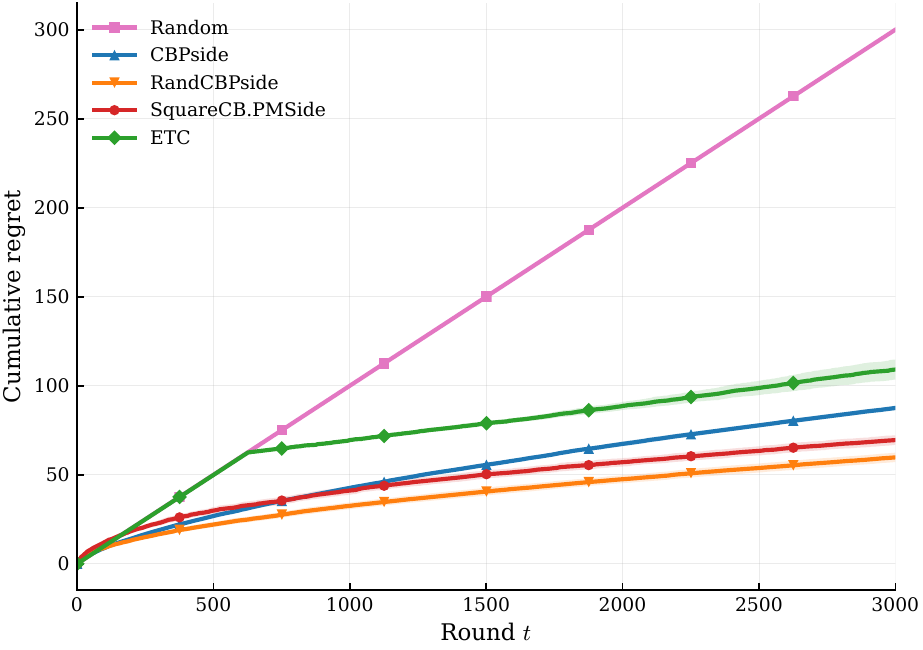}
		\caption{$c=0.1$}
		\label{fig:curves-c01}
	\end{subfigure}
	\hfill
	\begin{subfigure}[b]{0.32\linewidth}
		\includegraphics[width=\linewidth]{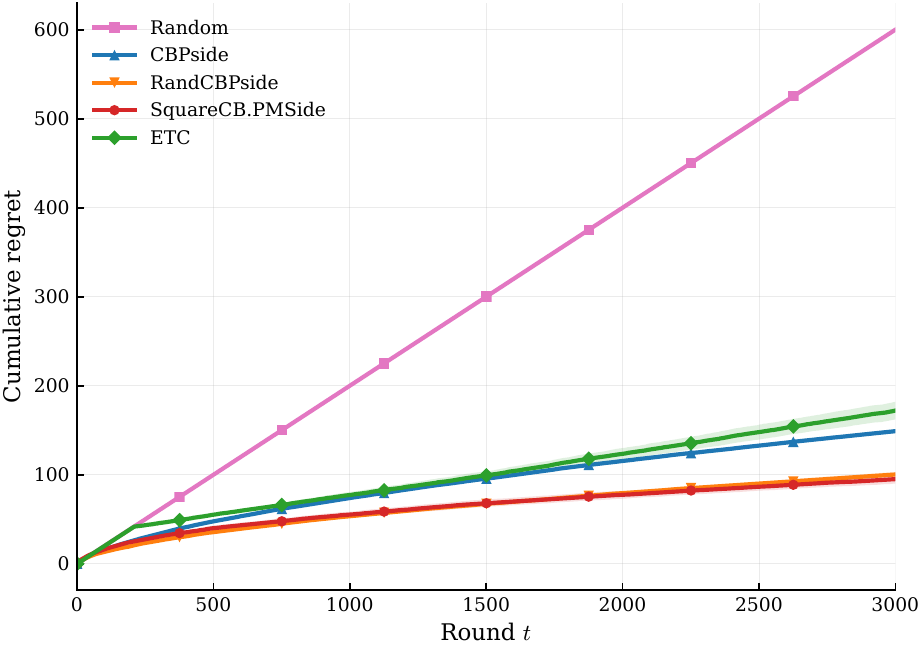}
		\caption{$c=0.2$}
		\label{fig:curves-c02}
	\end{subfigure}
	\hfill
	\begin{subfigure}[b]{0.32\linewidth}
		\includegraphics[width=\linewidth]{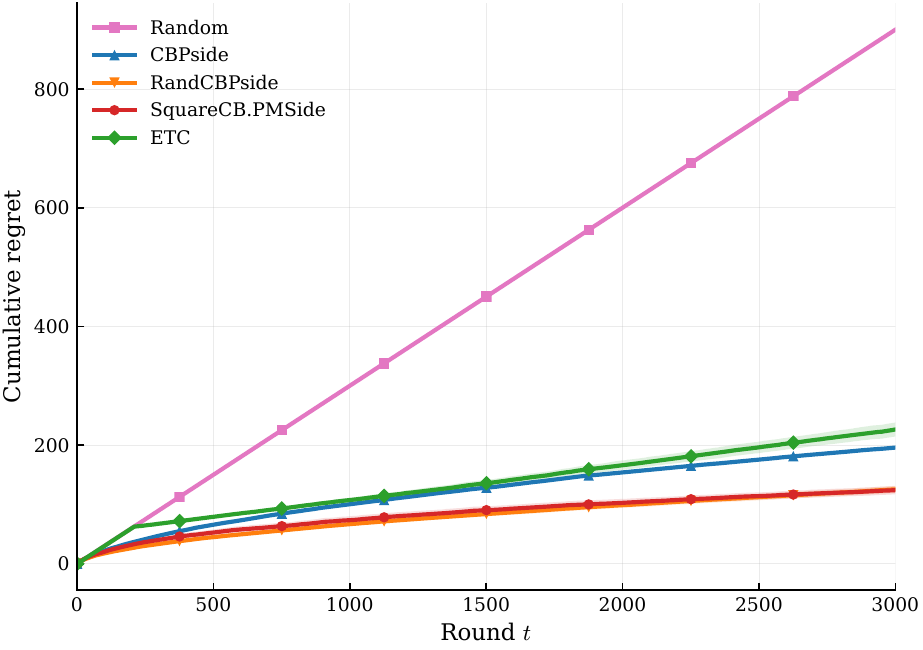}
		\caption{$c=0.3$}
		\label{fig:curves-c03}
	\end{subfigure}
	
	\vspace{0.8em}
	
	\begin{subfigure}[b]{0.32\linewidth}
		\includegraphics[width=\linewidth]{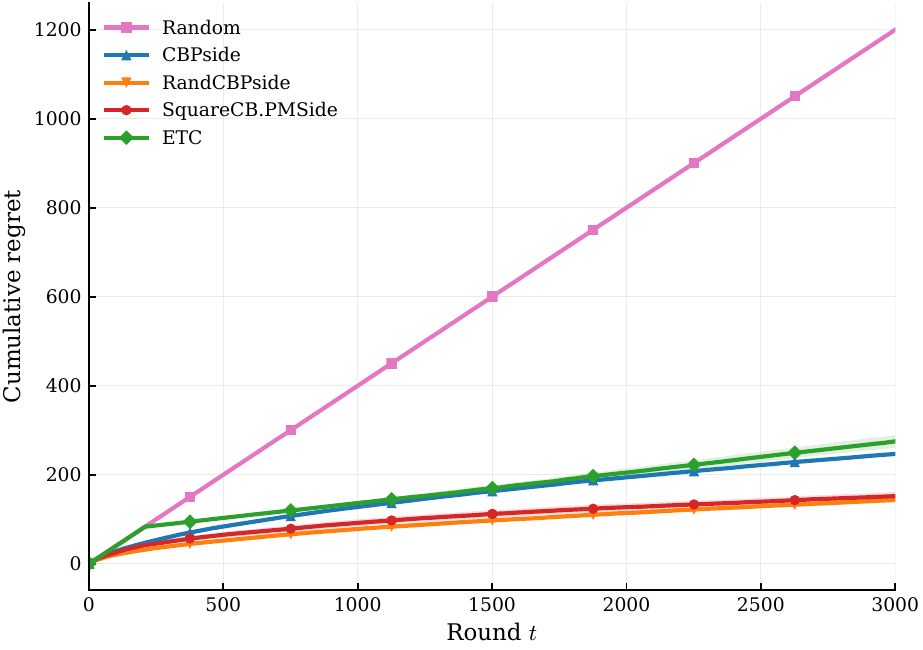}
		\caption{$c=0.4$}
		\label{fig:curves-c04}
	\end{subfigure}
	\hfill
	\begin{subfigure}[b]{0.32\linewidth}
		\includegraphics[width=\linewidth]{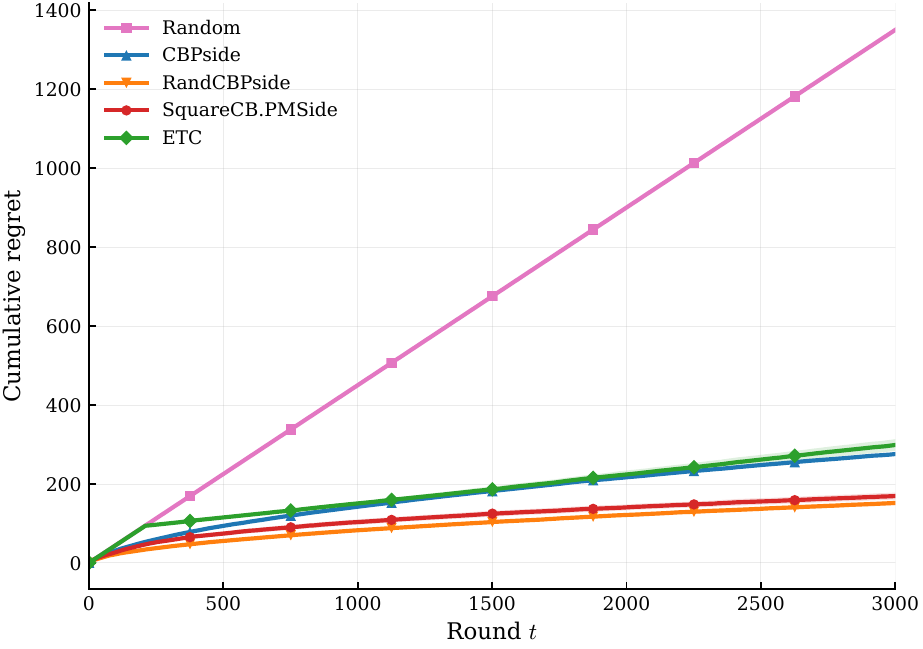}
		\caption{$c=0.45$}
		\label{fig:curves-c045}
	\end{subfigure}
	\hfill
	\begin{subfigure}[b]{0.32\linewidth}
		\phantom{\includegraphics[width=\linewidth]{figures/le_game/sweep_curves_c0.45.pdf}}
	\end{subfigure}
	
	\caption{%
		Cumulative regret against round $t$ for each value of the information cost
		$c$, with every algorithm at its best multiplier. Shaded bands are $95\%$
		Student-$t$ confidence intervals for the mean over $20$ seeds.
	}
	\label{fig:curves-all}
\end{figure}

\section{Algorithm for \texorpdfstring{$\boldsymbol{L}q$}{Lq}-decreasing tree construction}\label{app_section_decreasing_tree}
Given a locally observable game $\mathcal{G} = (\boldsymbol{L}, \boldsymbol{\Phi})$ and an outcome distribution $q \in \Delta(d)$, we denote $\ell_a := e_a^\T \boldsymbol{L} q$ for the loss of action $a$ under the outcome distribution $q$. Since a partial monitoring game can be represented as a neighborhood graph $(\Pi, \mathcal{E})$, we can construct a spanning in-tree $\mathcal{T} \subseteq \mathcal{E}$ rooted at the action with the smallest loss $\hat{a} \in \argmin_{a \in \Pi} \ell_a$. The following algorithm constructs an $\boldsymbol{L}q$-decreasing in-tree by selecting each action's best neighbor to be its parent, then constructing the in-tree $\mathcal{T}$ with the selected neighbors as edges.

\begin{algorithm}[H]
	\caption{Constructing an $\boldsymbol{L}q$-decreasing in-tree by selecting each action's parent}
	\label{alg:construct-in-tree}
	\begin{algorithmic}[1]
		\STATE \textbf{Input:} A locally observable game $\mathcal{G} = (\boldsymbol{L}, \boldsymbol{\Phi})$, its neighborhood graph $(\Pi, \mathcal{E})$, and $q \in \Delta(d)$.
		\STATE \textbf{Output:} A spanning in-tree $\mathcal{T} \subseteq \mathcal{E}$ with root $\mathrm{root}_{\mathcal{T}} = \hat{a}$ such that $\boldsymbol{L} q$ is $\mathcal{T}$-decreasing.
		\STATE For each $a \in \Pi$, compute $\ell_a = e_a^\T \boldsymbol{L} q$, and choose the root $\hat{a} \in \argmin_{a \in \Pi} \ell_a$.
		\STATE Choose $q^\circ \in \operatorname{ri}(\mathcal{C}_{\hat{a}})$, which is the relative interior of the cell $\mathcal{C}_{\hat{a}}$.
		\STATE Initialize $\mathcal{T} \gets \emptyset$.
		\FOR{each $a \in \Pi \setminus \{\hat{a}\}$}
		\STATE Find the best neighbors $B_a \gets \argmin_{b :\, (a,b) \in \mathcal{E}} \ell_b$.
		\STATE Choose $\mathrm{par}_{\mathcal{T}}(a) \in \argmin_{b \in B_a} e_b^\T \boldsymbol{L} q^\circ$, breaking any remaining ties arbitrarily.
		\STATE $\mathcal{T} \gets \mathcal{T} \cup \{(a, \mathrm{par}_{\mathcal{T}}(a))\}$.
		\ENDFOR
		\RETURN the in-tree $\mathcal{T}$ with $\mathrm{root}_{\mathcal{T}} = \hat{a}$.
	\end{algorithmic}
\end{algorithm}

The proof of such an $\boldsymbol{L}q$-decreasing in-tree construction is provided in \citet[Lemma 19]{lattimore2020exploration}.

\section{Auxiliary Lemmata}\label{section_lemmata_contextual}

\begin{lemma}[Pareto optimal comparators suffice]\label{lemma_pareto_suffices}
	For every $q \in \Delta(d)$,
	\begin{align*}
		\min_{\lambda \in \Delta(\Pi)} \lambda^\T \boldsymbol{L} q
		= \min_{a \in \Pi} e_a^\T \boldsymbol{L} q
		= \min_{a \in [k]} e_a^\T \boldsymbol{L} q
		= \min_{\lambda \in \Delta(k)} \lambda^\T \boldsymbol{L} q .
	\end{align*}
\end{lemma}

\begin{proof}
	A linear function on a simplex attains its minimum at a vertex, which gives the outer two equalities, so it
	suffices to show $\bigcup_{a \in \Pi} \mathcal{C}_a = \Delta(d)$. The cells cover $\Delta(d)$ by
	Definition~\ref{def_cells}, and $\bigcup_{a \notin \Pi} \mathcal{C}_a$ is a finite union of polytopes of dimension at
	most $d-2$, hence has empty interior relative to the $(d-1)$-dimensional set $\Delta(d)$. So
	$\bigcup_{a \in \Pi} \mathcal{C}_a$ is dense in $\Delta(d)$; being a finite union of closed sets it is closed,
	and therefore equals $\Delta(d)$.
\end{proof}

\newtheorem*{app_lemmata_theorem_regret_concentration}{Theorem \ref{app_theorem_regret_concentration}}

\begin{app_lemmata_theorem_regret_concentration}[Regret concentration generalized from \citet{qin2026comb}] \label{app_lemmata_theorem_regret_concentration}
	Let $\Lambda$ be a set of policies, $\mathcal{L}(\cdot \mid x): \Lambda \to [0, 1]$ a value function and $\mathcal{L}^\star(\cdot \mid x)$ and $\widehat{\mathcal{L}}_{\ell}(\cdot \mid x)$, $\ell \in [L]$ are value functions measured under the true model and estimated models respectively.
	$\lambda^\star \in \Lambda$ and $\hat{\lambda}_{\ell} \in \Lambda$ are respectively the optimal policies under respective true and estimated models. Let $C_{\ell} : \Lambda \times \Lambda \to \RR_+$ be coverage maps, $p_{\ell} \in \Lambda$ are decisions, $0 < \gamma_1 \leq \gamma_2 \leq \cdots \leq \gamma_L$ learning rates, and $V_{\ell}, E_{\ell} \geq 0$ numbers such that,
	\begin{enumerate}
		\item[\textup{(C1)}] \textup{(certificate)} for every $\ell \geq 1$ and every $\lambda \in \Lambda$,
		\[
		\widehat{\mathcal{L}}_{\ell}(p_{\ell} \mid x) - \widehat{\mathcal{L}}_{\ell}(\lambda\mid x)  + \frac{1}{\gamma_{\ell}}\, C_{\ell}(p_{\ell}, \lambda) \leq V_{\ell}.
		\]
		\item[\textup{(C2)}] \textup{(off-policy evaluation)} for every $\ell \geq 2$ and every $\lambda \in \Lambda$,
		\[
		\abs{ \widehat{\mathcal{L}}_{\ell}(\lambda - p_{\ell-1}\mid x) - \mathcal{L}^\star(\lambda- p_{\ell-1}\mid x) }
		\leq
		\frac{1}{6 \gamma_{\ell}}\, C_{\ell-1}(p_{\ell-1}, \lambda) + \frac{3 \gamma_{\ell}}{2}\, E_{\ell} .
		\]
	\end{enumerate}
	Define $G_1 \coloneqq 1$ and, for $\ell \geq 2$,
	\[
	\gamma_{\ell} G_{\ell} \coloneqq \gamma_{\ell-1} G_{\ell-1} + \gamma_{\ell-1} V_{\ell-1} + 9 \gamma_{\ell}^2 E_{\ell} ,
	\qquad\text{i.e.}\qquad
	G_{\ell} = \frac{1}{\gamma_{\ell}} \sbr{ \gamma_1 + \sum_{i=2}^{\ell} ( \gamma_{i-1} V_{i-1} + 9 \gamma_i^2 E_i ) }  .
	\]
	Then, for every $\ell \in [L]$ and every $\lambda \in \Lambda$,
	\begin{align}
		\REG(\lambda\mid x) \leq 2 \HREG_{\ell}(\lambda\mid x) + G_{\ell}
		\qquad\text{and}\qquad
		\HREG_{\ell}(\lambda\mid x) \leq 2 \REG(\lambda\mid x) + G_{\ell} .
		\label{eqn:oe2d_pm_concentration}
	\end{align}
	Consequently $\REG(p_{\ell}\mid x) \leq 2 V_{\ell} + G_{\ell}$ for every $\ell$, and
	\begin{align}
		\sum_{\ell=1}^L (\tau_{\ell} - \tau_{\ell-1})\, \REG(p_{\ell}\mid x)
		\leq
		\tau_1 + \sum_{\ell =2}^L \frac{\tau_{\ell} - \tau_{\ell-1}}{\gamma_{\ell}} \sbr{ 2 \gamma_{\ell} V_{\ell} + \gamma_1 + \sum_{i=2}^{\ell} ( \gamma_{i-1} V_{i-1} + 9 \gamma_i^2 E_i) }  .
		\label{eqn:oe2d-pm-total-regret}
	\end{align}
\end{app_lemmata_theorem_regret_concentration}

\begin{proof}
	We prove Eq.~\eqref{eqn:oe2d_pm_concentration} by induction on $\ell$. For $\ell = 1$ both inequalities hold with $G_1 = 1$, since $\REG$ and $\HREG_1$ take values in $[0, 1]$. Fix $\ell \geq 2$ and assume Eq.~\eqref{eqn:oe2d_pm_concentration} to hold at $\ell - 1$. 
	
	First, we establish that, for every $\lambda$, our policy and estimator guarantees Low Regret (LR) and Good Coverage(GC).
	
	That is, from (C1) for epoch $\ell-1$, we have
	
	\begin{align*}
		\widehat{\mathcal{L}}_{\ell-1}(p_{\ell-1}\mid x) - \widehat{\mathcal{L}}_{\ell-1}(\lambda\mid x)  + \frac{1}{\gamma_{\ell-1}}\, C_{\ell-1}(p_{\ell-1}, \lambda) &\leq V_{\ell-1} \\
		\widehat{\mathcal{L}}_{\ell-1}(p_{\ell-1}\mid x) - \widehat{\mathcal{L}}_{\ell-1}(\hat{\lambda}_{\ell-1}\mid x) - (\widehat{\mathcal{L}}_{\ell-1}(\lambda\mid x) -\widehat{\mathcal{L}}_{\ell-1}(\hat{\lambda}_{\ell-1}\mid x) )+ \frac{1}{\gamma_{\ell-1}}\, C_{\ell-1}(p_{\ell-1}, \lambda) &\leq V_{\ell-1} \\
		\HREG_{\ell-1}(p_{\ell-1}\mid x) + \left(\frac{1}{\gamma_{\ell-1}}\, C_{\ell-1}(p_{\ell-1}, \lambda) - \HREG_{\ell-1}(\lambda\mid x) \right) &\leq V_{\ell-1} \\
	\end{align*}
	
	Above statement is true for all $\lambda$, including the the case where  $\HREG_{\ell-1}(\lambda\mid x) = 0$ and $C_{\ell}(p_{\ell-1}, \lambda)$ is always $\geq 0$.
	Hence,
	
	\begin{align}\label{eq_low_regret}
		\HREG_{\ell-1}(p_{\ell-1}\mid x)   &\leq V_{\ell-1}. \tag{Low Regret (LR)} 
	\end{align}
	
	Also,
	\begin{align*}
		\left(\frac{1}{\gamma_{\ell-1}}\, C_{\ell-1}(p_{\ell-1}, \lambda) - \HREG_{\ell-1}(\lambda\mid x) \right) &\leq V_{\ell-1}.
	\end{align*}
	Hence,
	\begin{align}
		C_{\ell-1}(p_{\ell-1}, \lambda) \leq \gamma_{\ell-1} \sbr{ V_{\ell-1} + \HREG_{\ell-1}(\lambda\mid x) }.\tag{Good Coverage (GC)}
		\label{eqn:oe2d-pm-certificate-use}
	\end{align}
	
	Moreover,
	\begin{align*}
		&\mathrm{Reg}(\lambda \mid x ) - \widehat{\mathrm{Reg}}_{\ell} (\lambda \mid x)\\
		&\spacex=  \mathcal{L}^\star(\lambda \mid x) - \mathcal{L}^\star(\lambda^* \mid x) - \left( \widehat{\mathcal{L}}_{\ell}(\lambda \mid x) - \widehat{\mathcal{L}}_{\ell}(\hat{\lambda}_{\ell} \mid x) \right) \\
		&\spacex=  \mathcal{L}^\star(\lambda - p_{\ell-1} \mid x) - \mathcal{L}^\star(\lambda^*- p_{\ell-1}  \mid x) - \left( \widehat{\mathcal{L}}_{\ell}(\lambda - p_{\ell-1}  \mid x) - \widehat{\mathcal{L}}_{\ell}(\hat{\lambda}_{\ell} - p_{{\ell}-1} \mid x) \right) \\
		&\spacex=  \mathcal{L}^\star(\lambda - p_{\ell-1} \mid x) - \widehat{\mathcal{L}}_{\ell}(\lambda - p_{\ell-1}  \mid x)  - \left( \mathcal{L}^\star(\lambda^*- p_{\ell-1}  \mid x)  - \widehat{\mathcal{L}}_{\ell}(\hat{\lambda}_{\ell} - p_{{\ell}-1} \mid x) \right) \\
		&\spacex\leq  \mathcal{L}^\star(\lambda - p_{{\ell}-1} \mid x) - \widehat{\mathcal{L}}_{\ell}(\lambda - p_{{\ell}-1}  \mid x)  - \left( \mathcal{L}^\star(\lambda^*- p_{{\ell}-1}  \mid x)  - \widehat{\mathcal{L}}_{\ell}(\lambda^*  - p_{{\ell}-1} \mid x) \right) \tag{$\widehat{\mathcal{L}}_{\ell}(\lambda^* \mid x ) \geq \widehat{\mathcal{L}}_{\ell}(\hat{\lambda}_{\ell} \mid x)$}\\
		&\spacex\leq \lvert \mathcal{L}^\star(\lambda - p_{\ell-1} \mid x) - \widehat{\mathcal{L}}_{\ell}(\lambda - p_{\ell-1}  \mid x)  \rvert + \lvert   \mathcal{L}^\star(\lambda^\star- p_{\ell-1}  \mid x)  - \widehat{\mathcal{L}}_{\ell}(\lambda^* - p_{\ell-1} \mid x)  \rvert \\
		&\spacex\leq \frac{1}{6 \gamma_{\ell}}\, C_{{\ell}-1}(p_{{\ell}-1}, \lambda) + \frac{3 \gamma_{\ell}}{2}\, E_{\ell}  + \frac{1}{6 \gamma_{\ell}}\, C_{{\ell}-1}(p_{{\ell}-1}, \lambda^*) + \frac{3 \gamma_{\ell}}{2}\, E_{\ell} \tag{C2}\\
		&\spacex\leq \frac{1}{6 \gamma_{\ell}}\, \gamma_{\ell-1} \sbr{ V_{\ell-1} + \HREG_{\ell-1}(\lambda \mid x) }   + \frac{3 \gamma_{\ell}}{2}\, E_{\ell}  + \frac{1}{6 \gamma_{\ell}}\, \gamma_{\ell-1} \sbr{ V_{\ell-1} + \HREG_{\ell-1}(\lambda^\star \mid x) }  + \frac{3 \gamma_{\ell}}{2}\, E_{\ell} \tag{ \eqref{eqn:oe2d-pm-certificate-use}} \\
		&\spacex\leq \frac{1}{6 \gamma_{\ell}}\, \gamma_{\ell-1} \sbr{ V_{\ell-1} + 2 \REG(\lambda \mid x) + G_{\ell-1}  }   + \frac{3 \gamma_{\ell}}{2}\, E_{\ell}  + \frac{1}{6 \gamma_{\ell}}\, \gamma_{\ell-1} \sbr{ V_{\ell-1} +2 \REG(\lambda^\star \mid x) + G_{\ell-1} }  + \frac{3 \gamma_{\ell}}{2}\, E_{\ell} \tag{ \eqref{eqn:oe2d_pm_concentration} true for $\ell-1$} \\
		&\spacex\leq \frac{1}{6 \gamma_{\ell}}\, \gamma_{\ell-1} \sbr{ D_{\ell-1} + 2 \REG(\lambda \mid x) + G_{\ell-1}  }   + \frac{3 \gamma_{\ell}}{2}\, E_{\ell}  + \frac{1}{6 \gamma_{\ell}}\, \gamma_{\ell-1} \sbr{ V_{\ell-1} +  G_{\ell-1} }  + \frac{3 \gamma_{\ell}}{2}\, E_{\ell} \tag{ $\REG(\lambda^\star \mid x) = 0$} \\
		&\spacex= \frac{1}{ 3 \gamma_{\ell}} \left( \gamma_{\ell-1} G_{\ell-1} + \gamma_{\ell-1} V_{\ell-1} + 9 \gamma^2_{\ell} E_{\ell} \right) + \frac{\gamma_{\ell-1} }{3 \gamma_{\ell} }\REG(\lambda \mid x)  \\
		&\spacex\leq \frac{1}{3}\REG(\lambda \mid x)+ \frac{1}{3}G_{\ell} \tag{$\gamma_{\ell-1}  \leq \gamma_{\ell}  $}	
	\end{align*}
	
	Hence
	\begin{align*}
		\mathrm{Reg}(\lambda \mid x ) - \widehat{\mathrm{Reg}}_{\ell} (\lambda \mid x) &\leq \frac{1}{3}\REG(\lambda \mid x)+ \frac{1}{3}G_{\ell} \\
		\frac{2}{3} \mathrm{Reg}(\lambda \mid x ) &\leq  \widehat{\mathrm{Reg}}_{\ell} (\lambda \mid x) +  \frac{1}{3}G_{\ell} \\ 
		\mathrm{Reg}(\lambda \mid x )  &\leq \frac{3}{2}  \widehat{\mathrm{Reg}}_{\ell} (\lambda \mid x) +  \frac{1}{2}G_{\ell} \\ 
		&\leq  2 \widehat{\mathrm{Reg}}_{\ell} (\lambda \mid x) +  G_{\ell} \tag{Step **}\\ 
	\end{align*}
	
	Furthermore,
	
	\begin{align*}
		& \widehat{\mathrm{Reg}}_{\ell} (\lambda \mid x) - \mathrm{Reg}(\lambda \mid x )\\
		&\spacex=   \widehat{\mathcal{L}}_{\ell}(\lambda \mid x) - \widehat{\mathcal{L}}_{\ell}(\hat{\lambda}_{\ell} \mid x)   - \left(\mathcal{L}^\star(\lambda \mid x) - \mathcal{L}^\star(\lambda^* \mid x)\right) \\
		&\spacex=  \widehat{\mathcal{L}}_{\ell}(\lambda - p_{\ell-1}  \mid x) - \widehat{\mathcal{L}}_{\ell}(\hat{\lambda}_{\ell} - p_{\ell-1} \mid x) - \left(  \mathcal{L}^\star(\lambda - p_{\ell-1} \mid x) - \mathcal{L}^\star(\lambda^*- p_{\ell-1}  \mid x) \right) \\
		&\spacex=  \widehat{\mathcal{L}}_{\ell}(\lambda - p_{\ell-1}  \mid x)  -  \mathcal{L}^\star(\lambda - p_{\ell-1} \mid x)  - \left( \widehat{\mathcal{L}}_{\ell}(\hat{\lambda}_{\ell} - p_{\ell-1} \mid x)  - \mathcal{L}^\star(\lambda^*- p_{\ell-1}  \mid x)\right) \\
		&\spacex\leq \widehat{\mathcal{L}}_{\ell}(\lambda - p_{\ell-1}  \mid x)  -  \mathcal{L}^\star(\lambda - p_{\ell-1} \mid x)  - \left( \widehat{\mathcal{L}}_{\ell}(\hat{\lambda}_{\ell} - p_{\ell-1} \mid x)  - \mathcal{L}^\star(\hat{\lambda}_{\ell}- p_{\ell-1}  \mid x)\right)   \tag{$ \mathcal{L}^\star(\hat{\lambda}_{\ell} \mid x) \geq \mathcal{L}^\star(\lambda^\star \mid x) $}\\
		&\spacex\leq \lvert \widehat{\mathcal{L}}_{\ell}(\lambda - p_{\ell-1}  \mid x)  -  \mathcal{L}^\star(\lambda - p_{\ell-1} \mid x) \rvert + \lvert  \widehat{\mathcal{L}}_{\ell}(\hat{\lambda}_{\ell} - p_{\ell-1} \mid x)  - \mathcal{L}^\star(\hat{\lambda}_{\ell}- p_{\ell-1}  \mid x) \rvert \\
		&\spacex\leq \frac{1}{6 \gamma_{\ell}}\, C_{\ell-1}(p_{\ell-1}, \lambda) + \frac{3 \gamma_{\ell}}{2}\, E_{\ell}  + \frac{1}{6 \gamma_{\ell}}\, C_{\ell-1}(p_{\ell-1}, \hat{\lambda}_{\ell}) + \frac{3 \gamma_{\ell}}{2}\, E_{\ell} \tag{C2}\\
		&\spacex\leq \frac{1}{6 \gamma_{\ell}}\, \gamma_{\ell-1} \sbr{ V_{\ell-1} + \HREG_{\ell-1}(\lambda \mid x) }   + \frac{3 \gamma_{\ell}}{2}\, E_{\ell}  + \frac{1}{6 \gamma_{\ell}}\, \gamma_{\ell-1} \sbr{ V_{\ell-1} + \HREG_{\ell-1}( \hat{\lambda}_{\ell} \mid x) }  + \frac{3 \gamma_{\ell}}{2}\, E_{\ell} \tag{ \eqref{eqn:oe2d-pm-certificate-use}} \\
		&\spacex\leq \frac{1}{6 \gamma_{\ell}}\, \gamma_{\ell-1} \sbr{ V_{\ell-1} + 2 \REG(\lambda \mid x) + G_{\ell-1}  }   + \frac{3 \gamma_{\ell}}{2}\, E_{\ell}  + \frac{1}{6 \gamma_{\ell}}\, \gamma_{\ell-1} \sbr{ V_{\ell-1} +2 \REG( \hat{\lambda}_{\ell} \mid x) + G_{\ell-1} }  + \frac{3 \gamma_{\ell}}{2}\, E_{\ell} \tag{ \eqref{eqn:oe2d_pm_concentration} true for $\ell-1$} \\
		&\spacex\leq \frac{1}{6 \gamma_{\ell}}\, \gamma_{\ell-1} \sbr{ V_{\ell-1} + 2 \REG(\lambda \mid x) + G_{\ell-1}  }   + \frac{3 \gamma_{\ell}}{2}\, E_{\ell}  + \frac{1}{6 \gamma_{\ell}}\, \gamma_{\ell-1} \sbr{ V_{\ell-1} + 4  \HREG_{\ell}( \hat{\lambda}_{\ell} \mid x) + 2G_{\ell} + G_{\ell-1} } \\
		&\spacex\spacex + \frac{3 \gamma_{\ell}}{2}\, E_{\ell} \tag{ \eqref{eqn:oe2d_pm_concentration} true for $\ell$ (Step **)} \\
		&\spacex= \frac{1}{6 \gamma_{\ell}}\, \gamma_{\ell-1} \sbr{ V_{\ell-1} + 2 \REG(\lambda \mid x) + G_{\ell-1}  }   + \frac{3 \gamma_{\ell}}{2}\, E_{\ell}  + \frac{1}{6 \gamma_{\ell}}\, \gamma_{\ell-1} \sbr{ V_{\ell-1} + 2G_{\ell} + G_{\ell-1} }  + \frac{3 \gamma_{\ell}}{2}\, E_{\ell} \tag{$ \HREG_{\ell}( \hat{\lambda}_{\ell} \mid x)  = 0$ } \\
		&\spacex= \frac{1}{ 3 \gamma_{\ell}} \left( \gamma_{\ell-1} G_{\ell-1} + \gamma_{\ell-1} V_{\ell-1} + 9 \gamma^2_{\ell} E_{\ell} \right) + \frac{\gamma_{\ell-1} }{3 \gamma_{\ell} }\REG(\lambda \mid x)  + \frac{\gamma_{\ell-1} }{3 \gamma_{\ell}}G_{\ell}\\
		&\spacex\leq \frac{1}{3}\REG(\lambda \mid x)+ \frac{2}{3}G_{\ell} \tag{$\gamma_{\ell-1}  \leq \gamma_{\ell}  $}
	\end{align*}
	
	Hence,
	
	\begin{align*}
		\widehat{\mathrm{Reg}}_{\ell} (\lambda \mid x) - \mathrm{Reg}(\lambda \mid x ) &\leq \frac{1}{3}\REG(\lambda \mid x)+ \frac{2}{3}G_{\ell} \\
		\widehat{\mathrm{Reg}}_{\ell} (\lambda \mid x) &\leq  \frac{4}{3}\mathrm{Reg}(\lambda \mid x ) + \frac{2}{3}G_{\ell} \\
		\mathrm{Reg}(\lambda \mid x )  &\leq \frac{3}{2}  \widehat{\mathrm{Reg}}_{\ell} (\lambda \mid x) +  \frac{1}{2}G_{\ell} \\ 
		&\leq  2 \widehat{\mathrm{Reg}}_{\ell} (\lambda \mid x) +  G_{\ell} \\ 
	\end{align*}
	
	Finally, we prove the consequences. Apply (C1) at $\lambda = \hat{\lambda}_{\ell}$, we have
	\begin{align*}
		\HREG_{\ell}(p_{\ell} \mid x) &\leq  V_{\ell} \tag{\eqref{eq_low_regret}}\\
		\Rightarrow
		\REG(p_{\ell} \mid x)
		&\leq 2 \HREG_{\ell}(p_{\ell }\mid x) + G_{\ell} \leq 2 V_{\ell} + G_{\ell} .
		\tag{By \eqref{eqn:oe2d_pm_concentration}}
	\end{align*}
	Summing over epochs, we have
	\begin{align*}
		\sum_{\ell=1}^L (\tau_{\ell} - \tau_{\ell-1})\, \REG(p_{\ell}\mid x)
		&\leq \tau_1 + \sum_{\ell=2}^L (\tau_{\ell} - \tau_{\ell-1}) ( 2 V_{\ell} + G_{\ell} )
		\tag{$\REG(p_1\mid x) \leq 1$}
		\\
		&= \tau_1 + \sum_{\ell=2}^L \frac{\tau_{\ell} - \tau_{\ell-1}}{\gamma_{\ell}} \sbr{ 2 \gamma_{\ell} V_{\ell} + \gamma_1 + \sum_{i=2}^{\ell} ( \gamma_{i-1} V_{i-1} + 9 \gamma_i^2 E_i )} ,
		\tag{By the closed form of $G_{\ell}$}
	\end{align*}
	which is Eq.~\eqref{eqn:oe2d-pm-total-regret}.
\end{proof}

\begin{lemma}\label{dist_ineq}
	Let $P$, $Q$ are two probability simplices. $D_{\mathrm{H}}(P \,\big|\, Q)$, $D_{\mathrm{KL}}(P \,\big|\, Q)$, $\lVert P - Q\rVert_{\mathrm{TV}}$, $\lVert P - Q\rVert_{1}$ be the \textit{Hellinger} distance, \textit{Kullback} divergence, \textit{total variation} distance and level-1 norm between $P$, $Q$ respectively, then the following holds,
	\begin{align*}
		\frac{\lVert P - Q\rVert_{1}}{2} = 	\lVert P - Q\rVert_{\mathrm{TV}} \leq D_{\mathrm{H}}(P \,\big|\, Q) 
	\end{align*}
	and,
	\begin{align*}
		\frac{\lVert P - Q\rVert_{1}}{2} = 	\lVert P - Q\rVert_{\mathrm{TV}} \leq \sqrt{D_{\mathrm{KL}}(P \,\big|\, Q) }
	\end{align*}
	\begin{proof}		
		\begin{align*}
			\lVert P - Q\rVert_{\mathrm{TV}} \leq D_{\mathrm{H}}(P \,\big|\, Q)  \tag{ \citet{Tsybakov2008IntroductionTN}, Lemma 2.3}
		\end{align*}
		and,
		\begin{align*}
			\lVert P - Q\rVert_{\mathrm{TV}} \leq  \sqrt{D_{\mathrm{KL}}(P \,\big|\, Q) } \tag{ \citet{Tsybakov2008IntroductionTN}, Lemma 2.5}
		\end{align*}
		Furthermore,
		\begin{align*}
			\frac{\lVert P - Q\rVert_{1}}{2} = 	\lVert P - Q\rVert_{\mathrm{TV}}  \tag{\citet{levin2017markov}, Proposition 4.2  }
		\end{align*}
	\end{proof}
\end{lemma}

\begin{lemma}[Local observability, \citet{lattimore2020exploration}]\label{lem:bounds}
	If $\mathcal{G} = (\Phi, \boldsymbol{L})$ is locally observable and non-degenerate and actions $e = (a, b) \in \mathcal{E}$ are neighbours, then there exist functions $w_e(a,\cdot), w_e(b,\cdot) : \Sigma \to \mathbb{R}$ such that
	$\norm{w_e}_\infty \leq m/2$ and
	\begin{align*}
		\boldsymbol{L}_{ao} - \boldsymbol{L}_{bo} = w_e(a, \Phi_{ao}) + w_e(b, \Phi_{bo}) \text{ for all } o \in[d]\,. 
		\label{eq:w-cond}
	\end{align*} 
\end{lemma}

In Lemma \ref{lem:bounds}, the assumption of non-degeneracy is important here for $\norm{w_e}_\infty \leq m/2$ to hold. The Proposition \ref{lem:bounds-global} bounds $\norm{w_e}_\infty$ for more general case.

\begin{lemma}[\citet{lattimore2020exploration}]\label{lem:bounds-global}
	Suppose that $(L, \Phi)$ is globally observable and $a$ and $b$ are neighbours. Then there exists a function $w_e : [k] \times \Sigma \to \mathbb{R}$ such that for all $x \in [d]$,
	\begin{align*}
		\sum_{c=1}^k w_e(c, \Phi_{co}) = L_{ao} - L_{bo}\,.
	\end{align*}
	Furthermore, $w_e$ can be chosen so that $\norm{w_e}_\infty \leq d^{1/2} k^{d/2}$.
\end{lemma}

\begin{lemma}[Water-transfer operator adapted from \citet{lattimore2020exploration}]\label{lemma_water_transfer}
	Suppose that $\mathcal{G}$ is non-degenerate and locally observable and
	$q \in \Delta(d)$ Assume that we have an in-tree $\cT$ that is $Lq$-decreasing. Then there exists a function $W_q^{\cT} : \Delta(\Pi) \to \Delta(\Pi)$ such that the following hold for all $p \in \Delta(\Pi)$:
	\begin{enumerate}
		\item[(a)] The expected loss does not increase: $\left(W_q(p) 
		- p\right)^\top \boldsymbol{L} q \leq 0$.
		\item[(b)] Action probabilities are not too small: $W_q(p)[a] \geq p[a]/k$ for all $a \in [k]$.
		\item[(c)] Probabilities increase towards the root of the in-tree $\mathcal{T}$. That is, $W_q(p)[a] \leq W_q(p)[b]$ for all $e = (a, b) \in \mathcal{T}$.
	\end{enumerate}
\end{lemma}

\clearpage

\end{document}